\documentclass{article}
\usepackage{preprint,times}
\usepackage[T1]{fontenc}
\usepackage{amsmath,amssymb,amsthm,bm}
\usepackage{graphicx,booktabs,multirow,array,float}
\usepackage{algorithm,algpseudocode,wrapfig}
\usepackage[hidelinks]{hyperref}
\usepackage{url,placeins,microtype}
\usepackage{pgfplots}
\pgfplotsset{compat=1.18}
\newtheorem{proposition}{Proposition}[section]
\newtheorem{theorem}[proposition]{Theorem}
\newcommand{\E}{\mathbb{E}}
\newcommand{\N}{\mathcal{N}}
\newcommand{\sg}{\operatorname{sg}}

\title{Flow-Map GRPO: Reinforcement Learning for Few-Step Flow-Map Generators via Anchored Stochastic Composition}
\author{Zhiqi Li$^{1}$ \quad Wen Zhang$^{1}$ \quad Bo Zhu$^{1}$}
\hypersetup{pdftitle={Flow-Map GRPO: Reinforcement Learning for Few-Step Flow-Map Generators via Anchored Stochastic Composition},pdfauthor={Zhiqi Li, Wen Zhang, Bo Zhu},pdfsubject={Preprint}}
\begin{document}
\maketitle
\footnotetext[1]{Georgia Institute of Technology, Atlanta, GA, USA. Correspondence to: Zhiqi Li \texttt{<zli3167@gatech.edu>}.}

\begin{figure}[H]
\centering
\vspace{-5mm}
\includegraphics[width=\linewidth]{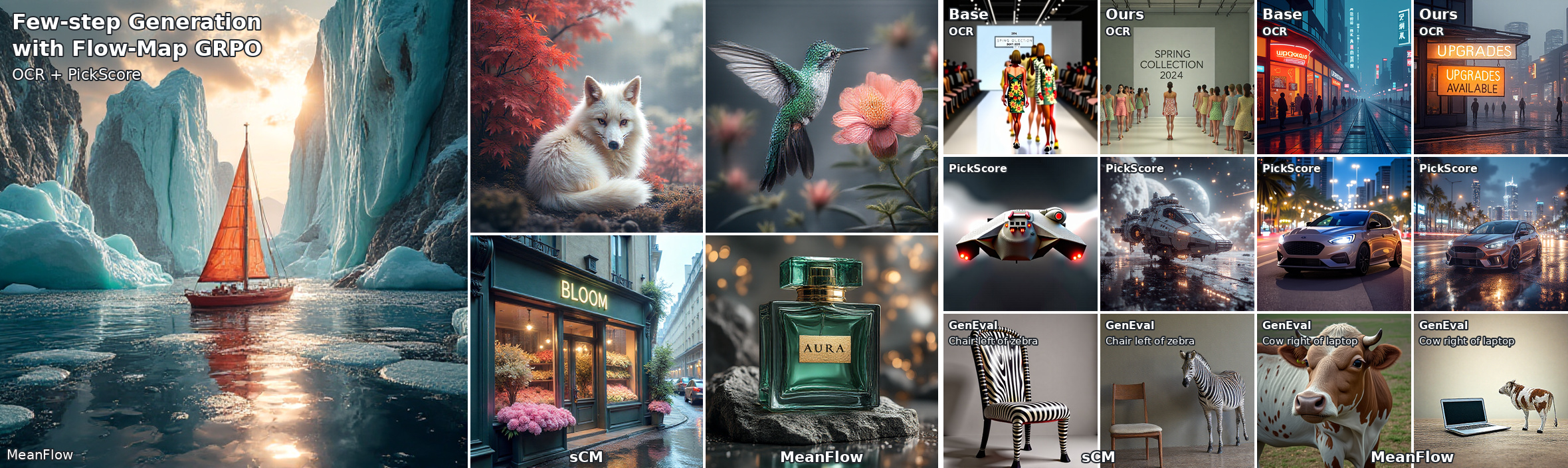}
\vspace{-5mm}
\caption{Few-step generation with Flow-Map GRPO. Left: images generated with Mix reward on MeanFlow and sCM. Right: Base/Ours comparisons on OCR, PickScore, and GenEval, with matched sampling steps from 2 to 8.}
\vspace{-2.5mm}
\label{fig:mainqual}
\end{figure}
\begin{abstract}
Few-step flow-map generators, such as consistency models and MeanFlow, accelerate sampling by learning long-range transport maps between noise and data. However, their deterministic transitions do not directly provide the stochastic trajectories and tractable likelihood ratios required by reinforcement learning (RL) post-training. Existing SDE-based stochasticization techniques target velocity-based samplers and do not directly extend to long-range flow-map transitions. We propose \textbf{Flow-Map GRPO}, an online RL post-training framework for deterministic few-step flow-map generators. Its key component, \textbf{Anchored Stochastic Flow Map Composition (ASFMC)}, combines deterministic transport with anchor-based conditional resampling. We establish the conditions under which this construction preserves the marginal probability path and develop tractable local- and endpoint-anchor policies for two-time and single-time flow maps. These policies enable a unified GRPO training procedure. Experiments on FLUX-based MeanFlow and sCM generators demonstrate substantial improvements in OCR, PickScore, and GenEval at different numbers of inference steps, including joint OCR--PickScore gains with mixed rewards. Controlled ablations show that the stochastic transition design is essential for translating training rewards into generation quality. Flow-Map GRPO enables effective RL alignment of pretrained deterministic flow-map generators while retaining their original parameterization, without retraining them as native stochastic models.
\end{abstract}

\section{Introduction}
Diffusion and continuous-time flow models have become a powerful framework for image generation~\citep{ho2020denoising,rombach2022high,lipman2022flow,liu2023flow}. Their sampling procedures typically integrate learned dynamics over many steps. Few-step flow-map generators instead learn long-range mappings between noise and data, amortizing this integration into one or a few network evaluations. Representative methods include consistency models, shortcut models, and MeanFlow~\citep{song2023consistency,lu2025simplifying,frans2025shortcut,geng2025mean}. These models offer an attractive starting point for efficient generation, but adapting them to downstream rewards remains challenging. The difficulty is that fast deterministic sampling does not directly define a stochastic policy. Online reinforcement learning methods such as DDPO and Flow-GRPO explore alternative generations through stochastic trajectories and optimize their likelihood ratios~\citep{black2024training,liu2026flow}. A deterministic flow-map transition is a point mass, so it cannot supply the nondegenerate transition density needed by these methods. For velocity-based generators, an SDE reformulation provides a local Gaussian sampling rule. A long-range flow map, however, represents an entire transport interval. Adding Gaussian noise to its output does not generally reproduce the corresponding finite-time SDE transition, and resolving that SDE with many small steps would sacrifice the generator's sampling advantage.

We address this problem at the level of the flow map itself. Our key construction, \textbf{Anchored Stochastic Flow Map Composition (ASFMC)}, combines long-range deterministic transport with anchor-based conditional resampling to construct a stochastic policy for RL post-training. Given a deterministic transition, ASFMC uses an auxiliary anchor to define a resampling step with a tractable density. A local anchor leads to a Gaussian drift correction around the transported sample; an endpoint anchor uses a predicted clean sample and the model's noising schedule. ASFMC supports both MeanFlow-style two-time maps and sCM-style endpoint models without changing their network parameterization. Building on ASFMC, we introduce \textbf{Flow-Map GRPO}, which optimizes stochastic flow-map policies using group-relative rewards. The framework supports both individual-image and image-set rewards, enabling reward-based post-training while retaining the generator's original few-step sampling capability. To our knowledge, Flow-Map GRPO is the first likelihood-ratio RL method for pretrained MeanFlow generators and the first unified RL framework spanning pretrained one-time and two-time flow-map generators, enabled by stochastic policy construction directly at the flow-map level.

Empirically, ablations on ImageNet demonstrate the importance of the proposed stochastic policy construction for effective post-training of few-step generators, while experiments on FLUX-based MeanFlow and sCM show substantial gains under both single-reward and mixed-reward objectives.  \textbf{Our contributions are as follows.} (1) We identify the mismatch between deterministic long-range flow maps and likelihood-ratio post-training, and introduce ASFMC as a tractable stochasticization mechanism. (2) We develop local and endpoint policies and a unified GRPO training procedure for pretrained two-time and one-time flow-map generators. (3) We validate the framework through controlled ImageNet ablations and text-to-image experiments on FLUX-based MeanFlow and sCM under both single-reward and mixed-reward objectives.

\section{Background}\label{sec:background}\label{sec:flow_map}
\paragraph{Flow models and flow maps.}
Let $\mathcal D=\{x^i\}_{i=1}^n$ denote samples from an unknown data distribution $p_1=p_{\mathrm{data}}$ on $\mathcal X\subseteq\mathbb R^d$. A continuous-time generative model constructs a process $\{X_t\}_{t\in[0,1]}$ whose marginals evolve from a simple base distribution $p_0$, typically a standard Gaussian, to $p_1$. These marginals form a probability path $\{p_t\}_{t\in[0,1]}$. A deterministic probability-flow ODE transports samples from $p_0$ to $p_1$ along this path:
\begin{equation}
    \frac{dx_t}{dt}=u_t(x_t),\qquad x_0\sim p_0=\N(0,I).
    \label{eq:ode}
\end{equation}
whose marginal densities satisfy the continuity equation $\partial_t p_t(x)+\nabla_x\!\cdot\!\bigl(p_t(x)u_t(x)\bigr)=0$.
Many flow generators use the affine path $X_t=a_tX_1+b_tX_0$, where $X_0\sim\N(0,I)$ and $X_1\sim p_1$ are independent, with $(a_0,b_0)=(0,1)$ and $(a_1,b_1)=(1,0)$.
The Flow Matching path is recovered by choosing $(a_t,b_t)=(t,1-t)$, while the DDIM path is recovered by choosing $(a_t,b_t)=(\sqrt{1-(1-t)^2},1-t)$.  The generation condition $c$ is omitted when unambiguous.  The solution of Equation~\eqref{eq:ode} defines a time-dependent flow $\psi_t:\mathcal X\to\mathcal X$ with $\psi_0(x)=x$, such that $x_t=\psi_t(x_0)$, and the flow map is defined as the long-range transition between any two times $t$ and $r$, $\psi_{t\to r}=\psi_r\circ\psi_t^{-1}$. Unlike multi-step methods that learn a velocity field $u_t^\theta$ and generate samples $x_1$ by numerically integrating Equation~\eqref{eq:ode}, flow-map models learn this long-range transport directly. For example, MeanFlow parameterizes the flow map through an average velocity, $\psi^\theta_{t\to r}(x_t)=x_t+(r-t)u^\theta_{t\to r}(x_t)$, and enables efficient few-step generation, such as two-step sampling $x_{0.5}=x_0+0.5u^\theta_{0\to0.5}(x_0)$ and $x_1=x_{0.5}+0.5u^\theta_{0.5\to1}(x_{0.5})$, with $x_0\sim p_0$.

\begin{figure}[!t]
\centering
\begin{minipage}[t]{0.48\linewidth}
\vspace{-3mm}
\centering\includegraphics[width=\linewidth]{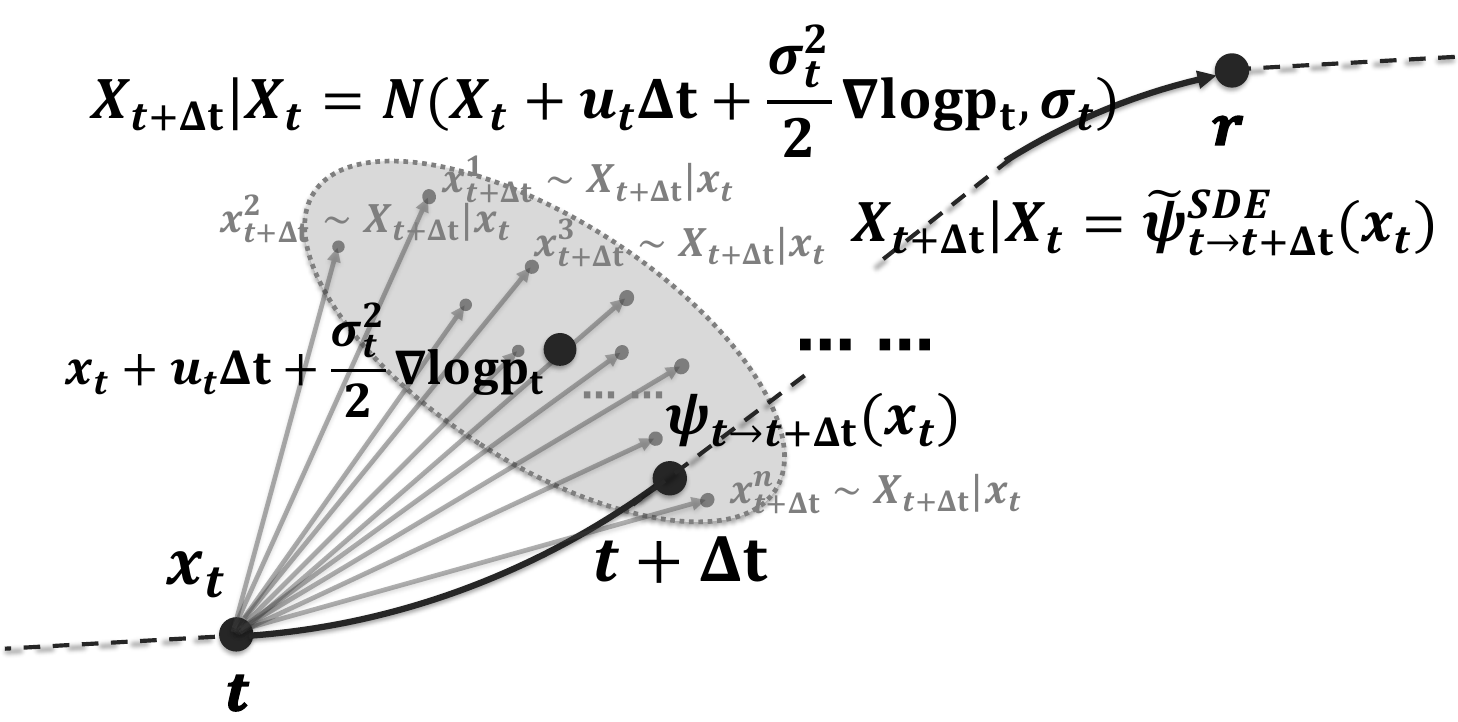}\\
\vspace{-2mm}
{\small (a) Infinitesimal transition}
\end{minipage}\hfill
\begin{minipage}[t]{0.48\linewidth}
\vspace{-3mm}
\centering\includegraphics[width=\linewidth]{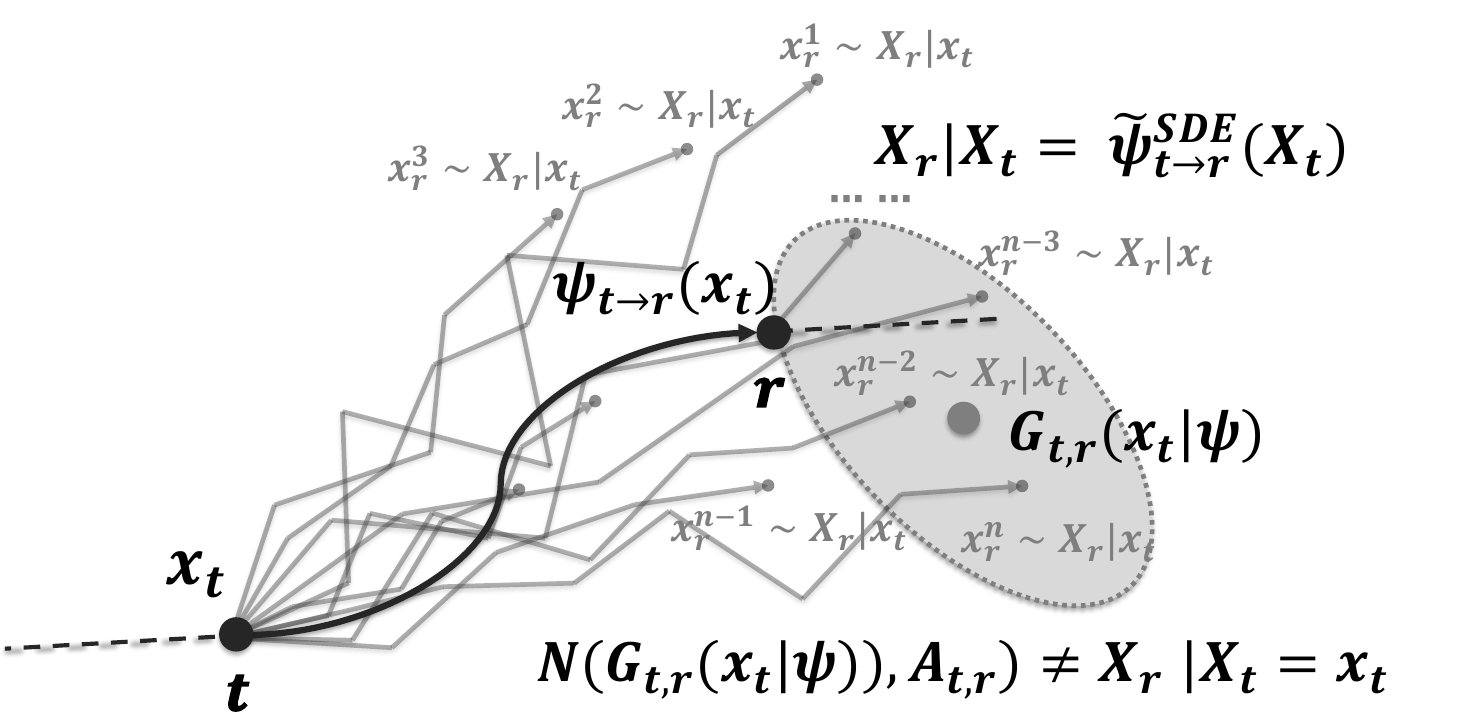}\\
\vspace{-2mm}
{\small (b) Long-range transition}
\end{minipage}
\vspace{-3mm}
\caption{Stochasticization at different transition scales. A local SDE discretization gives a tractable Gaussian transition for an infinitesimal velocity step. Over a finite interval, accumulated nonlinear stochastic dynamics generally produce a non-Gaussian transition. This motivates constructing a stochastic policy directly around the long-range flow map.}
\vspace{-5mm}
\label{fig:sde_local_vs_longrange}
\end{figure}

\paragraph{Reward-based post-training.}\label{sec:rl_flow}
Recent works post-train diffusion and flow models through reinforcement learning with differentiable or non-differentiable rewards, including human preference, aesthetic quality, and downstream metrics~\citep{liu2026flow,li2025mixgrpo,xue2025dancegrpo,black2024training}. Generation is formulated as a multi-step Markov decision process (MDP). On a sampling grid $0=t_0<\cdots<t_N=1$, the state is $s_i=(c,t_i,x_{t_i})$, the action is the next sample $a_i=x_{t_{i+1}}$, and the policy is $\pi_\theta(a_i\mid s_i)=p_\theta(x_{t_{i+1}}\mid x_{t_i},c)$. The next state is $(c,t_{i+1},a_i)$, and a terminal reward $r(x_1,c)$ evaluates the sample. Policy-gradient methods optimize the model through its transition densities. Stochastic policies support exploration and allow likelihood ratios to be computed for policy updates and reference-policy regularization.  DDPO uses stochastic diffusion denoising transitions~\citep{black2024training}; Flow-GRPO introduces stochasticity into velocity-based flow sampling through a marginal-preserving SDE~\citep{liu2026flow}. For the ODE in Equation~\eqref{eq:ode}, these dynamics take the form $dX_t=[u_t(X_t)+\frac{\sigma_t^2}{2}\nabla_x\log p_t(X_t)]\,dt+\sigma_t\,dW_t$, where $q_t(x):=\nabla_x\log p_t(x)$ is the marginal score, $W_t$ is a Wiener process, and $\sigma_t$ controls the noise. The score correction preserves the marginal probability path in continuous time. This is important for introducing stochastic exploration while keeping the sampled states consistent with the pretrained model's noise levels.

\section{Method}\label{sec:method}
\subsection{Challenge for Flow Maps}\label{sec:challenge}
For RL post-training, as with the multi-step models in Flow-GRPO~\citep{liu2026flow}, flow maps need to be randomized to support stochastic exploration while preserving the pretrained marginal probability path. Specifically, if $X_t\sim p_t$, the randomized flow map $\widetilde\psi_{t\to r}$ should satisfy $\widetilde\psi_{t\to r}(X_t)\sim p_r$. The transition ablations (Section~\ref{sec:ablation}) and midpoint-anchor results (Appendix~\ref{app:midpoint}) examine the practical importance of this requirement.  A natural idea is to follow Flow-GRPO and inject stochasticity through an equivalent SDE whose marginal distributions match those of the deterministic ODE. For a small time step, its transition is approximated by the Gaussian update
\begin{equation}
    X_{t+\Delta t}\approx X_t+\left[u_t(X_t)+\frac{\sigma_t^2}{2}\nabla_x\log p_t(X_t)\right]\Delta t
    +\sigma_t\sqrt{\Delta t}\,\epsilon_t,\qquad\epsilon_t\sim\N(0,I).
    \label{eq:local_sde_transition}
\end{equation}
This allows stochasticity to be injected into velocity-based samplers, as in Flow-GRPO. However, flow maps are long-range transport operators between arbitrary time pairs, and their transitions are not infinitesimal. The long-range transition induced by the SDE cannot be reduced to a simple additive-noise form: a deterministic part plus an input-independent Gaussian perturbation (Fig.~\ref{fig:sde_local_vs_longrange}).

\begin{samepage}
\begin{theorem}[Obstruction to additive-Gaussian SDE transitions]\label{thm:long_range_sde}
For nonlinear drift, the exact SDE transition over $t<r$ need not admit a representation
\begin{equation}
    X_r=G_{t,r}(X_t;\psi,\mathrm d\psi,\ldots)+A_{t,r}\epsilon,
    \qquad\epsilon\sim\N(0,I),\quad\epsilon\perp X_t,
    \label{eq:gaussian_obstruction}
\end{equation}
where $G_{t,r}(X_t;\psi,\mathrm d\psi,\ldots)$ is a deterministic function of the initial state $X_t$, the flow-map family $\psi$, and its derivatives, denoted by $\mathrm d\psi,\ldots$. The matrix $A_{t,r}$ is deterministic and independent of $X_t$. In particular, smooth additive-noise diffusions exist whose conditional transition laws are non-Gaussian, ruling out any such universal representation based on the deterministic flow map and its derivatives (see Appendix~\ref{app:finite_time} for proof).
\end{theorem}
\end{samepage}
The exact transition satisfies $X_r=X_t+\int_t^r\beta_s(X_s)\,ds+\int_t^r\sigma_s\,dW_s$, where $\beta_s=u_s+\tfrac12\sigma_s^2q_s$; its drift integral depends on the stochastic path. Resolving that trajectory with many small steps would sacrifice the few-step advantage. As a result, the SDE-based stochasticization used in Flow-GRPO cannot be directly transferred to flow maps, which calls for a different stochasticization mechanism tailored to long-range flow-map transitions.

\subsection{Anchored Stochastic Flow Map Composition}\label{sec:anchor}
Our Anchored Stochastic Flow Map Composition (ASFMC) addresses this problem by transporting the deterministic output at time $r$ to an anchor at time $\tau$, then sampling back to $r$ from the conditional distribution $p_{r\mid\tau}$. Specifically, it defines the randomized flow map $\widetilde\psi_{t\to r}(X_t)=\widetilde X_r$ through
\begin{equation}
    X_r=\psi_{t\to r}(X_t),\qquad
    X_{\mathrm{anc}}=X_\tau=\psi_{r\to\tau}(X_r),\qquad
    \widetilde X_r\sim p_{r\mid\tau}(\cdot\mid X_\tau).
    \label{eq:anchor}
\end{equation}
The deterministic segment $r\to\tau$ produces an anchor with marginal $p_\tau$, and the conditional transition $\tau\to r$ introduces randomness while returning to the target marginal $p_r$.

\begin{theorem}[Path preservation of ASFMC]\label{thm:asfmc_path_preservation}
Let $X_t\sim p_t$ and define $X_r=\psi_{t\to r}(X_t)$ and $X_\tau=\psi_{r\to\tau}(X_r)$ using accurate flow maps. Let $p_{r\mid\tau}$ be a reverse conditional distribution whose joint law has marginals $p_r$ and $p_\tau$. If $\widetilde X_r\sim p_{r\mid\tau}(\cdot\mid X_\tau)$, then $\widetilde X_r\sim p_r$. Thus, the stochastic flow map induced by ASFMC preserves the marginal distribution of $\psi_{t\to r}$. (See Appendix~\ref{app:local_derivation} for proof.)
\end{theorem}

We describe two anchor choices below. The transition ablations (Section~\ref{sec:ablation}) and midpoint-anchor results (Appendix~\ref{app:midpoint}) demonstrate the practical importance of preserving the target marginal.

\begin{figure}[t]
\centering
\begin{minipage}[t]{0.32\linewidth}
\centering\includegraphics[width=\linewidth]{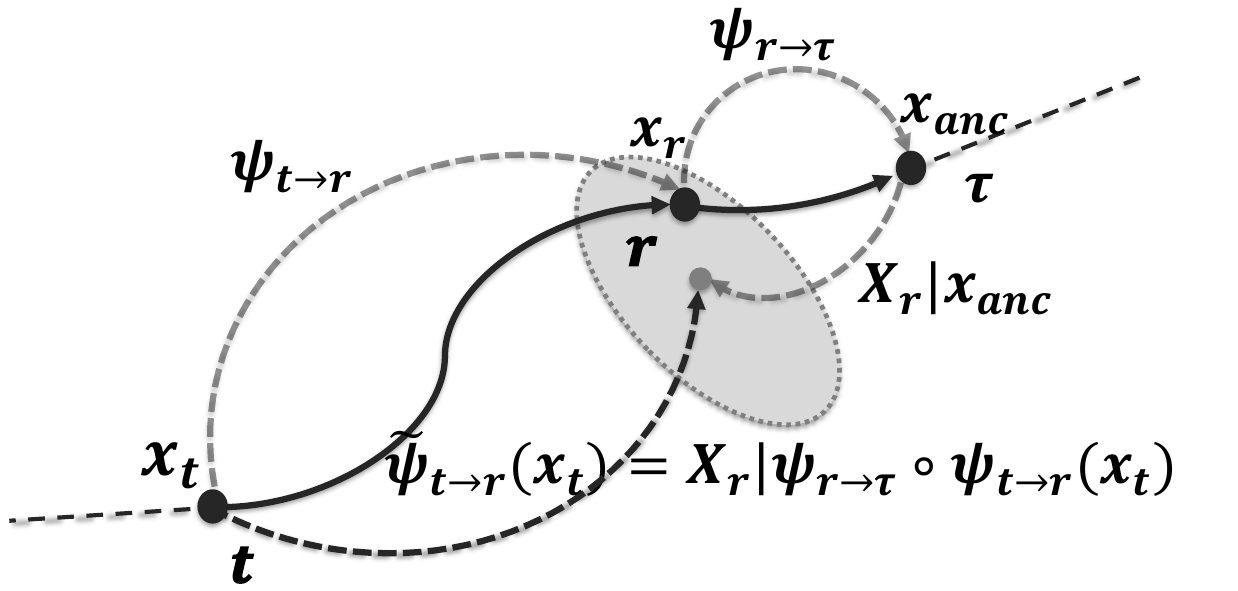}\\
{\small (a) ASFMC principle}
\end{minipage}\hfill
\begin{minipage}[t]{0.32\linewidth}
\centering\includegraphics[width=\linewidth]{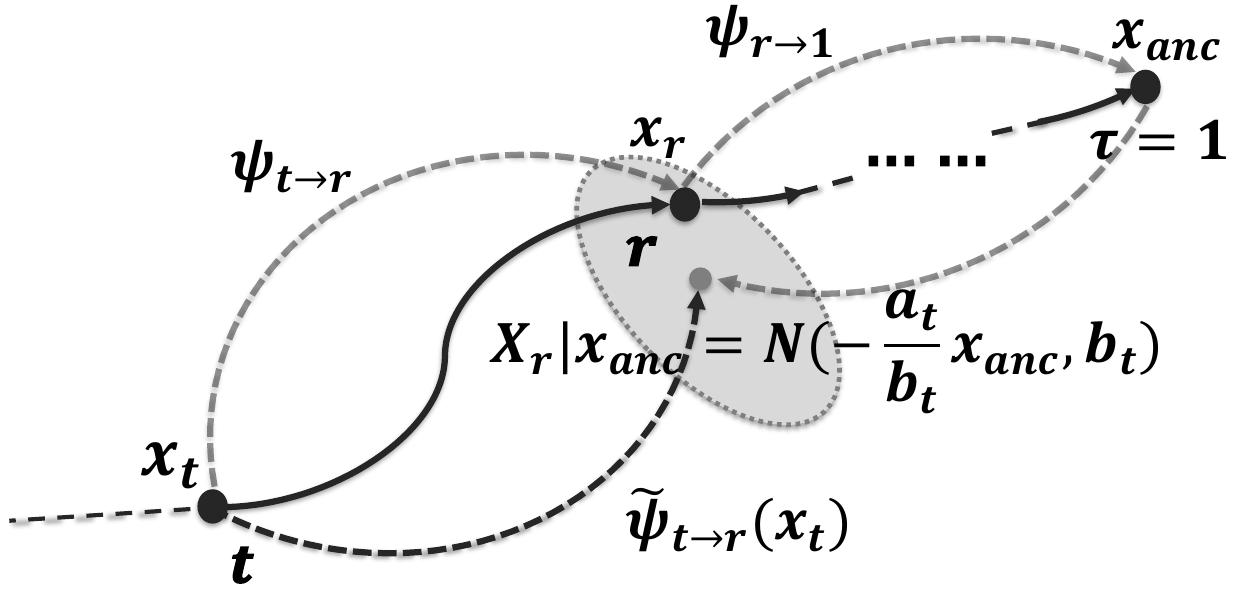}\\
{\small (b) Endpoint anchor}
\end{minipage}\hfill
\begin{minipage}[t]{0.32\linewidth}
\centering\includegraphics[width=\linewidth]{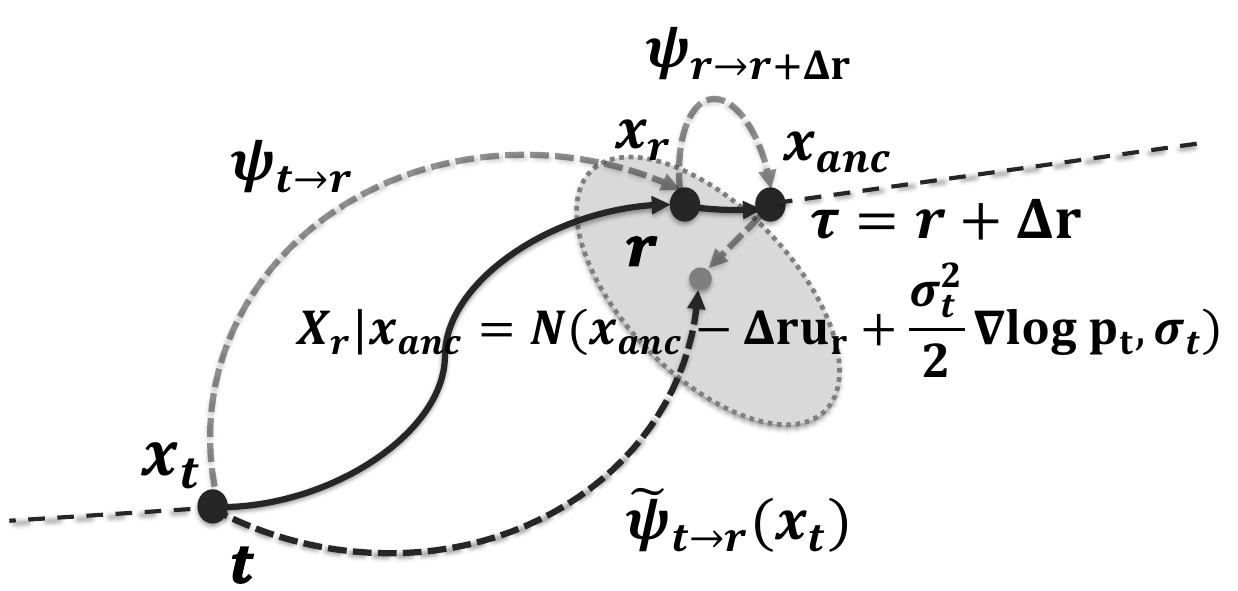}\\
{\small (c) Local anchor}
\end{minipage}
\vspace{-3mm}
\caption{Anchored stochastic composition combines deterministic transport with conditional resampling. The construction transports to an anchor and samples back to the target time. A nearby anchor yields a local drift correction, while a clean endpoint prediction supplies the anchor for scheduled re-noising. Both instantiations define tractable Gaussian policies for the same GRPO update.}
\vspace{-5mm}
\label{fig:asfmc_overview}
\end{figure}

\paragraph{Local anchor.}\label{sec:two_time}
The local anchor chooses $\tau=r+h$ for small $h>0$, with $r+h<1$. Over this short interval, the reverse-SDE conditional transition from the anchor $X_\tau$ to time $r$ admits a local Gaussian approximation. To sample back from the anchor, we reverse the path-preserving SDE, obtaining the drift $u_s-\tfrac12\sigma_s^2q_s$ with $s$ running from $\tau$ down to $r$ (Appendix~\ref{app:local_derivation}). We choose $\sigma_s^2=2\lambda^2b_s(\dot a_sb_s-a_s\dot b_s)/a_s$, which recovers the noise schedule of~\citet{liu2026flow} for the linear path, up to the noise-scale parameterization. Discretizing this reverse step and substituting the affine-path score identity gives
\begin{equation}
\begin{aligned}
    \widetilde X_r^\star
    &=X_\tau-h\left[u_\tau(X_\tau)-\tfrac12\sigma_\tau^2q_\tau(X_\tau)\right]
      +\sigma_\tau\sqrt h\,\xi+O_{L^2}(h^{3/2})\\
    &=\left(1-h\lambda^2\frac{\dot a_\tau}{a_\tau}\right)X_\tau
      -h(1-\lambda^2)u_\tau(X_\tau)+\sigma_\tau\sqrt h\,\xi
      +O_{L^2}(h^{3/2}),
\end{aligned}
    \label{eq:local_anchor_reverse_expansion}
\end{equation}
where $\xi\sim\N(0,I)$. For the deterministic segment $r\to\tau$, Equation~\eqref{eq:ode} gives $X_\tau=\psi_{r\to r+h}(X_r)=X_r+hu_r(X_r)+O(h^2)$, with $X_r=\psi_{t\to r}(X_t)$. Substituting this expansion and evaluating the smooth coefficients at $r$ cancels the leading forward and reverse velocity terms. Using model predictions yields the local-anchor stochastic flow map
\begin{equation}
    \widetilde\psi^{\mathrm{loc}}_{t\to r}(X_t)=\widetilde X_r
    =X_r+h\lambda^2\left[u_r^\theta(X_r)-\frac{\dot a_r}{a_r}X_r\right]
      +\sigma_r\sqrt h\,\xi,
    \qquad\xi\sim\N(0,I),
    \label{eq:local_anchor_update}
\end{equation}
where $X_r=\psi^\theta_{t\to r}(X_t)$ and $u_r^\theta$ is the instantaneous-velocity prediction. For a velocity-parameterized flow map $\psi^\theta_{t\to r}(x)=x+(r-t)u^\theta_{t\to r}(x)$, this prediction is obtained directly by equal-time evaluation, $u_r^\theta(x)=u^\theta_{r\to r}(x)$. The anchor construction thus yields a drift correction paired with Gaussian noise. The following theorem bounds the local approximation error. See Appendix~\ref{app:local_derivation} for proof.

\begin{theorem}[Error of the local-anchor approximation]\label{thm:local_anchor_error}
Let $X_r=\psi_{t\to r}(X_t)$ and $X_{r+h}=\psi_{r\to r+h}(X_r)$ use exact transport. Let $\widetilde X_r^\star$ be the exact reverse-SDE sample from $X_{r+h}$ to time $r$, and let $\widetilde X_r^{\mathrm{loc}}$ be Equation~\eqref{eq:local_anchor_update} with exact velocity. Couple both samples using the same Brownian increment, $\xi=W_h/\sqrt h$. Under the smoothness and moment assumptions in Appendix~\ref{app:local_derivation},
\begin{equation}
    \|\widetilde X_r^\star-\widetilde X_r^{\mathrm{loc}}\|_{L^2}=O(h^{3/2})
    \qquad\text{as }h\to0.
    \label{eq:local_anchor_error}
\end{equation}

\end{theorem}

\paragraph{Endpoint anchor.}\label{sec:endpoint_anchor}\label{sec:one_time}
The endpoint anchor chooses $\tau=1$. Starting from $X_r=\psi^\theta_{t\to r}(X_t)$, the anchor is obtained by transporting $X_r$ to the data endpoint, $X_1=\psi^\theta_{r\to1}(X_r)$. For the affine probability path, $X_r=a_rX_1+b_rX_0$, with $X_0\sim\N(0,I)$ independent of $X_1$. The endpoint conditional therefore has the closed form $p_{r\mid1}(\cdot\mid X_1)=\N(a_rX_1,b_r^2I)$. Resampling from this conditional gives
\begin{equation}
    \widetilde\psi^{\mathrm{end}}_{t\to r}(X_t)
    =a_r\psi^\theta_{r\to1}\bigl(\psi^\theta_{t\to r}(X_t)\bigr)+b_r\xi,
    \qquad\xi\sim\N(0,I).
    \label{eq:one_time_anchor}
\end{equation}
Unlike the local anchor, it resamples from the endpoint conditional rather than using a short local update. Its quality consequently depends on the accuracy of the long-range endpoint transport $\psi^\theta_{r\to1}$.  Both the local- and endpoint-anchor constructions yield Gaussian policies with explicit means and variances, allowing the same likelihood-ratio training procedure.

\begin{figure}[!t]
\centering
\vspace{-3mm}
\includegraphics[width=0.9\linewidth]{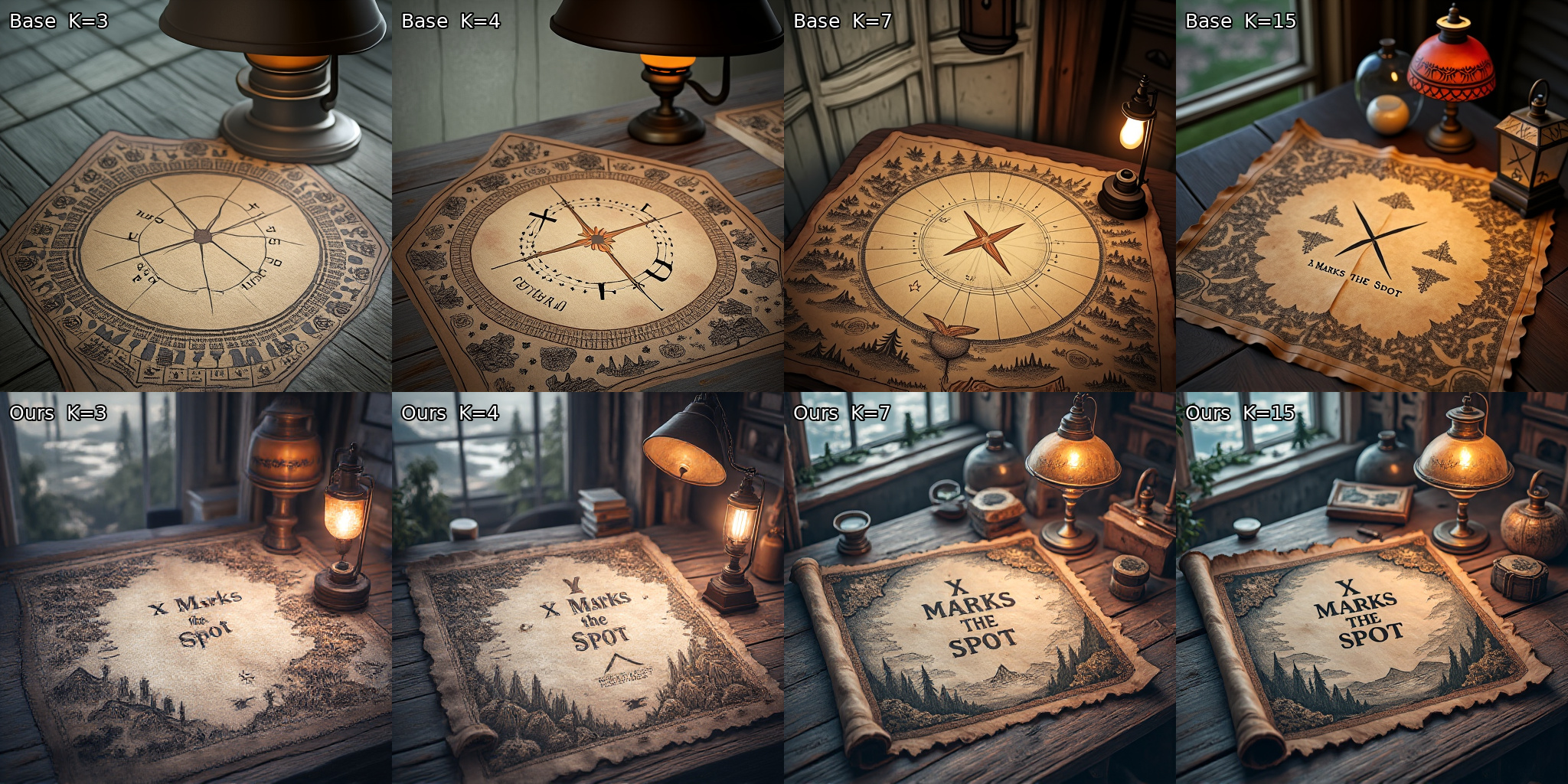}
\vspace{-3mm}
\caption{MeanFlow before (top) and after (bottom) joint OCR--PickScore post-training. Prompt: ``A weathered treasure map laid out on an old wooden table, with `X Marks the Spot' clearly visible in the center, surrounded by intricate illustrations of mountains, forests, and a distant coastline, all under the warm glow of a vintage lamp.'' Lettering, material detail, and lighting are improved}
\vspace{-6mm}
\label{fig:meanflow_mixed_qual}
\end{figure}

\paragraph{Midpoint anchor.}\label{sec:midpoint_anchor}
We use \emph{midpoint anchor} for an anchor time $r<\tau<1$, without requiring the arithmetic midpoint. Unlike an endpoint anchor, it retains information about both data and noise. The following proposition explains why direct Gaussian re-noising does not suffice.
\begin{proposition}[Posterior constraint for midpoint anchors]\label{prop:midpoint_anchor}
Let $X_s=a_sX_1+b_sX_0$, where $X_0\sim\N(0,I)$ is independent of $X_1$. For $r<\tau<1$ with $a_\tau\ne0$, define $c_{r,\tau}=b_r-a_rb_\tau/a_\tau$. Then
\begin{equation}
 \frac{a_r}{a_\tau}x+c_{r,\tau}X_0^{\mathrm{post}}
 \sim p_{r\mid\tau}(\cdot\mid x),
 \qquad X_0^{\mathrm{post}}\sim p_{0\mid\tau}(\cdot\mid x).
 \label{eq:midpoint_conditional}
\end{equation}
If $c_{r,\tau}\ne0$, replacing $X_0^{\mathrm{post}}$ with independent standard Gaussian noise is exact if and only if $p_{0\mid\tau}(\cdot\mid x)=p_0$. For the linear path and data with finite second moment, this substitution fails to recover the conditional kernel for every $0\le r<\tau<1$.  (see Appendix~\ref{app:midpoint} for proof.)
\end{proposition}
A midpoint anchor thus requires posterior sampling or an approximation to it; simply reusing the endpoint noise rule omits this dependence. Our local and endpoint constructions provide tractable alternatives without introducing a separate posterior sampler. Appendix~\ref{app:midpoint} reports the corresponding midpoint-anchor experiment: OCR gains remain small while all five DrawBench metrics degrade (Tab.~\ref{tab:midpoint_sweep}), and stochastic training rollouts exhibit visible artifacts (Fig.~\ref{fig:midpoint_rollouts}).

\vspace{-1mm}
\subsection{Flow-Map GRPO}\label{sec:flowmap_grpo}
\vspace{-1mm}
We divide sampling into $K$ stochastic transitions and one terminal deterministic transition, using $0=r_0<\cdots<r_K=r_{\mathrm{end}}<r_{K+1}=1$. The state at step $j$ is $s_j=(c,r_{j-1},r_j,x_{r_{j-1}})$, the action is $x_{r_j}$, and the transition policy is $\pi_\theta(x_{r_j}\mid s_j)=\N\!\left(\mu_j^\theta(s_j),\nu_j^2I\right)$. Writing $X_{r_j}=\psi^\theta_{r_{j-1}\to r_j}(x_{r_{j-1}})$, the local anchor gives $\mu_j^\theta=X_{r_j}+h_j\lambda^2[u_{r_j}^\theta(X_{r_j})-(\dot a_{r_j}/a_{r_j})X_{r_j}]$ and $\nu_j^2=\sigma_{r_j}^2h_j$, where $h_j=0.01(r_j-r_{j-1})$ is the local anchor interval (Equation~\eqref{eq:local_anchor_update}). The endpoint anchor gives $\mu_j^\theta=a_{r_j}\psi^\theta_{r_j\to1}(X_{r_j})$ and $\nu_j^2=b_{r_j}^2$ (Equation~\eqref{eq:one_time_anchor}). After the stochastic steps, we apply the final endpoint map and decode the image. The reward is computed on the final decoded image, and the deterministic terminal transition is excluded from the policy-ratio terms.

For each condition $c$, we generate $G$ trajectories and compute their rewards $R_i$. The normalized advantage is $A_i=(R_i-\mu_c)/(\sigma_{\mathrm{global}}+\epsilon_{\mathrm{std}})$, and the clipped advantage is $\widehat A_i=\operatorname{clip}(A_i,-A_{\max},A_{\max})$. Here $\mu_c=\frac{1}{G}\sum_{k=1}^{G}R_k$. The quantity $\sigma_{\mathrm{global}}$ is the reward standard deviation over the rollout batch, $\epsilon_{\mathrm{std}}$ stabilizes the denominator, and $A_{\max}$ is the advantage clipping threshold; Appendix~\ref{appendix:t2i_details} specifies the statistics and settings. For trajectory $i$ at step $j$, the likelihood ratio is $\rho_{ij}=\pi_\theta(x_{r_j}^{(i)}\mid s_j^{(i)})/\pi_{\theta_{\mathrm{old}}}(x_{r_j}^{(i)}\mid s_j^{(i)})$, where $\theta_{\mathrm{old}}$ denotes the policy parameters used to collect the rollout. The clipped GRPO objective is
\begin{equation}
    J(\theta)=\E\!\left[\frac{1}{GK}\sum_{i=1}^G\sum_{j=1}^K
    \min\!\left(\rho_{ij}\widehat A_i,
    \operatorname{clip}(\rho_{ij},1-\epsilon,1+\epsilon)\widehat A_i\right)\right].
    \label{eq:FM_GRPO}
\end{equation}
\begin{wrapfigure}[22]{r}{0.53\textwidth}
\vspace{-7mm}
\vspace{-\intextsep}
\begin{minipage}{\linewidth}
\begin{algorithm}[H]
\caption{Flow-Map GRPO with ASFMC}
\label{alg:flowmap_grpo}
\small
\algrenewcommand\algorithmicindent{0.8em}
\begin{algorithmic}[1]
\Require Generator $\theta$, grid $\{r_j\}$, group size $G$, reward $R$, learning rate $\eta_{\mathrm{opt}}$
\While{training steps remain}
\State $\theta_{\mathrm{old}}\gets\theta$; sample conditions $c$
\For{each $c$ and $i=1,\ldots,G$}
\State Sample $x^{(i)}_{r_0}\sim\N(0,I)$
\For{$j=1,\ldots,K$}
\State Form state $s_j^{(i)}$
\If{local policy}
\State Predict transport and local velocity
\State Obtain $(\mu_j,\nu_j)$ from (\ref{eq:local_anchor_update})
\Else
\State Predict endpoint $X_1$
\State $\mu_j=a_{r_j}X_1$, $\nu_j=b_{r_j}$
\EndIf
\State Sample $x^{(i)}_{r_j}\sim\N(\mu_j,\nu_j^2I)$
\State Store $s_j^{(i)}$, $x_{r_j}^{(i)}$, $\log\pi_{\theta_{\mathrm{old}}}(x_{r_j}^{(i)}\mid s_j^{(i)})$
\EndFor
\State Predict final endpoint; decode $I_i$
\EndFor
\State Compute rewards $R_i$ and advantages $\widehat A_i$
\State Recompute $\log\pi_\theta(x_{r_j}^{(i)}\mid s_j^{(i)})$
\State $\theta\gets\theta+\eta_{\mathrm{opt}}\nabla_\theta J(\theta)$ using (\ref{eq:FM_GRPO})
\EndWhile
\end{algorithmic}
\end{algorithm}
\end{minipage}
\end{wrapfigure}
We update LoRA adapters while keeping the pretrained weights frozen. Actions and rewards are fixed during each update. Appendix~\ref{app:contents_4} specifies the likelihood normalization and gradient conventions in optimization.

Algorithm~\ref{alg:flowmap_grpo} summarizes the procedure. The endpoint anchor also applies to one-time models, which fix the target time to the data endpoint $1$ and learn $f_t^\theta(x_t)=\psi^\theta_{t\to1}(x_t)$. For the endpoint anchor, accurate flow maps satisfy $\psi_{r\to1}\circ\psi_{t\to r}=\psi_{t\to1}$. The one-time model therefore predicts the result of these two transport segments directly, supplying $X_1=f_{r_{j-1}}^\theta(x_{r_{j-1}})$ in Equation~\eqref{eq:one_time_anchor}. Its Gaussian policy therefore uses $\mu_j^\theta=a_{r_j}f_{r_{j-1}}^\theta(x_{r_{j-1}})$ and $\nu_j^2=b_{r_j}^2$, and is optimized with the same GRPO objective. The parameter settings for MeanFlow and sCM are provided in Appendix~\ref{appendix:t2i_details}. Evaluation can use 3 to 16 sampling steps, independently of the $K+1$ steps during training (Section~\ref{sec:experiment}).

Flow-Map GRPO supports both image-level and image-set rewards, whether differentiable or non-differentiable, because policy-gradient optimization does not require gradients through the reward. For image-level rewards such as OCR, PickScore, and GenEval, each trajectory receives the score of its generated image, $R_i=r(I_i\mid c)$. For an image-set reward $\mathcal R(\mathbb{X}_g;\mathbb{X}_r)$, where $\mathbb{X}_g=\{I_1,\ldots,I_G\}$ is the generated group and $\mathbb{X}_r$ is a reference set for condition $c$, we follow \citet{miao2024diverse} and assign each trajectory its marginal contribution, $R_i=\mathcal R(\mathbb{X}_g;\mathbb{X}_r)-\mathcal R(\mathbb{X}_g\setminus\{I_i\};\mathbb{X}_r)$. Both reward types then use the same advantage computation and GRPO objective. Our ImageNet experiments use the negative MMD-squared estimate as an image-set reward.

\section{Experiments}\label{sec:experiment}
We conduct ablation studies on ImageNet and evaluate Flow-Map GRPO on FLUX-based MeanFlow and sCM text-to-image generators with single and mixed rewards. We use $K$ for the number of transitions before the final endpoint step, so sampling takes $K+1$ steps. Training uses $K=4$, corresponding to five sampling steps. In the appendix, we also report results at other inference step counts, demonstrating flexibility to use different numbers of steps during training and inference.

\subsection{Ablation Study}\label{sec:ablation}
We post-train MeanFlow and CM on ImageNet using LoRA adaptation and the maximum mean discrepancy (MMD) image-set reward of \citet{miao2024diverse}. All variants use the same training and evaluation settings. Appendix~\ref{appendix:ablation_details} gives the configurations. We evaluate generation quality using Fr\'echet Inception Distance (FID), Precision, and Recall. Lower FID indicates closer agreement between generated and real-image distributions, while higher Precision and Recall indicate better sample fidelity and broader coverage of the real-image distribution, respectively.

\begin{wraptable}{R}{0.56\textwidth}
\vspace{-4mm}
\vspace{-\intextsep}
\centering\small
\caption{Five-step ablations at epoch 400. MeanFlow uses deterministic sampling; CM uses native stochastic sampling. Each uses its own Base.}
\label{tab:meanflow_drift_ablation}
\setlength{\tabcolsep}{3pt}
\begin{tabular}{lrrr}
\toprule
Policy & FID $\downarrow$ & Prec. $\uparrow$ & Rec. $\uparrow$ \\
\midrule
\multicolumn{4}{l}{\textit{MeanFlow}}\\
Base & 25.43 & 0.591 & 0.340 \\
Ours & \textbf{16.72} & \textbf{0.774} & \textbf{0.421} \\
Drift + const. noise & 32.71 & 0.609 & 0.214 \\
No drift & 70.13 & 0.269 & 0.195 \\
No drift + const. noise & 49.25 & 0.384 & 0.102 \\
\midrule
\multicolumn{4}{l}{\textit{CM}}\\
Base & 25.69 & 0.567 & 0.381 \\
Ours & \textbf{15.65} & \textbf{0.784} & \textbf{0.438} \\
Constant noise & 42.16 & 0.460 & 0.107 \\
\bottomrule
\end{tabular}
\vspace{-3mm}
\end{wraptable}

\paragraph{Local-anchor ablation study.}
We ablate the drift correction $h\lambda^2[u_r^\theta(X_r)-(\dot a_r/a_r)X_r]$ and noise scale $\sigma_r\sqrt h$ in Equation~\eqref{eq:local_anchor_update}, keeping the transport and training settings fixed. \emph{Ours} retains both terms as derived from the local anchor. \emph{Drift + const.\ noise} retains the drift correction but replaces the scheduled noise scale with a constant matched to the total injected variance. \emph{No drift} removes the entire drift correction while retaining the scheduled noise. \emph{No drift + const.\ noise} removes the drift correction and uses the same variance-matched constant noise scale. The two no-drift variants thus directly add Gaussian noise to the deterministic MeanFlow output, using scheduled or constant noise, respectively. Only the correct form derived from the local anchor improves FID, Precision, and Recall together (Tab.~\ref{tab:meanflow_drift_ablation}), supporting the necessity of our method. Simply adding noise to the deterministic flow-map output is insufficient.  See Fig.~\ref{fig:ablation_curves} for learning dynamics.

\paragraph{Endpoint-anchor ablation study.}
We replace the noise coefficient $b_r$ of our endpoint-anchor policy in Equation~\eqref{eq:one_time_anchor} with a variance-matched constant, keeping the mean $a_rX_1$ unchanged. Our endpoint-anchor policy (\emph{Ours}) improves fidelity and coverage under the native CM sampler, whereas constant noise reverses these gains, supporting the necessity of the correct endpoint-anchor formulation. The learning dynamics are shown in Fig.~\ref{fig:ablation_curves}.

\begin{figure}[t]
\centering
\vspace{-3mm}
\begin{tikzpicture}
\begin{axis}[at={(0.000\textwidth,0)},anchor=south west,scale only axis,width=0.27\textwidth,height=3.0cm,xmin=0,xmax=400,xtick={0,200,400},xlabel={Epoch},xlabel style={yshift=2pt},title style={yshift=-2pt},grid=major,grid style={gray!20},tick label style={font=\scriptsize},label style={font=\small},title style={font=\small},scaled ticks=false,title={FID $\downarrow$},ymin=10,ymax=75,ytick={20,40,60}]
\addplot[thick,blue!75!black,solid,mark=*,mark size=1.25pt] coordinates {(0,25.43) (25,23.95) (50,22) (75,20.15) (100,18.45) (125,18.1) (150,18.14) (175,17.37) (200,16.86) (225,16.74) (250,16.83) (275,17.24) (300,17.08) (325,16.69) (350,16.71) (375,16.45) (400,16.72)};
\addplot[thick,orange!90!black,dashed,mark=square*,mark size=1.25pt] coordinates {(0,25.43) (25,29.13) (50,33.74) (75,33.43) (100,33.79) (125,33.35) (150,33.82) (175,33.24) (200,32.64) (225,32.94) (250,33.29) (275,32.16) (300,31.25) (325,32.61) (350,33.54) (375,32.79) (400,32.71)};
\addplot[thick,red!75!black,densely dotted,mark=triangle*,mark size=1.25pt] coordinates {(0,25.43) (25,25.56) (50,32.7) (75,44.58) (100,50.99) (125,61.32) (150,60.5) (175,60.47) (200,64.43) (225,65.33) (250,69.73) (275,68.14) (300,66.04) (325,70.09) (350,68.86) (375,71.55) (400,70.13)};
\addplot[thick,violet!80!black,dashdotted,mark=diamond*,mark size=1.25pt] coordinates {(0,25.43) (25,32.08) (50,47.7) (75,54.84) (100,50.65) (125,52.39) (150,50.74) (175,51.73) (200,48.81) (225,46.71) (250,48.96) (275,49.62) (300,45.78) (325,44.9) (350,45.26) (375,49.05) (400,49.25)};
\addplot[thick,teal!80!black,dashed,mark=o,mark size=1.25pt] coordinates {(0,25.69) (25,23.57) (50,21.93) (75,19.53) (100,17.59) (125,16.86) (150,16.54) (175,16.33) (200,15.88) (225,15.96) (250,16.1) (275,16.07) (300,16.05) (325,16.04) (350,15.93) (375,15.59) (400,15.65)};
\addplot[thick,brown!85!black,dashdotted,mark=square,mark size=1.25pt] coordinates {(0,25.69) (25,26.12) (50,36.45) (75,43.6) (100,47.85) (125,46.24) (150,44.23) (175,46.71) (200,44.11) (225,45.99) (250,46.13) (275,44.67) (300,46.2) (325,45.08) (350,42.98) (375,40.07) (400,42.16)};
\end{axis}
\begin{axis}[at={(0.335\textwidth,0)},anchor=south west,scale only axis,width=0.27\textwidth,height=3.0cm,xmin=0,xmax=400,xtick={0,200,400},xlabel={Epoch},xlabel style={yshift=2pt},title style={yshift=-2pt},grid=major,grid style={gray!20},tick label style={font=\scriptsize},label style={font=\small},title style={font=\small},scaled ticks=false,title={Precision $\uparrow$},ymin=0.2,ymax=0.85,ytick={0.2,0.4,0.6,0.8}]
\addplot[thick,blue!75!black,solid,mark=*,mark size=1.25pt] coordinates {(0,0.591) (25,0.639) (50,0.697) (75,0.74) (100,0.77) (125,0.767) (150,0.765) (175,0.77) (200,0.77) (225,0.779) (250,0.775) (275,0.758) (300,0.75) (325,0.755) (350,0.764) (375,0.766) (400,0.774)};
\addplot[thick,orange!90!black,dashed,mark=square*,mark size=1.25pt] coordinates {(0,0.591) (25,0.638) (50,0.582) (75,0.6) (100,0.606) (125,0.615) (150,0.59) (175,0.592) (200,0.613) (225,0.623) (250,0.612) (275,0.604) (300,0.623) (325,0.615) (350,0.599) (375,0.593) (400,0.609)};
\addplot[thick,red!75!black,densely dotted,mark=triangle*,mark size=1.25pt] coordinates {(0,0.591) (25,0.649) (50,0.577) (75,0.444) (100,0.407) (125,0.341) (150,0.336) (175,0.342) (200,0.32) (225,0.304) (250,0.287) (275,0.285) (300,0.293) (325,0.279) (350,0.285) (375,0.263) (400,0.269)};
\addplot[thick,violet!80!black,dashdotted,mark=diamond*,mark size=1.25pt] coordinates {(0,0.591) (25,0.551) (50,0.411) (75,0.369) (100,0.393) (125,0.373) (150,0.39) (175,0.392) (200,0.411) (225,0.429) (250,0.407) (275,0.403) (300,0.426) (325,0.428) (350,0.41) (375,0.379) (400,0.384)};
\addplot[thick,teal!80!black,dashed,mark=o,mark size=1.25pt] coordinates {(0,0.567) (25,0.62) (50,0.649) (75,0.679) (100,0.717) (125,0.742) (150,0.761) (175,0.763) (200,0.763) (225,0.761) (250,0.764) (275,0.762) (300,0.756) (325,0.752) (350,0.767) (375,0.776) (400,0.784)};
\addplot[thick,brown!85!black,dashdotted,mark=square,mark size=1.25pt] coordinates {(0,0.567) (25,0.6) (50,0.505) (75,0.447) (100,0.417) (125,0.435) (150,0.439) (175,0.418) (200,0.436) (225,0.414) (250,0.411) (275,0.429) (300,0.439) (325,0.439) (350,0.458) (375,0.49) (400,0.46)};
\end{axis}
\begin{axis}[at={(0.670\textwidth,0)},anchor=south west,scale only axis,width=0.27\textwidth,height=3.0cm,xmin=0,xmax=400,xtick={0,200,400},xlabel={Epoch},xlabel style={yshift=2pt},title style={yshift=-2pt},grid=major,grid style={gray!20},tick label style={font=\scriptsize},label style={font=\small},title style={font=\small},scaled ticks=false,title={Recall $\uparrow$},ymin=0.05,ymax=0.5,ytick={0.1,0.2,0.3,0.4,0.5}]
\addplot[thick,blue!75!black,solid,mark=*,mark size=1.25pt] coordinates {(0,0.34) (25,0.318) (50,0.317) (75,0.36) (100,0.371) (125,0.401) (150,0.394) (175,0.422) (200,0.415) (225,0.421) (250,0.428) (275,0.429) (300,0.434) (325,0.447) (350,0.444) (375,0.435) (400,0.421)};
\addplot[thick,orange!90!black,dashed,mark=square*,mark size=1.25pt] coordinates {(0,0.34) (25,0.184) (50,0.158) (75,0.145) (100,0.12) (125,0.153) (150,0.184) (175,0.217) (200,0.188) (225,0.172) (250,0.17) (275,0.176) (300,0.174) (325,0.168) (350,0.17) (375,0.191) (400,0.214)};
\addplot[thick,red!75!black,densely dotted,mark=triangle*,mark size=1.25pt] coordinates {(0,0.34) (25,0.242) (50,0.127) (75,0.106) (100,0.089) (125,0.101) (150,0.105) (175,0.108) (200,0.13) (225,0.168) (250,0.144) (275,0.186) (300,0.172) (325,0.166) (350,0.191) (375,0.174) (400,0.195)};
\addplot[thick,violet!80!black,dashdotted,mark=diamond*,mark size=1.25pt] coordinates {(0,0.34) (25,0.165) (50,0.114) (75,0.121) (100,0.094) (125,0.078) (150,0.082) (175,0.093) (200,0.093) (225,0.089) (250,0.105) (275,0.128) (300,0.118) (325,0.11) (350,0.11) (375,0.113) (400,0.102)};
\addplot[thick,teal!80!black,dashed,mark=o,mark size=1.25pt] coordinates {(0,0.381) (25,0.37) (50,0.401) (75,0.412) (100,0.465) (125,0.452) (150,0.452) (175,0.442) (200,0.451) (225,0.457) (250,0.454) (275,0.467) (300,0.459) (325,0.461) (350,0.432) (375,0.443) (400,0.438)};
\addplot[thick,brown!85!black,dashdotted,mark=square,mark size=1.25pt] coordinates {(0,0.381) (25,0.301) (50,0.168) (75,0.147) (100,0.114) (125,0.115) (150,0.112) (175,0.111) (200,0.119) (225,0.109) (250,0.114) (275,0.106) (300,0.096) (325,0.095) (350,0.088) (375,0.092) (400,0.107)};
\end{axis}
\end{tikzpicture}\par\vspace{0.5ex}
{\scriptsize\setlength{\tabcolsep}{5pt}
\begin{tabular}{lll}
\tikz[baseline=-0.5ex]\draw[thick,blue!75!black,solid] plot[mark=*,mark size=1.25pt,mark indices={2}] coordinates {(0,0) (0.175,0) (0.35,0)}; MeanFlow: Ours & \tikz[baseline=-0.5ex]\draw[thick,orange!90!black,dashed] plot[mark=square*,mark size=1.25pt,mark indices={2}] coordinates {(0,0) (0.175,0) (0.35,0)}; MeanFlow: drift + const. noise & \tikz[baseline=-0.5ex]\draw[thick,red!75!black,densely dotted] plot[mark=triangle*,mark size=1.25pt,mark indices={2}] coordinates {(0,0) (0.175,0) (0.35,0)}; MeanFlow: no drift \\
\tikz[baseline=-0.5ex]\draw[thick,violet!80!black,dashdotted] plot[mark=diamond*,mark size=1.25pt,mark indices={2}] coordinates {(0,0) (0.175,0) (0.35,0)}; MeanFlow: no drift + const. noise & \tikz[baseline=-0.5ex]\draw[thick,teal!80!black,dashed] plot[mark=o,mark size=1.25pt,mark indices={2}] coordinates {(0,0) (0.175,0) (0.35,0)}; CM: Ours & \tikz[baseline=-0.5ex]\draw[thick,brown!85!black,dashdotted] plot[mark=square,mark size=1.25pt,mark indices={2}] coordinates {(0,0) (0.175,0) (0.35,0)}; CM: constant noise \\
\end{tabular}}
\vspace{-3mm}
\caption{Learning dynamics of Ours and ablation variants, measured by FID, Precision, and Recall. MeanFlow compares Ours with No drift, Drift + const.\ noise, and No drift + const.\ noise; CM compares Ours with constant noise.}
\label{fig:ablation_curves}
\vspace{-5mm}
\end{figure}
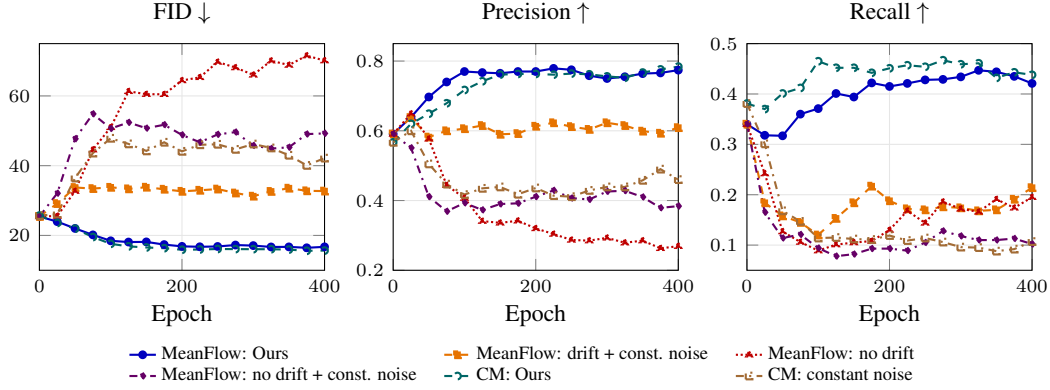

\subsection{Single-Score Optimization}\label{sec:single_score}
\paragraph{Rewards.}
We post-train FLUX.1-lite MeanFlow and sCM~\citep{pu2025few} separately with OCR and PickScore rewards. OCR measures text-rendering accuracy, and PickScore measures human image preference. Training uses five sampling steps; Appendix~\ref{appendix:t2i_details} gives the settings. Appendix~\ref{app:meanflow_anchor_choice} motivates the local-anchor choice for MeanFlow in this T2I setting.

\begin{table}[t]
\centering\small
\vspace{-5mm}
\caption{OCR-only and PickScore-only post-training at $K=4$. Each reward block compares Base with a separately trained Flow-Map GRPO adapter. Task Metrics measure RL improvements on the target tasks using their respective test sets. DrawBench Metrics are auxiliary general-quality metrics used to monitor degradation outside the optimized reward. Evaluation details are in Appendix~\ref{appendix:t2i_details}.}
\label{tab:single_summary}
\label{tab:single_meanflow}
\label{tab:single_scm}
\setlength{\tabcolsep}{1.4pt}
\renewcommand{\arraystretch}{1.08}
\resizebox{\linewidth}{!}{
\begin{tabular}{llcc@{\hspace{0.5em}}|@{\hspace{0.5em}}ccccc}
\toprule
\multirow{2}{*}{\textbf{Model}} & \multirow{2}{*}{\textbf{Reward}} & \multicolumn{2}{c}{\textbf{Task Metrics}} & \multicolumn{5}{c}{\textbf{DrawBench Metrics}} \\
\cmidrule(lr){3-4}\cmidrule(lr){5-9}
 & & \shortstack{\textbf{OCR}\\\textbf{Acc.}$\uparrow$} & \shortstack{\textbf{PickScore}\\\textbf{Task}$\uparrow$} & \shortstack{\textbf{DB-}\\\textbf{PickScore}$\uparrow$} & \textbf{Aesthetic}$\uparrow$ & \textbf{DeQA}$\uparrow$ & \textbf{ImgRwd}$\uparrow$ & \textbf{UniRwd}$\uparrow$ \\

\midrule
MeanFlow Base & \multirow{2}{*}{OCR} & 0.3491 & -- & 21.6298 & \textbf{5.5368} & \textbf{4.1712} & 0.3918 & 2.7187 \\
MeanFlow + Flow-Map GRPO LoRA &  & \textbf{0.8504} & -- & \textbf{21.9869} & 5.5111 & 4.1396 & \textbf{0.5386} & \textbf{2.8049} \\
\midrule
MeanFlow Base & \multirow{2}{*}{PickScore} & -- & 20.9758 & 21.6298 & 5.5368 & \textbf{4.1712} & 0.3918 & 2.7187 \\
MeanFlow + Flow-Map GRPO LoRA &  & -- & \textbf{23.0772} & \textbf{23.3711} & \textbf{6.1510} & 4.0829 & \textbf{1.1402} & \textbf{3.2005} \\
\midrule
sCM Base & \multirow{2}{*}{OCR} & 0.4903 & -- & 21.4993 & \textbf{5.4747} & 3.7469 & \textbf{0.5750} & 2.6725 \\
sCM + Flow-Map GRPO LoRA &  & \textbf{0.8863} & -- & \textbf{21.5647} & 5.4049 & \textbf{3.7800} & 0.3475 & \textbf{2.7137} \\
\midrule
sCM Base & \multirow{2}{*}{PickScore} & -- & 20.3517 & 21.4993 & 5.4747 & 3.7469 & 0.5750 & 2.6725 \\
sCM + Flow-Map GRPO LoRA &  & -- & \textbf{23.4820} & \textbf{23.8084} & \textbf{6.0360} & \textbf{4.2343} & \textbf{1.2600} & \textbf{3.3361} \\
\bottomrule
\end{tabular}
}
\vspace{-5mm}
\end{table}

\paragraph{Results.}
Both models improve substantially on both target rewards (Tab.~\ref{tab:single_summary}). At five steps, OCR accuracy exceeds 0.85 for both models, and PickScore increases by more than two points. PickScore optimization also improves all five DrawBench metrics for sCM and four for MeanFlow; OCR optimization shows mixed effects on these auxiliary metrics. Appendices~\ref{app:ocr_sweeps} and~\ref{app:pickscore_sweeps} show that the target-reward gains persist at every tested inference step count, demonstrating transfer from five-step training to other inference step counts without retraining.
\vspace{-2mm}
\paragraph{Comparative experiments.}
\begin{wraptable}{r}{0.47\textwidth}
\vspace{-\intextsep}
\centering\scriptsize
\caption{Single-reward post-training at $K=4$.}
\label{tab:method_comparison}
\setlength{\tabcolsep}{3pt}
\begin{tabular}{lcc@{\hspace{0.7em}}cc}
\toprule
\multirow{2}{*}{\textbf{Method}} & \multicolumn{2}{c}{\textbf{sCM}} & \multicolumn{2}{c}{\textbf{MeanFlow}} \\
\cmidrule(lr){2-3}\cmidrule(lr){4-5}
& \textbf{OCR}$\uparrow$ & \textbf{PickScore}$\uparrow$ & \textbf{OCR}$\uparrow$ & \textbf{PickScore}$\uparrow$ \\
\midrule
Base & 0.4903 & 20.3517 & 0.3491 & 20.9758 \\
SFT & 0.4588 & 20.7893 & 0.3360 & 20.9763 \\
NFT & 0.5363 & 20.2200 & 0.3576 & 20.5418 \\
\textbf{Ours} & \textbf{0.8863} & \textbf{23.4820} & \textbf{0.8503} & \textbf{23.0772} \\
\bottomrule
\end{tabular}
\vspace{-3mm}
\end{wraptable}
In the absence of an established likelihood-ratio RL baseline for pretrained MeanFlow generators, Tab.~\ref{tab:method_comparison} compares SFT and DiffusionNFT as representative existing post-training approaches. We use the same pretrained models, prompt splits, reward models, sampling settings, and number of optimizer steps; Appendix~\ref{app:baseline_objectives} details the baseline objectives and matched comparison settings. Flow-Map GRPO achieves the highest reward in all four comparisons. SFT and DiffusionNFT yield limited gains or reduce the target reward. Flow-map structure creates two challenges for these objectives: native supervision couples long-range predictions to an unadapted instantaneous field, while forward-noising regression requires local conditional predictions that flow-map outputs do not generally represent. Appendix~\ref{app:baseline_objectives} explains these objective mismatches; Appendix~\ref{appendix:ocr_training_curves} shows the training curves.
\vspace{-2mm}
\subsection{Mixed-Score Optimization}\label{sec:mixed_score}
\vspace{-1mm}
We evaluate OCR--PickScore mixed-reward optimization and GenEval compositional-reward optimization, using a separate adapter for each objective on each backbone. For OCR--PickScore, we use a mixture of OCR and PickScore prompts. For prompts with an OCR target, we use $R_{\mathrm{mix}}=0.7\widetilde R_{\mathrm{OCR}}+0.3\widetilde R_{\mathrm{Pick}}$; otherwise, we use $\widetilde R_{\mathrm{Pick}}$ alone. The $\tilde{~}$ denotes reward normalization. GenEval~\citep{ghosh2023geneval} rewards satisfaction of the objects, counts, colors, and spatial relations specified by the prompt. Appendix~\ref{appendix:t2i_details} gives the reward definitions and training configurations.
\vspace{-2mm}
\paragraph{Results.}
Tab.~\ref{tab:mixed_summary} shows that OCR--PickScore mixed-reward optimization improves both target metrics on both backbones, while GenEval optimization substantially improves compositional accuracy. Fig.~\ref{fig:meanflow_mixed_qual} illustrates the mixed-reward gains in text rendering and visual detail. Mixed-reward optimization also yields broad gains or comparable performance across the DrawBench metrics and inference step counts. GenEval optimization shows mixed effects on the auxiliary quality metrics. The mixed-reward adapters improve both tasks jointly, although their individual task scores remain below those of the corresponding single-reward adapters. Appendices~\ref{app:geneval_sweeps} and~\ref{app:mixed_sweeps} show that target-reward improvements transfer to other inference step counts.
\vspace{-1mm}
\section{Related Work}
\vspace{-2mm}
\noindent\textbf{Few-step generative models.} Consistency models and sCM learn endpoint predictors for efficient generation~\citep{song2023consistency,song2023improved,lu2025simplifying}. Shortcut models, Flow Map Matching, and MeanFlow learn long-range transitions or average velocities~\citep{frans2025shortcut,boffi2025flow,geng2025mean,geng2025improved}. These methods reduce the sampling cost of diffusion and flow models through distillation or derivative-based training. Flow-Map GRPO operates after this stage: it adapts pretrained generators to downstream rewards while retaining their prediction interface.

\begin{table}[t]
\vspace{-5mm}
\centering\small
\caption{GenEval and joint OCR--PickScore post-training at $K=4$. The objectives use separate adapters. Joint training combines OCR and PickScore prompts with the reward mixture described in Section~\ref{sec:mixed_score}. Task Metrics measure RL improvements on the target tasks using their test sets. DrawBench Metrics are general-quality metrics used to monitor degradation outside the optimized reward.}
\label{tab:mixed_summary}
\label{tab:geneval_summary}
\setlength{\tabcolsep}{1.4pt}
\renewcommand{\arraystretch}{1.08}
\resizebox{\linewidth}{!}{
\begin{tabular}{llccc@{\hspace{0.5em}}|@{\hspace{0.5em}}ccccc}
\toprule
\multirow{2}{*}{\textbf{Model}} & \multirow{2}{*}{\textbf{Reward}} & \multicolumn{3}{c}{\textbf{Task Metrics}} & \multicolumn{5}{c}{\textbf{DrawBench Metrics}} \\
\cmidrule(lr){3-5}\cmidrule(lr){6-10}
 & & \shortstack{\textbf{OCR}\\\textbf{Acc.}$\uparrow$} & \shortstack{\textbf{PickScore}\\\textbf{Task}$\uparrow$} & \textbf{GenEval}$\uparrow$ & \shortstack{\textbf{DB-}\\\textbf{PickScore}$\uparrow$} & \textbf{Aesthetic}$\uparrow$ & \textbf{DeQA}$\uparrow$ & \textbf{ImgRwd}$\uparrow$ & \textbf{UniRwd}$\uparrow$ \\
\midrule
MeanFlow Base & \multirow{2}{*}{GenEval} & -- & -- & 0.5041 & 21.6298 & \textbf{5.5368} & \textbf{4.1712} & 0.3918 & 2.7187 \\
MeanFlow + Flow-Map GRPO LoRA &  & -- & -- & \textbf{0.8454} & \textbf{21.8184} & 5.3905 & 3.2086 & \textbf{0.6426} & \textbf{2.8719} \\
\midrule
sCM Base & \multirow{2}{*}{GenEval} & -- & -- & 0.4513 & 21.4993 & \textbf{5.4747} & 3.7469 & 0.5750 & 2.6725 \\
sCM + Flow-Map GRPO LoRA &  & -- & -- & \textbf{0.8049} & \textbf{22.1046} & 5.2401 & \textbf{3.8938} & \textbf{0.6514} & \textbf{3.0835} \\
\midrule
MeanFlow Base & \multirow{2}{*}{OCR+PickScore} & 0.3491 & 20.9758 & -- & 21.6298 & 5.5368 & 4.1712 & 0.3918 & 2.7187 \\
MeanFlow + Flow-Map GRPO LoRA &  & \textbf{0.7719} & \textbf{22.4964} & -- & \textbf{23.0506} & \textbf{5.9299} & \textbf{4.1892} & \textbf{1.0439} & \textbf{3.1876} \\
\midrule
sCM Base & \multirow{2}{*}{OCR+PickScore} & 0.4903 & 20.3517 & -- & 21.4993 & 5.4747 & 3.7469 & 0.5750 & 2.6725 \\
sCM + Flow-Map GRPO LoRA &  & \textbf{0.8568} & \textbf{22.5153} & -- & \textbf{23.0403} & \textbf{5.9494} & \textbf{4.3311} & \textbf{0.9741} & \textbf{3.1004} \\
\bottomrule
\end{tabular}
}
\vspace{-5mm}
\end{table}

\noindent\textbf{Reinforcement learning for generation.}
DDPO optimizes diffusion models using stochastic denoising trajectories~\citep{black2024training}. Flow-GRPO extends online RL to velocity-based flow matching through stochastic sampling, while DanceGRPO and MixGRPO further develop reward-based visual generation and efficient rollout strategies~\citep{liu2026flow,xue2025dancegrpo,li2025mixgrpo}. RLCM applies policy-gradient RL to Gaussian consistency-sampler transitions~\citep{oertell2024rlcm}, while ROCM uses direct reward gradients with distributional regularization~\citep{shekhar2025rocm}. Our contribution is the stochastic policy construction for deterministic long-range maps, where a velocity-level SDE discretization does not directly provide the desired few-step transition.

\noindent\textbf{Stochastic and reward-aware flow maps.}
Meta Flow Maps, Diamond Maps, and Strong Stochastic Flow Maps learn stochastic mapping mechanisms for posterior sampling or reward alignment~\citep{potaptchik2026meta,holderrieth2026diamond,mccallum2026strong}. Related approaches study stochastic few-step sampling, inference-time alignment, and reward guidance~\citep{passaro2026stochastic,holderrieth2026glass,huang2026guide}. Flow-Map GRPO instead constructs a post-training policy around a pretrained deterministic generator. It introduces exploration through anchored resampling, without requiring the base model to be retrained as a native stochastic flow map.

\noindent\textbf{Concurrent work.}
Concurrent to our work, MeanFlowNFT~\citep{huang2026meanflownft} adapts DiffusionNFT's forward-process RL to MeanFlow, using the MeanFlow identity to reconcile instantaneous-velocity supervision with average-velocity sampling. RAVEN~\citep{lu2026raven} introduces CM-GRPO for video post-training through Gaussian consistency-sampler transitions. In contrast, Flow-Map GRPO constructs stochastic policies around pretrained flow-map generators and provides a likelihood-ratio RL formulation for MeanFlow-style two-time maps and consistency-style one-time predictors. See Appendix~\ref{app:related_work} for extended related work.

\vspace{-2mm}
\section{Conclusion}
\vspace{-2mm}
We introduced Flow-Map GRPO, a framework for online reward-based post-training of deterministic few-step flow-map generators. ASFMC supplies tractable local and endpoint policies that support GRPO updates while retaining the pretrained model parameterization. Controlled ablations show that the stochastic transition is crucial for transferring training improvements to generation, and text-to-image experiments demonstrate substantial gains in PickScore, OCR, and GenEval on MeanFlow and sCM. Joint OCR--PickScore optimization improves both target tasks with one adapter per backbone, while the auxiliary metrics reveal quality trade-offs that depend on the backbone and number of sampling steps. These results support anchored stochastic composition as an effective route to reward-based adaptation of pretrained few-step generators.  Regarding limitations and future work, our empirical study focuses on image generation with MeanFlow and sCM; the text-to-image results use one training seed and a fixed reward mixture. Extending the evaluation across seeds, modalities, and reward combinations will further characterize the framework's scope. Adaptive anchor selection and resampling schedules offer a natural direction for improving the quality--cost trade-off.
\label{maintext:end}
\clearpage
\bibliography{physics_informed,refs_application,refs_flowmap,refs_generative,refs_obj,refs_rl}
\bibliographystyle{preprint}
\clearpage
\appendix
\section*{Appendix Contents}
\begingroup
\setlength{\parindent}{0pt}
\setlength{\parskip}{4pt}
\textbf{\hyperref[app:contents_1]{\ref*{app:contents_1}\quad Theoretical Details and Proofs}}\nobreak\dotfill\pageref*{app:contents_1}\par
\hspace*{1em}\hyperref[app:finite_time]{\ref*{app:finite_time}\quad Proof of the long-range SDE obstruction}\nobreak\dotfill\pageref*{app:finite_time}\par
\hspace*{1em}\hyperref[app:local_derivation]{\ref*{app:local_derivation}\quad ASFMC path preservation and anchor constructions}\nobreak\dotfill\pageref*{app:local_derivation}\par
\hspace*{1em}\hyperref[app:contents_4]{\ref*{app:contents_4}\quad Gaussian likelihoods and optimization}\nobreak\dotfill\pageref*{app:contents_4}\par
\hspace*{1em}\hyperref[app:midpoint]{\ref*{app:midpoint}\quad Midpoint anchors and the posterior constraint}\nobreak\dotfill\pageref*{app:midpoint}\par
\hspace*{1em}\hyperref[app:anchor_interpretation]{\ref*{app:anchor_interpretation}\quad Unified Interpretation of Anchor Choices}\nobreak\dotfill\pageref*{app:anchor_interpretation}\par
\hspace*{1em}\hyperref[sec:appendix_one_time_two_time]{\ref*{sec:appendix_one_time_two_time}\quad Evaluation Sampling for One-Time and Two-Time Models}\nobreak\dotfill\pageref*{sec:appendix_one_time_two_time}\par
\textbf{\hyperref[app:related_work]{\ref*{app:related_work}\quad Extended Related Work}}\nobreak\dotfill\pageref*{app:related_work}\par
\textbf{\hyperref[app:contents_7]{\ref*{app:contents_7}\quad Experiment Details}}\nobreak\dotfill\pageref*{app:contents_7}\par
\hspace*{1em}\hyperref[appendix:ablation_details]{\ref*{appendix:ablation_details}\quad MeanFlow and CM Ablation Details}\nobreak\dotfill\pageref*{appendix:ablation_details}\par
\hspace*{1em}\hyperref[appendix:t2i_details]{\ref*{appendix:t2i_details}\quad Single-Score and Mixed-Score Post-Training Details}\nobreak\dotfill\pageref*{appendix:t2i_details}\par
\hspace*{1em}\hyperref[app:meanflow_anchor_choice]{\ref*{app:meanflow_anchor_choice}\quad Anchor Choice for MeanFlow Text-to-Image Post-Training}\nobreak\dotfill\pageref*{app:meanflow_anchor_choice}\par
\hspace*{1em}\hyperref[app:baseline_objectives]{\ref*{app:baseline_objectives}\quad Baseline Setting and Challenges Posed by Flow Maps}\nobreak\dotfill\pageref*{app:baseline_objectives}\par
\hspace*{1em}\hyperref[appendix:ocr_training_curves]{\ref*{appendix:ocr_training_curves}\quad OCR and PickScore Training Curves}\nobreak\dotfill\pageref*{appendix:ocr_training_curves}\par
\textbf{\hyperref[app:full_results]{\ref*{app:full_results}\quad Complete Text-to-Image Results}}\nobreak\dotfill\pageref*{app:full_results}\par
\hspace*{1em}\hyperref[app:ocr_sweeps]{\ref*{app:ocr_sweeps}\quad OCR-Only Evaluations}\nobreak\dotfill\pageref*{app:ocr_sweeps}\par
\hspace*{1em}\hyperref[app:pickscore_sweeps]{\ref*{app:pickscore_sweeps}\quad PickScore-Only Evaluations}\nobreak\dotfill\pageref*{app:pickscore_sweeps}\par
\hspace*{1em}\hyperref[app:geneval_sweeps]{\ref*{app:geneval_sweeps}\quad GenEval Evaluations}\nobreak\dotfill\pageref*{app:geneval_sweeps}\par
\hspace*{1em}\hyperref[app:mixed_sweeps]{\ref*{app:mixed_sweeps}\quad Joint OCR--PickScore Evaluations}\nobreak\dotfill\pageref*{app:mixed_sweeps}\par
\textbf{\hyperref[sec:additional]{\ref*{sec:additional}\quad Additional Qualitative Results}}\nobreak\dotfill\pageref*{sec:additional}\par
\hspace*{1em}\hyperref[app:qual_meanflow_ocr]{\ref*{app:qual_meanflow_ocr}\quad MeanFlow: OCR}\nobreak\dotfill\pageref*{app:qual_meanflow_ocr}\par
\hspace*{1em}\hyperref[app:qual_scm_ocr]{\ref*{app:qual_scm_ocr}\quad sCM: OCR}\nobreak\dotfill\pageref*{app:qual_scm_ocr}\par
\hspace*{1em}\hyperref[app:qual_meanflow_pickscore]{\ref*{app:qual_meanflow_pickscore}\quad MeanFlow: PickScore}\nobreak\dotfill\pageref*{app:qual_meanflow_pickscore}\par
\hspace*{1em}\hyperref[app:qual_scm_pickscore]{\ref*{app:qual_scm_pickscore}\quad sCM: PickScore}\nobreak\dotfill\pageref*{app:qual_scm_pickscore}\par
\hspace*{1em}\hyperref[app:qual_meanflow_geneval]{\ref*{app:qual_meanflow_geneval}\quad MeanFlow: GenEval}\nobreak\dotfill\pageref*{app:qual_meanflow_geneval}\par
\hspace*{1em}\hyperref[app:qual_scm_geneval]{\ref*{app:qual_scm_geneval}\quad sCM: GenEval}\nobreak\dotfill\pageref*{app:qual_scm_geneval}\par
\hspace*{1em}\hyperref[app:qual_meanflow_mixed]{\ref*{app:qual_meanflow_mixed}\quad MeanFlow: OCR + PickScore}\nobreak\dotfill\pageref*{app:qual_meanflow_mixed}\par
\hspace*{1em}\hyperref[app:qual_scm_mixed]{\ref*{app:qual_scm_mixed}\quad sCM: OCR + PickScore}\nobreak\dotfill\pageref*{app:qual_scm_mixed}\par
\endgroup

\section{Theoretical Details and Proofs}\label{app:contents_1}
\subsection{Proof of the long-range SDE obstruction}\label{app:finite_time}
\paragraph{Theorem~\ref{thm:long_range_sde} (Obstruction to additive-Gaussian SDE transitions).}
For nonlinear drift, the exact SDE transition over $t<r$ need not admit a representation
\begin{equation*}
    X_r=G_{t,r}(X_t;\psi,\mathrm d\psi,\ldots)+A_{t,r}\epsilon,
    \qquad\epsilon\sim\N(0,I),\quad\epsilon\perp X_t,
\end{equation*}
where $G_{t,r}(X_t;\psi,\mathrm d\psi,\ldots)$ is a deterministic function of the initial state $X_t$, the flow-map family $\psi$, and its derivatives, denoted by $\mathrm d\psi,\ldots$. The matrix $A_{t,r}$ is deterministic and independent of $X_t$. In particular, smooth additive-noise diffusions exist whose conditional transition laws are non-Gaussian, ruling out any such universal representation based on the deterministic flow map and its derivatives.

\begin{proof}
We exhibit a smooth nonlinear diffusion with a non-Gaussian conditional transition. Fix $X_0=x$ and consider $dX_s=\beta(X_s)\,ds+\sigma\,dW_s$, where $\sigma>0$ is constant and $\beta$ is smooth and bounded with bounded derivatives. The choice $\beta(x)=\sin x$ satisfies these assumptions. The generator is $L=\beta\partial_x+\frac{\sigma^2}{2}\partial_{xx}$. For the polynomial test functions used below, finite moments and the generator expansion give
\begin{equation}
 \E[f(X_h)\mid X_0=x]
 =f(x)+hLf(x)+\frac{h^2}{2}L^2f(x)+\frac{h^3}{6}L^3f(x)+o(h^3).
 \label{eq:generator_expansion}
\end{equation}
Write $M_k=\E[(X_h-x)^k\mid X_0=x]$ and evaluate $\beta,\beta',\beta''$ at $x$. Applying Equation~\eqref{eq:generator_expansion} to $(z-x)^k$ yields the terms needed for the third centered moment:
\begin{equation}
\begin{aligned}
 M_1&=\beta h+\left(\tfrac12\beta\beta'+\tfrac14\sigma^2\beta''\right)h^2+O(h^3),\\
 M_2&=\sigma^2h+(\beta^2+\sigma^2\beta')h^2+O(h^3),\\
 M_3&=3\sigma^2\beta h^2+
 \left(\beta^3+\tfrac92\sigma^2\beta\beta'+\tfrac74\sigma^4\beta''\right)h^3+o(h^3).
\end{aligned}
\end{equation}
Consequently, with $m_h=\E[X_h\mid X_0=x]$, the terms of order $h^2$, $\beta^3h^3$, and $\sigma^2\beta\beta'h^3$ cancel:
\begin{equation}
 \E[(X_h-m_h)^3\mid X_0=x]
 =M_3-3M_1M_2+2M_1^3
 =\sigma^4\beta''(x)h^3+o(h^3).
 \label{eq:kappa3_expansion}
\end{equation}
For $\beta(x)=\sin x$ and $x=\pi/2$, this quantity is nonzero for all sufficiently small $h>0$. In contrast, any representation $G_{0,h}(x;\psi,\mathrm d\psi,\ldots)+A_{0,h}\epsilon$ with independent Gaussian $\epsilon$ has a Gaussian conditional law and zero third centered moment. No such representation can therefore hold for general long-range SDE transitions. A scalar example suffices to disprove the general claim; it can also be embedded in a higher-dimensional diffusion using independent additional coordinates.

This example is compatible with the probability-flow setting. Start the diffusion from a smooth positive initial density and let $p_s$ be its marginal density. Its Fokker--Planck equation is $\partial_sp_s=-\partial_x(\beta p_s)+\frac{\sigma^2}{2}\partial_{xx}p_s$. Defining $u_s=\beta-\frac{\sigma^2}{2}\partial_x\log p_s$ gives $\partial_sp_s=-\partial_x(u_sp_s)$ and $\beta=u_s+\frac{\sigma^2}{2}\partial_x\log p_s$. Thus the same diffusion is a path-preserving stochasticization of its probability-flow ODE wherever the associated smooth flow is defined. The obstruction concerns a general additive-Gaussian representation; it does not assert that every nonlinear diffusion has non-Gaussian transitions at every time pair.

\end{proof}

\subsection{ASFMC path preservation and anchor constructions}\label{app:local_derivation}

\paragraph{Theorem~\ref{thm:asfmc_path_preservation} (Path preservation of ASFMC).}
Let $X_t\sim p_t$ and define $X_r=\psi_{t\to r}(X_t)$ and $X_\tau=\psi_{r\to\tau}(X_r)$ using accurate flow maps. Let $p_{r\mid\tau}$ be a reverse conditional distribution whose joint law has marginals $p_r$ and $p_\tau$. If $\widetilde X_r\sim p_{r\mid\tau}(\cdot\mid X_\tau)$, then $\widetilde X_r\sim p_r$. Thus, the stochastic flow map induced by ASFMC preserves the marginal distribution of $\psi_{t\to r}$.
\begin{proof}
Exact transport first gives $X_r\sim p_r$ and then $X_\tau\sim p_\tau$. Let $P_{r,\tau}$ be the joint law defining the selected conditional, so $P_{r,\tau}(dx_r,dx_\tau)=p_{r\mid\tau}(dx_r\mid x_\tau)p_\tau(dx_\tau)$ has first marginal $p_r$. For every measurable set $A$,
\begin{equation}
 \Pr(\widetilde X_r\in A)=\int p_{r\mid\tau}(A\mid x_\tau)\,p_\tau(dx_\tau)=p_r(A).
\end{equation}
This establishes preservation at one transition. Applying the same argument successively on a sampling grid preserves every grid-time marginal, provided each transport and conditional is exact. The joint law used for resampling may be the forward diffusion for a local anchor or the affine interpolant for an endpoint anchor; it need not be the deterministic flow's trajectory coupling.
\end{proof}

\paragraph{Affine-path velocity and score.}
Let $X_s=a_sX_1+b_sX_0$, with $X_0\sim\N(0,I)$ independent of $X_1$, as in Section~\ref{sec:background}. Work on an interior interval where $a_s,b_s>0$ and $D_s=\dot a_sb_s-a_s\dot b_s>0$. Assume that the path coefficients are smooth and the conditional moments below exist. Differentiating the Gaussian conditional density and averaging over $X_1$ gives
\begin{equation}
 q_s(x)=\nabla_x\log p_s(x)
 =-\frac{x-a_s\E[X_1\mid X_s=x]}{b_s^2}
 =-\frac{1}{b_s}\E[X_0\mid X_s=x].
\end{equation}
The marginal velocity of the affine interpolant is $u_s(x)=\E[\dot a_sX_1+\dot b_sX_0\mid X_s=x]$. Together with $x=a_s\E[X_1\mid X_s=x]+b_s\E[X_0\mid X_s=x]$, this implies~\citep{lipman2024flow}
\begin{equation}
 q_s(x)=\frac{a_su_s(x)-\dot a_sx}{b_sD_s}.
 \label{eq:score}
\end{equation}
For Flow Matching, $(a_s,b_s)=(s,1-s)$ gives $D_s=1$; for the DDIM path, $(a_s,b_s)=(\sqrt{1-(1-s)^2},1-s)$ gives $D_s=1/a_s$. Both are positive for $0<s<1$. Boundary transitions are handled by the endpoint maps rather than by substituting a zero denominator into Equation~\eqref{eq:score}.

\paragraph{Forward and reverse stochastic dynamics.}
For a spatially constant diffusion coefficient $\sigma_s$, the forward drift $u_s+\frac12\sigma_s^2q_s$ preserves the marginal path because
\begin{equation}
 -\nabla\!\cdot\!\left[\left(u_s+\tfrac12\sigma_s^2q_s\right)p_s\right]
 +\tfrac12\sigma_s^2\Delta p_s=-\nabla\!\cdot(u_sp_s).
\end{equation}
Reversing this diffusion yields the drift $B_s=u_s-\frac12\sigma_s^2q_s$ when $s$ runs backward. More explicitly, starting at $\widehat X_0=X_\tau$ and using elapsed reverse time $v\in[0,h]$, the conditional sampler is $d\widehat X_v=-B_{\tau-v}(\widehat X_v)\,dv+\sigma_{\tau-v}\,dW_v$. Choosing $\sigma_s^2=2\lambda^2b_sD_s/a_s$ yields
\begin{equation}
 B_s(x)=(1-\lambda^2)u_s(x)+\lambda^2\frac{\dot a_s}{a_s}x,
 \qquad \tfrac12\sigma_s^2q_s(x)=\lambda^2\left[u_s(x)-\frac{\dot a_s}{a_s}x\right].
 \label{eq:reverse_anchor_drift}
\end{equation}

\paragraph{Local-anchor expansion.}
Set $X_r=\psi_{t\to r}(X_t)$, $\tau=r+h<1$, and $X_\tau=\psi_{r\to\tau}(X_r)$ using exact transport. Under the assumptions below, the forward segment and the reverse stochastic step satisfy
\begin{equation}
\begin{aligned}
 X_\tau&=X_r+hu_r(X_r)+O_{L^2}(h^2),\\
 \widetilde X_r^\star&=X_\tau-hB_\tau(X_\tau)+\sigma_\tau\sqrt h\,\xi+O_{L^2}(h^{3/2}),
 \qquad\xi\sim\N(0,I).
\end{aligned}
\end{equation}
Since $B_\tau(X_\tau)=B_r(X_r)+O_{L^2}(h)$ and $\sigma_\tau=\sigma_r+O(h)$, substitution gives
\begin{equation}
\begin{aligned}
 \widetilde X_r^\star
 &=X_r+h[u_r(X_r)-B_r(X_r)]+\sigma_r\sqrt h\,\xi+O_{L^2}(h^{3/2})\\
 &=X_r+h\lambda^2\left[u_r(X_r)-\frac{\dot a_r}{a_r}X_r\right]
   +\sigma_r\sqrt h\,\xi+O_{L^2}(h^{3/2}).
\end{aligned}
\end{equation}
This recovers both the reverse-step expansion in Equation~\eqref{eq:local_anchor_reverse_expansion} and the local policy in Equation~\eqref{eq:local_anchor_update}. The correction is produced by composing the forward anchor transport with reverse resampling. The following proof bounds the remainder under a shared Brownian increment.

\paragraph{Theorem~\ref{thm:local_anchor_error} (Error of the local-anchor approximation).}
Let $X_r=\psi_{t\to r}(X_t)$ and $X_{r+h}=\psi_{r\to r+h}(X_r)$ use exact transport. Let $\widetilde X_r^\star$ be the exact reverse-SDE sample from $X_{r+h}$ to time $r$, and let $\widetilde X_r^{\mathrm{loc}}$ be Equation~\eqref{eq:local_anchor_update} with exact velocity. Couple both samples using the same Brownian increment, $\xi=W_h/\sqrt h$. Under the smoothness and moment assumptions in Appendix~\ref{app:local_derivation},
\begin{equation*}
    \|\widetilde X_r^\star-\widetilde X_r^{\mathrm{loc}}\|_{L^2}=O(h^{3/2})
    \qquad\text{as }h\to0.
\end{equation*}
Write $\tau=r+h$ and $B_s(x)=u_s(x)-\tfrac12\sigma_s^2q_s(x)$ for the reverse-time drift. On an interval containing $[r,r+h]$, assume that $a_s$ is bounded away from zero; $u_s$ and $B_s$ are globally Lipschitz in space, continuously differentiable in time, and have linear-growth bounds, including for their time derivatives; and $\sigma_s$ is continuously differentiable with bounded derivative. Assume $\E\|X_r\|^2<\infty$. These conditions make the following bounds uniform for sufficiently small $h$ and hold for smooth coefficients with the corresponding bounded-derivative and moment bounds.

\begin{proof}
Parameterize reverse time by $v\in[0,h]$. The exact resampling process satisfies $\widehat X_0=X_\tau$ and $d\widehat X_v=-B_{\tau-v}(\widehat X_v)\,dv+\sigma_{\tau-v}\,dW_v$, where $W$ is independent of $X_r$. Thus $\widetilde X_r^\star=\widehat X_h$. Use the same increment $\xi=W_h/\sqrt h$ in the local approximation. Exact deterministic transport gives
\begin{equation}
    X_\tau=X_r+hu_r(X_r)+R_{\mathrm{ode}},
    \qquad\|R_{\mathrm{ode}}\|_{L^2}=O(h^2).
\end{equation}
The reverse-SDE integral can be decomposed as
\begin{equation}
\begin{aligned}
    \widehat X_h&=X_\tau-hB_r(X_r)+\sigma_rW_h+R_b+R_\sigma,\\
    R_b&=-\int_0^h\bigl[B_{\tau-v}(\widehat X_v)-B_r(X_r)\bigr]dv,\\
    R_\sigma&=\int_0^h(\sigma_{\tau-v}-\sigma_r)\,dW_v.
\end{aligned}
\end{equation}
The moment and Lipschitz assumptions imply $\|\widehat X_v-X_\tau\|_{L^2}\le C\sqrt v$ and $\|X_\tau-X_r\|_{L^2}\le Ch$. Consequently, Minkowski's inequality gives
\begin{equation}
    \|R_b\|_{L^2}\le C\int_0^h(h+\sqrt v)\,dv=O(h^{3/2}).
\end{equation}
By It\^o isometry and the time regularity of $\sigma$, $\|R_\sigma\|_{L^2}^2=d\int_0^h|\sigma_{\tau-v}-\sigma_r|^2dv=O(h^3)$. Substituting the transport expansion and the score identity gives
\begin{equation}
\begin{aligned}
    \widehat X_h
    &=X_r+h\bigl[u_r(X_r)-B_r(X_r)\bigr]+\sigma_r\sqrt h\,\xi
      +R_{\mathrm{ode}}+R_b+R_\sigma\\
    &=X_r+h\lambda^2\left[u_r(X_r)-\frac{\dot a_r}{a_r}X_r\right]
      +\sigma_r\sqrt h\,\xi+R_{\mathrm{ode}}+R_b+R_\sigma.
\end{aligned}
\end{equation}
The leading terms are precisely $\widetilde X_r^{\mathrm{loc}}$. The triangle inequality therefore yields the claimed $O(h^{3/2})$ strong error. This is a local discretization bound with exact transport and velocity; learned-model approximation errors are separate.
\end{proof}

\paragraph{Endpoint-anchor conditional and one-time implementation.}
For $\tau=1$, the affine-path boundary conditions give $X_1\sim p_1$ and $X_r\mid X_1=x_1\sim\N(a_rx_1,b_r^2I)$. Exact transport produces $X_1=\psi_{r\to1}(\psi_{t\to r}(X_t))\sim p_1$, so
\begin{equation}
 \widetilde X_r=a_rX_1+b_r\xi,\qquad\xi\sim\N(0,I),\quad\xi\perp X_1
 \quad\Longrightarrow\quad\widetilde X_r\sim p_r.
\end{equation}
This is the endpoint instance of Theorem~\ref{thm:asfmc_path_preservation} and yields Equation~\eqref{eq:one_time_anchor}. It is exact for the affine conditional; it does not require a short-interval approximation. For accurate flow maps, the composition law gives $\psi_{r\to1}\circ\psi_{t\to r}=\psi_{t\to1}$. A one-time model therefore supplies the endpoint directly through $f_t^\theta(x_t)$, giving mean $a_rf_t^\theta(x_t)$ and variance $b_r^2I$ as used in Section~\ref{sec:flowmap_grpo}. This implementation choice does not restrict the general endpoint-anchor construction to one-time models. Appendix~\ref{sec:appendix_one_time_two_time} gives both sampling interfaces. At $r=1$, $b_1=0$ and the final prediction is deterministic.

The path-preservation theorem uses exact transport and the exact conditional kernel. The local Gaussian rule approximates its reverse conditional with the error quantified above, whereas the endpoint Gaussian conditional is analytic and its model implementation depends on endpoint-prediction accuracy. Both provide explicit policy densities at stochastic steps with positive variance.

\subsection{Gaussian likelihoods and optimization}\label{app:contents_4}
At a stochastic transition, let $\pi_\theta(x'\mid s)=\N(\mu_\theta(s),\nu^2I)$ with $\nu>0$ fixed by the sampling schedule and $x'\in\mathbb R^d$. The joint log density and its mean gradient are
\begin{equation}
 \log\pi_\theta(x'\mid s)
 =-\frac{\|x'-\mu_\theta(s)\|^2}{2\nu^2}-d\log\nu-\frac d2\log(2\pi),
 \qquad
 \nabla_\theta\log\pi_\theta
 =\frac{(D_\theta\mu_\theta)^\top(x'-\mu_\theta)}{\nu^2}.
\end{equation}
The local mean is the corrected transport in Equation~\eqref{eq:local_anchor_update}, with variance $\sigma_r^2h$. The endpoint mean is the transported endpoint scaled by $a_r$ in Equation~\eqref{eq:one_time_anchor}, with variance $b_r^2$. Current and rollout policies use the same variance and are evaluated at the same stored action and state. Hence their density ratio simplifies to
\begin{equation}
 \rho_{ij}
 =\exp\!\left(
 \frac{\|x_{r_j}^{(i)}-\mu_j^{\mathrm{old}}(s_j^{(i)})\|^2
       -\|x_{r_j}^{(i)}-\mu_j^\theta(s_j^{(i)})\|^2}{2\nu_j^2}
 \right).
 \label{eq:gaussian_policX_ratio}
\end{equation}
This is the ratio defined in Section~\ref{sec:flowmap_grpo}. If a KL term is used, the Gaussian KL with the same covariance is $D_{\mathrm{KL}}(\pi_\theta\|\pi_{\mathrm{ref}})=\|\mu_\theta-\mu_{\mathrm{ref}}\|^2/(2\nu^2)$.

The implementations average log densities over latent coordinates. Writing $\ell_\theta=d^{-1}\log\pi_\theta$, their normalized ratio is $\rho_{ij}^{\mathrm{norm}}=\exp(\ell_\theta-\ell_{\mathrm{old}})=\rho_{ij}^{1/d}$. The numerical clipping threshold is applied to this normalized ratio; it is distinct from the joint-density ratio in Equation~\eqref{eq:gaussian_policX_ratio}. This convention rescales log-ratio changes for high-dimensional image latents. The reported experiments use zero KL coefficient.

Advantages use the condition-specific reward mean and rollout-batch standard deviation, followed by clipping, as defined in Section~\ref{sec:flowmap_grpo}. ImageNet uses the current class-group mean, while text-to-image training uses running per-prompt statistics (Appendix~\ref{appendix:t2i_details}). Rewards and sampled actions are held fixed during each update. Gradients are stopped through the auxiliary equal-time velocity in the local correction; the transport prediction supplies the trainable mean. Only LoRA parameters are updated. The deterministic terminal step produces the reward image and contributes no Gaussian likelihood term because its variance is zero.

\subsection{Midpoint anchors and the posterior constraint}\label{app:midpoint}
\paragraph{Proposition~\ref{prop:midpoint_anchor} (Posterior constraint for midpoint anchors).}
Let $X_s=a_sX_1+b_sX_0$, where $X_0\sim\N(0,I)$ is independent of $X_1$. For $r<\tau<1$ with $a_\tau\ne0$, define $c_{r,\tau}=b_r-a_rb_\tau/a_\tau$. Then
\begin{equation*}
 \frac{a_r}{a_\tau}x+c_{r,\tau}X_0^{\mathrm{post}}
 \sim p_{r\mid\tau}(\cdot\mid x),
 \qquad X_0^{\mathrm{post}}\sim p_{0\mid\tau}(\cdot\mid x).
\end{equation*}
If $c_{r,\tau}\ne0$, replacing $X_0^{\mathrm{post}}$ with independent standard Gaussian noise is exact if and only if $p_{0\mid\tau}(\cdot\mid x)=p_0$. For the linear path and data with finite second moment, this substitution fails to recover the conditional kernel for every $0\le r<\tau<1$.  (see Appendix~\ref{app:midpoint} for proof.)

\begin{proof}
Consider a midpoint anchor $r<\tau<1$ under the affine interpolant $X_s=a_sX_1+b_sX_0$, with independent $X_0\sim\N(0,I)$ and $X_1\sim p_1$. The arithmetic midpoint $\tau=(r+1)/2$ is one such choice. For $a_\tau\ne0$, eliminating $X_1$ gives
\begin{equation}
 X_r=\frac{a_r}{a_\tau}X_\tau+
 \left(b_r-\frac{a_rb_\tau}{a_\tau}\right)X_0.
\end{equation}
Although this identity is affine in $X_0$, conditioning on $X_\tau=x_\tau$ changes the distribution of that noise. When the data density exists, Bayes' rule gives
\begin{equation}
 p_{0\mid\tau}(x_0\mid x_\tau)\propto
 p_0(x_0)\,p_1\!\left(\frac{x_\tau-b_\tau x_0}{a_\tau}\right),
\end{equation}
where the Jacobian is constant in $x_0$. The interpolant conditional is therefore sampled by
\begin{equation}
 \frac{a_r}{a_\tau}x_\tau+
 \left(b_r-\frac{a_rb_\tau}{a_\tau}\right)X_0^{\mathrm{post}}
 \sim p_{r\mid\tau}(\cdot\mid x_\tau),
 \qquad X_0^{\mathrm{post}}\sim p_{0\mid\tau}(\cdot\mid x_\tau).
\end{equation}
When $c_{r,\tau}\ne0$, the affine map from $X_0^{\mathrm{post}}$ to $X_r$ is invertible. Consequently, replacing the posterior law of $X_0^{\mathrm{post}}$ with $p_0$ yields the same conditional distribution exactly when the two noise laws agree. For the linear path, $c_{r,\tau}=(\tau-r)/\tau>0$. Moreover, independence of $X_0$ and $X_1$ and the finite-second-moment assumption imply
\begin{equation}
 \operatorname{Cov}(X_0,X_\tau)=(1-\tau)I\ne0.
\end{equation}
Thus $X_0$ and $X_\tau$ are dependent, so $p_{0\mid\tau}(\cdot\mid x)$ cannot equal $p_0$ for almost every anchor state. The unconditional Gaussian substitution therefore differs from the true conditional on a set of anchor states with positive probability. This proves the claim for every $0\le r<\tau<1$.

Replacing $X_0^{\mathrm{post}}$ with independent standard Gaussian noise generally does not implement this conditional, since a midpoint anchor contains information about both $X_1$ and $X_0$. At the endpoint, $a_1=1$ and $b_1=0$, so independence recovers the Gaussian kernel $a_rx_1+b_r\xi$. Thus, the obstacle for midpoint anchors is the unavailable posterior, not the choice of a midpoint anchor itself. This affine-interpolant conditional and the reverse-diffusion construction used for the local policy are different couplings; the identity above does not equate their transition laws.

\end{proof}

\paragraph{Midpoint-anchor experiment.}
We evaluate an OCR-trained MeanFlow adapter using a midpoint-anchor kernel with anchor fraction 0.5 and scheduled Gaussian re-noising. We follow the MeanFlow OCR training settings and use the same task test sets, resolution, and guidance scale as in Appendix~\ref{appendix:t2i_details}, replacing the local-anchor kernel with the midpoint-anchor kernel. We evaluate the EMA adapter after 1,000 optimizer steps using deterministic MeanFlow sampling across inference step counts. The adapter is evaluated with batch size two, as in its training-time evaluation; the reused Base results use batch size four.

At $K=4$, midpoint-anchor post-training raises OCR only from 0.3491 to 0.3718, compared with 0.8504 for the local-anchor adapter evaluated at training step 800 (Table~\ref{tab:ocr_meanflow_uniform}). At $K=3$, OCR falls below Base. All five DrawBench metrics decrease at every reported number of sampling steps (Table~\ref{tab:midpoint_sweep}). These results complement the controlled transition ablations in Section~\ref{sec:ablation}: tractable Gaussian noise alone does not ensure useful post-training. The weak task adaptation and broad quality degradation are consistent with the posterior mismatch analyzed above; the sweep measures its downstream consequences under deterministic evaluation, rather than directly measuring stochastic rollout marginals.

\begin{table}[H]
\centering\small
\caption{MeanFlow OCR post-training with the midpoint-anchor kernel, evaluated using deterministic sampling and the EMA adapter at training step 1,000. Task Metrics measure RL improvements on the target tasks; DrawBench Metrics are auxiliary general-quality metrics used to monitor degradation outside the optimized reward. Bold indicates the better score within each Base/adapter pair. DB Sum is the unnormalized sum of the five DrawBench metrics. Unmeasured tasks are denoted by --.}
\label{tab:midpoint_sweep}
\setlength{\tabcolsep}{1.4pt}
\renewcommand{\arraystretch}{1.08}
\resizebox{\linewidth}{!}{
\begin{tabular}{rlccc@{\hspace{0.5em}}|@{\hspace{0.5em}}cccccc}
\toprule
\multirow{2}{*}{$K$} & \multirow{2}{*}{\textbf{Model}} & \multicolumn{3}{c}{\textbf{Task Metrics}} & \multicolumn{6}{c}{\textbf{DrawBench Metrics}} \\
\cmidrule(lr){3-5}\cmidrule(lr){6-11}
 & & \textbf{GenEval}$\uparrow$ & \shortstack{\textbf{OCR}\\\textbf{Acc.}$\uparrow$} & \shortstack{\textbf{PickScore}\\\textbf{Task}$\uparrow$} & \shortstack{\textbf{DB-}\\\textbf{PickScore}$\uparrow$} & \textbf{Aesthetic}$\uparrow$ & \textbf{DeQA}$\uparrow$ & \textbf{ImgRwd}$\uparrow$ & \textbf{UniRwd}$\uparrow$ & \textbf{DB Sum}$\uparrow$ \\
\midrule
3 & MeanFlow Base & -- & \textbf{0.2876} & -- & \textbf{21.3577} & \textbf{5.4391} & \textbf{4.0001} & \textbf{0.2157} & \textbf{2.5495} & \textbf{33.5621} \\
3 & MeanFlow + Midpoint GRPO LoRA & -- & 0.2767 & -- & 21.0636 & 5.2149 & 3.3155 & -0.0996 & 2.2472 & 31.7417 \\
\midrule
4 & MeanFlow Base & -- & 0.3491 & -- & \textbf{21.6298} & \textbf{5.5368} & \textbf{4.1712} & \textbf{0.3918} & \textbf{2.7187} & \textbf{34.4483} \\
4 & MeanFlow + Midpoint GRPO LoRA & -- & \textbf{0.3718} & -- & 21.4442 & 5.3572 & 3.7296 & 0.1650 & 2.4845 & 33.1805 \\
\midrule
7 & MeanFlow Base & -- & 0.4517 & -- & \textbf{22.0761} & \textbf{5.6090} & \textbf{4.2500} & \textbf{0.5744} & \textbf{2.9216} & \textbf{35.4311} \\
7 & MeanFlow + Midpoint GRPO LoRA & -- & \textbf{0.5020} & -- & 21.9719 & 5.4718 & 4.0544 & 0.4473 & 2.8023 & 34.7476 \\
\midrule
15 & MeanFlow Base & -- & 0.5507 & -- & \textbf{22.2302} & \textbf{5.6489} & \textbf{4.3217} & \textbf{0.6870} & \textbf{3.0208} & \textbf{35.9086} \\
15 & MeanFlow + Midpoint GRPO LoRA & -- & \textbf{0.5865} & -- & 22.1645 & 5.5345 & 4.2350 & 0.5891 & 2.9602 & 35.4833 \\
\bottomrule
\end{tabular}
}
\end{table}

\paragraph{Stochastic rollout quality.}
Figure~\ref{fig:midpoint_rollouts} shows five-step samples from the midpoint-anchor policy during late-stage training. The visible grain, distorted textures, and malformed lettering show that poor sample quality is already present in the trajectories used to collect rewards. Together with the posterior constraint above and the weak downstream results in Table~\ref{tab:midpoint_sweep}, these observations support the importance of preserving the intended marginals when constructing the exploration policy.

\begin{figure}[H]
\centering
\includegraphics[width=0.8\linewidth]{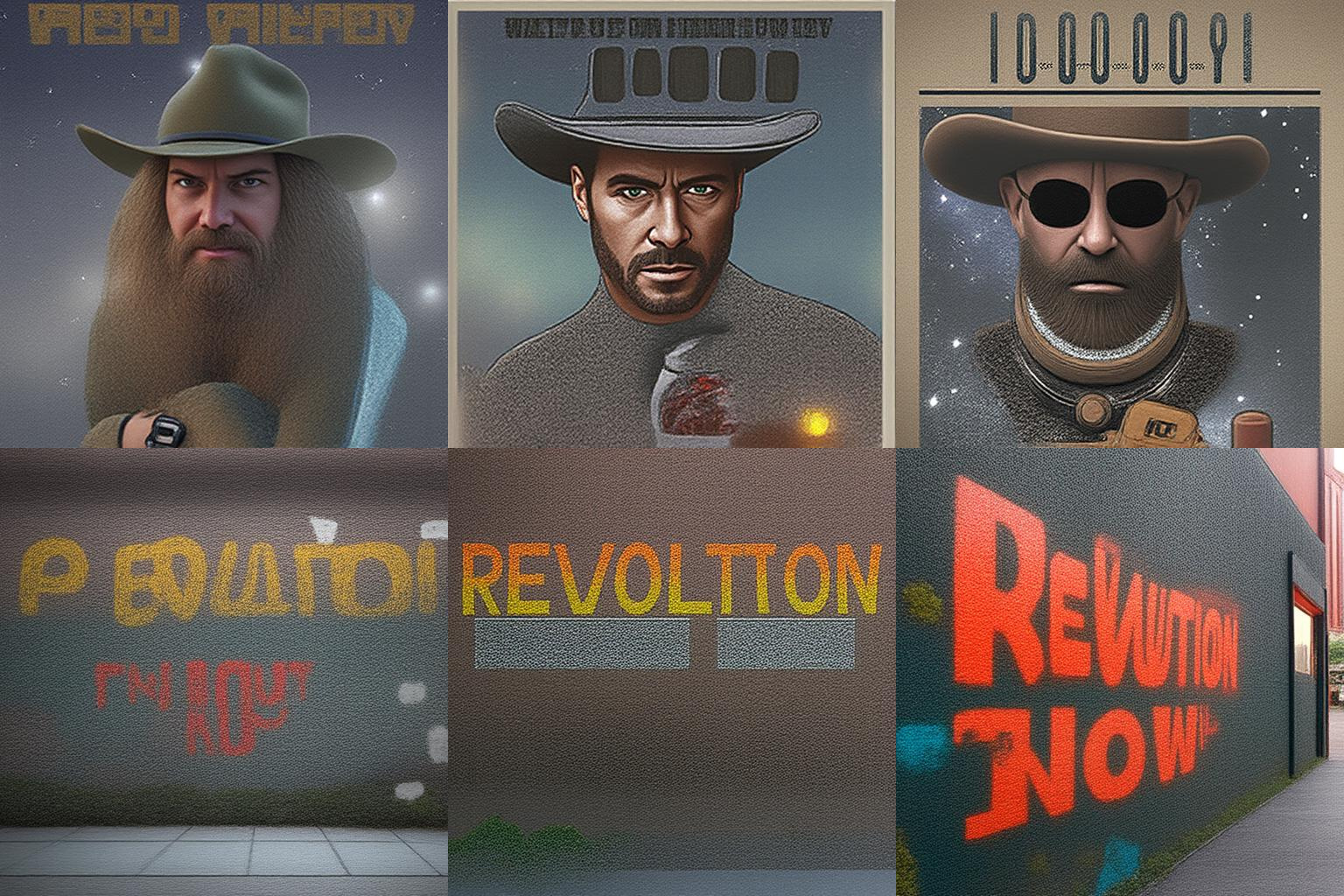}
\caption{Five-step stochastic training rollouts from the MeanFlow midpoint-anchor policy. These late-training samples exhibit grain, texture artifacts, and malformed text under the policy used for reward collection. The deterministic evaluation sweep is reported separately in Table~\ref{tab:midpoint_sweep}.}
\label{fig:midpoint_rollouts}
\end{figure}

\subsection{Unified Interpretation of Anchor Choices}\label{app:anchor_interpretation}
The local-anchor policy can also be viewed as a Langevin-style corrector induced by stochastic flow-map composition: nearby-anchor transport followed by a short stochastic reverse-resampling step, including its drift component, yields a correction with matched drift and noise around the deterministic flow-map output (Equation~\eqref{eq:local_anchor_update}). The endpoint-anchor policy, in contrast, recovers the familiar clean-prediction-and-re-noising structure of consistency samplers (Equation~\eqref{eq:one_time_anchor}). These two cases correspond to anchor choices for which the required reverse conditional is either available in closed form or admits a controlled local approximation (Theorem~\ref{thm:local_anchor_error}). A generic midpoint anchor does not enjoy the same tractability: its conditional depends on a nontrivial posterior over the original noise variable, so the independent Gaussian re-noising rule considered here is generally mismatched (Proposition~\ref{prop:midpoint_anchor}). This does not preclude midpoint anchors in principle; they could be realized with an appropriate posterior or conditional sampler, or with an approximation sufficiently accurate for the desired transition.

More broadly, these cases illustrate our view of stochasticizing deterministic flow maps. Rather than perturbing the output of a long-range map with an arbitrary noise kernel, we construct stochastic transitions by composing deterministic transport with conditional resampling through an anchor. The local Langevin-style corrector and the consistency-model re-noising sampler then arise as two tractable instances of this same principle. From this perspective, the anchor determines both where stochasticity is introduced and the conditional structure required to preserve the underlying probability path. For the local anchor, this structure includes the pairing of drift and noise. This view explains both the failure of naive midpoint re-noising and the strong degradation observed when the local drift correction is removed (Table~\ref{tab:midpoint_sweep} and Section~\ref{sec:ablation}).

\subsection{Evaluation Sampling for One-Time and Two-Time Models}\label{sec:appendix_one_time_two_time}
A two-time model directly predicts $\psi^\theta_{t\to r}(x_t)$; MeanFlow implements it through an average velocity. A one-time model directly predicts only $f_t^\theta(x_t)=\psi^\theta_{t\to1}(x_t)$; an intermediate sample is constructed by the chosen sampler from that endpoint estimate. For the affine stochastic sampler, this is $x_r=a_rf_t^\theta(x_t)+b_r\xi$. Algorithm~\ref{alg:flowmap_sampling} describes sampling during evaluation; stochastic rollout collection for RL post-training follows Algorithm~\ref{alg:flowmap_grpo}.

\begin{algorithm}[H]
\caption{Evaluation sampling with two-time maps or one-time endpoint predictors}
\label{alg:flowmap_sampling}
\begin{algorithmic}[1]
\Require Grid $0=r_0<\cdots<r_{K+1}=1$, condition $c$, generator, path $(a_r,b_r)$
\State Sample $x_{r_0}\sim p_0$
\For{$j=1,\ldots,K+1$}
\If{the generator directly predicts two-time maps}
\State $x_{r_j}\gets\psi^\theta_{r_{j-1}\to r_j}(x_{r_{j-1}}\mid c)$
\Else
\State $\widehat x_1\gets f^\theta_{r_{j-1}}(x_{r_{j-1}}\mid c)$
\If{$j\le K$}
\State Sample $\xi_j\sim\N(0,I)$ independently
\State $x_{r_j}\gets a_{r_j}\widehat x_1+b_{r_j}\xi_j$
\Else
\State $x_1\gets\widehat x_1$
\EndIf
\EndIf
\EndFor
\State \Return $x_1$
\end{algorithmic}
\end{algorithm}

\section{Extended Related Work}\label{app:related_work}

\paragraph{Continuous-time generative models.}
Diffusion models and continuous-time flow-based models provide a foundation for high-quality generative modeling. Diffusion models define a stochastic noising process and learn to reverse it through denoising or score estimation~\citep{ho2020denoising,song2020denoising,song2020score}. Flow matching and rectified flow learn time-dependent velocity fields that transport a simple prior distribution to the data distribution through an ODE~\citep{lipman2022flow,liu2023flow}. Stochastic interpolants connect deterministic flows and stochastic diffusions through probability paths~\citep{albergo2025stochastic}. These formulations support principled generative sampling, but their standard numerical solvers typically require many sequential evaluations. Our setting concerns generators that have already replaced this integration with learned long-range maps.

\paragraph{Few-step flow-map generation.}
Few-step methods reduce sampling cost by learning long-range transport operators. Consistency models learn mappings from noisy states to clean data and support one-step or few-step generation through self-consistency constraints~\citep{song2023consistency,song2023improved,geng2024consistency}. Flow Map Matching provides a framework for learning two-time flow maps of probability-flow dynamics and connects these maps with consistency models and progressive distillation~\citep{boffi2025flow}. Shortcut models and MeanFlow parameterize long-range transitions or average velocities, enabling direct jumps across time intervals~\citep{frans2025shortcut,geng2025mean,geng2025improved}. The learned map $\psi_{t\rightarrow r}$ is deterministic, even when a sampling procedure composes it with re-noising. A deterministic map alone does not define a nondegenerate stochastic policy or the transition-density ratios needed by GRPO. Flow-Map GRPO supplies this policy interface while preserving the pretrained prediction parameterization.

\paragraph{Reinforcement learning for generative models.}
Reinforcement learning aligns generative models with task-level rewards, including human preference, aesthetic quality, and downstream evaluation metrics. DDPO formulates diffusion sampling as a multi-step Markov decision process and applies policy-gradient optimization to stochastic denoising trajectories~\citep{black2024training}. DanceGRPO, Flow-GRPO, and MixGRPO develop related approaches for visual generation and flow-based models~\citep{xue2025dancegrpo,liu2026flow,li2025mixgrpo}. In particular, Flow-GRPO introduces stochasticity into velocity-based flow models through an SDE formulation, obtaining stochastic transitions and likelihood ratios for GRPO. GLASS Flows studies stochastic transition sampling for inference-time alignment of continuous-time flow and diffusion models~\citep{holderrieth2026glass}.

The distinction is the generator's transition interface. Velocity-based samplers admit local stochastic updates obtained by discretizing their continuous-time dynamics. A few-step flow-map generator instead predicts a long-range transport directly. Over such an interval, an SDE transition integrates the effects of the intervening stochastic path and is generally non-Gaussian; its conditional law is not determined by an independent Gaussian perturbation of the deterministic endpoint. Thus, applying a local velocity-based stochasticization rule to a long-range map requires an additional policy construction. Our anchored composition separates deterministic transport from stochastic resampling and provides explicit transition densities for likelihood-ratio optimization.

\paragraph{Reward-based post-training of consistency models.}
RLCM~\citep{oertell2024rlcm} formulates the multistep consistency sampler as a Gaussian policy and applies clipped policy-gradient updates to terminal image rewards. Its clean prediction followed by Gaussian re-noising is closely related to our endpoint-anchor case. ROCM~\citep{shekhar2025rocm} instead optimizes consistency models using direct reward gradients with distributional regularization. These works establish that consistency-specific reward optimization predates the concurrent methods below. Our contribution extends beyond this endpoint setting: ASFMC supplies a common construction for endpoint predictors and two-time transport maps, including a local-anchor policy for MeanFlow and an analysis of the posterior constraint on intermediate anchors.

\paragraph{Stochastic and reward-aware flow maps.}
Several methods learn stochastic variants of flow maps for posterior sampling and reward alignment. Meta Flow Maps extend consistency models and flow maps to one-step posterior sampling, producing different samples from $p(x_1\mid x_t)$ for a fixed intermediate state~\citep{potaptchik2026meta}. Diamond Maps develops stochastic flow-map models designed for efficient reward alignment~\citep{holderrieth2026diamond}. Strong Stochastic Flow Maps learns strong solution maps of additive-noise SDEs~\citep{mccallum2026strong}, while stochastic few-step models study efficient sampling from SDE-defined conditional distributions~\citep{passaro2026stochastic}. Flow Map Reward Guidance uses flow maps for training-free reward guidance from an optimal-control perspective~\citep{huang2026guide}.

These works address several distinct settings: learning native stochastic maps, approximating posterior or SDE solution operators, and guiding generation at inference time. We address RL post-training of an existing deterministic flow-map generator. \textbf{ASFMC preserves the deterministic flow-map parameterization.} It composes the pretrained transport with anchor-based resampling to obtain a stochastic policy, and reward optimization updates adapters through GRPO. The local and endpoint constructions support two-time and one-time predictors within the same framework. For the local policy, the continuous-time derivation motivates the coupled drift and noise terms; we do not assume that a finite-strength discretization exactly reproduces the underlying SDE transition.

\paragraph{Concurrent work.}
Concurrent to our work, MeanFlowNFT~\citep{huang2026meanflownft} brings DiffusionNFT's forward-process RL objective to average-velocity generators. Its central issue is the mismatch between the instantaneous velocity optimized by DiffusionNFT and the average velocity used by MeanFlow. It uses the MeanFlow identity to construct an induced instantaneous-velocity predictor for training, while retaining average-velocity sampling. RAVEN~\citep{lu2026raven} introduces CM-GRPO for video post-training, defining GRPO on the Gaussian transition induced by a consistency sampler's clean-endpoint prediction. We introduce a unified framework for likelihood-ratio RL with flow-map generators. Anchored stochastic composition resolves the policy-construction problem through local and endpoint resampling, covering both MeanFlow-style two-time maps and consistency-style one-time predictors without changing their prediction interfaces.

\paragraph{Positioning.}
Our contribution is a unified stochastic policy construction for pretrained few-step flow-map generators. ASFMC provides tractable exploration and transition densities while retaining the pretrained model's architecture and parameterization. This enables reward-based adaptation of MeanFlow and consistency models through a common GRPO objective. It also separates the design of the stochastic training policy from the generator's inference sampler, allowing the adapted model to retain its one-step or few-step generation interface.

\section{Experiment Details}\label{app:contents_7}
\FloatBarrier
\subsection{MeanFlow and CM Ablation Details}\label{appendix:ablation_details}
\paragraph{Base model and reward.}
We use a pretrained class-conditional MeanFlow SiT-B/4 checkpoint with distilled classifier-free guidance at $256\times256$ resolution. The generator operates in a $4\times32\times32$ VAE latent space. The base weights are frozen, and LoRA adapters with rank 64 and scaling factor 128 are applied to the attention and feed-forward projections. For each class, the reward compares a group of 24 generated images with 50 real reference images using the negative unbiased MMD-squared estimate of \cite{miao2024diverse}. We use 2,048-dimensional InceptionV3 features~\citep{szegedy2016rethinking} and the polynomial kernel $k(z,z')=(z^\top z'/2048+1)^3$. The reward for image $x_m$ is its leave-one-out contribution,
\begin{equation}
    \hat R(x_m)=R_{\mathrm{MMD}}(\mathbb{X}_g;\mathbb{X}_r)-R_{\mathrm{MMD}}(\mathbb{X}_g\setminus\{x_m\};\mathbb{X}_r),
\end{equation}
where $\mathbb{X}_g$ and $\mathbb{X}_r$ denote the generated and reference sets. Thus, the GRPO group also serves as the image set used to compute the reward.

\paragraph{Local-anchor ablation study: implementation.}
For clarity, we describe the implementation using reverse time $s=1-t$, where $s=1$ denotes noise and $s=0$ denotes data, opposite to the time convention in the main text. We construct the five-step rollout grid by taking five uniformly spaced nodes from $1$ to $0.05$ and appending the data endpoint $0$, giving $(1,0.7625,0.525,0.2875,0.05,0)$. For a stochastic step $s_i\to s_{i+1}$, let $h_i=0.01(s_i-s_{i+1})$ and denote the deterministic transition by $\psi^\theta_{s_i\to s_{i+1}}(x_{s_i})$. The full transition specializes Equation~\eqref{eq:local_anchor_update} to the linear path, with this local anchor interval:
\begin{equation}
\begin{aligned}
    d_i &= -\frac{h_i\lambda^2}{1-s_{i+1}}
    \left(\psi^\theta_{s_i\to s_{i+1}}(x_{s_i})+(1-s_{i+1})\operatorname{sg}[u_\theta(\psi^\theta_{s_i\to s_{i+1}}(x_{s_i}),s_{i+1})]\right),\\
    x_{s_{i+1}} &= \underbrace{\psi^\theta_{s_i\to s_{i+1}}(x_{s_i})}_{\text{transport}}+\underbrace{d_i}_{\text{drift correction}}+\underbrace{\sigma_i\xi_i}_{\text{resampling noise}},
    \\
    \sigma_i &=\lambda\sqrt{\frac{2h_i s_{i+1}}{1-s_{i+1}}},
    \qquad \xi_i\sim\mathcal{N}(0,I),
\end{aligned}
\label{eq:local_anchor_ablation}
\end{equation}
where $u_\theta$ is the instantaneous velocity, evaluated by setting the MeanFlow time interval to zero, and $\operatorname{sg}$ stops gradients through this velocity prediction. We set $\lambda=5$ for ImageNet. The four stochastic steps have standard deviations approximately $(0.617,0.362,0.219,0.079)$. The \emph{No drift} ablation sets $d_i=0$ and keeps these standard deviations. Both constant-noise variants set $\sigma_i=\bar\sigma$, where $\bar\sigma^2=\frac{1}{4}\sum_{i=0}^{3}\sigma_i^2$, giving $\bar\sigma\approx0.376$. The four table rows therefore use $(d_i,\sigma_i)$ (Ours), $(d_i,\bar\sigma)$ (Drift + const.\ noise), $(0,\sigma_i)$ (No drift), and $(0,\bar\sigma)$ (No drift + const.\ noise). Each step uses independent Gaussian noise. The final transition from $0.05$ to $0$ is deterministic and excluded from the policy objective in all variants.

The no-drift variants therefore implement the most direct stochasticization of MeanFlow: first compute its deterministic long-range output, then add Gaussian noise. Explicitly,
\begin{equation}
\begin{aligned}
    \psi^\theta_{s_i\to s_{i+1}}(x_{s_i})&=x_{s_i}-(s_i-s_{i+1})\bar u^\theta_{s_i\to s_{i+1}}(x_{s_i}),\\
    x_{s_{i+1}}&=\psi^\theta_{s_i\to s_{i+1}}(x_{s_i})+\eta_i\xi_i,
    \qquad \eta_i\in\{\sigma_i,\bar\sigma\}.
\end{aligned}
\end{equation}
Here $\bar u^\theta$ is the predicted average velocity, and $\eta_i=\sigma_i$ and $\eta_i=\bar\sigma$ correspond to \emph{No drift} and \emph{No drift + const.\ noise}, respectively. Noise is added to the transported state $\psi^\theta_{s_i\to s_{i+1}}(x_{s_i})$, leaving the deterministic MeanFlow update unchanged. These baselines isolate whether output noise alone suffices for RL post-training; our local-anchor policy additionally supplies the drift correction $d_i$ paired with its noise schedule.

\paragraph{Endpoint-anchor ablation study: implementation.}
For the ImageNet CM implementation, we reuse the pretrained MeanFlow SiT-B/4 weights with the endpoint predictor $f_\theta(x_s,s)=x_s-s\bar u^\theta_{s\to0}(x_s)$, where $\bar u^\theta$ denotes the MeanFlow average velocity in the reverse-time convention. The target time is fixed to zero for every network evaluation. On the same rollout grid, our endpoint-anchor policy (\emph{Ours}) in Equation~\eqref{eq:one_time_anchor} becomes
\begin{equation}
    x_{s_{i+1}}=\underbrace{(1-s_{i+1})f_\theta(x_{s_i},s_i)}_{\text{scaled endpoint prediction}}+\underbrace{s_{i+1}\xi_i}_{\text{resampling noise}},
    \qquad \xi_i\sim\mathcal{N}(0,I).
\label{eq:endpoint_anchor_ablation}
\end{equation}
The noise standard deviations of \emph{Ours} are $(0.7625,0.525,0.2875,0.05)$. The constant-noise variant keeps the mean unchanged and replaces these scales with $\bar\sigma_{\mathrm{CM}}=(\frac14\sum_{i=0}^3 s_{i+1}^2)^{1/2}\approx0.485$. Only the noise coefficient changes: the endpoint predictor and its multiplier $1-s_{i+1}$ remain fixed, and the summed variance over the four stochastic steps is preserved. The last endpoint prediction is deterministic.

\paragraph{Training setting.}
Each epoch samples 48 class groups with 24 images per group and performs two optimizer updates. We use AdamW with learning rate $3\times10^{-4}$, weight decay $10^{-4}$, and gradient clipping at 1.0. Advantages are centered within each group, divided by the global reward standard deviation, and clipped to $[-5,5]$. The GRPO clip range is $10^{-4}$ and the KL coefficient is zero. Transition log probabilities are averaged over latent dimensions. Training uses fp32 on a single H100 GPU, without a time shift. All variants are trained for 400 epochs under the same configuration, and we evaluate the final checkpoints. All reported runs use training seed 42.

\paragraph{Evaluation setting.}
Following the evaluation setting of \citet{miao2024diverse}, we compare generated and held-out real-image sets using FID, Precision, and Recall. We use ImageNet classes 0--99. For each class, the first 50 training images form the reward reference set, the next 20 form the held-out test set, and the next 30 form a separate validation set. These sets do not overlap. Evaluation compares 50 generated images per class (5,000 in total) with 20 test images per class (2,000 in total). FID, Precision, and Recall are computed from pooled InceptionV3 features~\citep{szegedy2016rethinking} across all 100 classes using the same PyTorch evaluator for every row in \autoref{tab:meanflow_drift_ablation}. The main table uses deterministic five-step MeanFlow sampling for the MeanFlow-trained adapters and native five-step CM sampling for the CM-trained adapters. The native CM sampler uses the noise coefficients in Equation~\eqref{eq:endpoint_anchor_ablation} and a fixed evaluation seed, including for the constant-noise-trained adapter; this evaluation is stochastic. Each sampler is compared with its own pretrained baseline. One-step evaluation uses the same direct noise-to-data transition for both policies.

\FloatBarrier
\subsection{Single-Score and Mixed-Score Post-Training Details}
\label{appendix:t2i_details}

\paragraph{Base checkpoints.}
We use the FLUX.1-lite few-step generators released by T2I-Distill~\cite{pu2025few}. The released checkpoints include both MeanFlow and sCM variants. MeanFlow is treated as a two-time flow-map model, which directly parameterizes transitions $\psi^\theta_{t\to r}$, while sCM is treated as a one-time endpoint flow-map model, which predicts $\psi^\theta_{t\to 1}$. For each method, the ``Base'' model denotes the corresponding released T2I-Distill checkpoint without any RL post-training. The post-trained model denotes the same frozen base generator plus a reward-specific LoRA adapter trained with Flow-Map GRPO.

\paragraph{Reward-specific post-training.}
We train separate LoRA adapters for four reward objectives:
\begin{equation}
    R \in \{\mathrm{OCR},\ \mathrm{PickScore},\ \mathrm{GenEval},\ R_{\mathrm{mix}}\}.
\end{equation}
Single-reward adapters use the corresponding training prompt set; mixed-reward adapters sample OCR and PickScore prompts in equal proportions. In each case, the model is post-trained on these prompts and the reward is evaluated on the final generated image \cite{liu2026flow}. The runs share the policy objective and LoRA architecture; Table~\ref{tab:updated_runs} lists their training steps. In particular, the pretrained generator is frozen and only the LoRA adapter is updated. This setup allows us to evaluate whether Flow-Map GRPO can improve few-step flow-map generators without full-parameter finetuning.

\paragraph{Flow-Map GRPO rollout.}
During RL post-training, we use $K=4$ stochastic flow-map transitions for all models and all rewards. The rollout has six time nodes defining four stochastic policy transitions and a final deterministic endpoint transition. The stochastic transitions are used to compute policy likelihood ratios, while the final deterministic transition is used only to obtain the decoded image for reward evaluation. For a prompt $c$, we sample a group of $G$ trajectories, compute the final-image rewards, normalize the rewards into group-level advantages, and optimize the LoRA policy with the clipped GRPO objective in \autoref{eq:FM_GRPO}.

\paragraph{Common training hyperparameters.}
The following settings describe OCR, GenEval, and mixed-reward post-training. Per-device batch sizes depend on the experiment.

\begin{table}[H]
\centering
\small
\caption{Common Flow-Map GRPO post-training hyperparameters for text-to-image experiments.}
\label{tab:appendix_t2i_training_config}
\begin{tabular}{lc}
\toprule
\textbf{Hyperparameter} & \textbf{Value} \\
\midrule
Base model & T2I-Distill FLUX.1-lite MeanFlow / sCM checkpoints \\
Trainable parameters & LoRA only \\
Frozen parameters & Pretrained flow-map generator \\
Resolution & $512 \times 512$ \\
Guidance scale & $3.5$ \\
Max text sequence length & $256$ \\
Stochastic policy steps & $K=4$ \\
Images per prompt / group size & $G=24$ \\
Training batch size & Experiment-dependent \\
Number of GPUs & $8$ H100 GPUs \\
Mixed precision & bf16 \\
Gradient accumulation steps & $24$ \\
Optimizer & AdamW \\
Learning rate & $3\times 10^{-4}$ \\
Adam betas & $(0.9, 0.999)$ \\
Weight decay & $10^{-4}$ \\
Max gradient norm & $1.0$ \\
GRPO clip range & $10^{-4}$ \\
Advantage clipping & $5$ \\
KL coefficient $\beta$ & $0.0$ \\
EMA & enabled \\
EMA decay & $0.9$ \\
EMA update interval & $8$ optimizer steps \\
Seed & $42$ \\
\bottomrule
\end{tabular}
\end{table}

\paragraph{LoRA configuration.}
All adapters use rank 64, scaling factor 128, and zero dropout. Following \citet{liu2026flow}, LoRA is applied to the attention query, key, value, and output projections, including the additional conditioning projections, and to both feed-forward projections in the image and context streams.

\paragraph{Advantage normalization and GRPO objective.}
For each prompt group, rewards are gathered across devices and normalized before the policy update. We use per-prompt reward statistics together with a global batch standard deviation. Concretely, for a trajectory $i$ associated with prompt $c$, the advantage is computed as
\begin{equation}
    A_i
    =
    \frac{
    R(x^i_1|c)-\mu_c
    }{
    \sigma_{\mathrm{global}}+\epsilon_{\mathrm{std}}
    },
\end{equation}
where $\mu_c$ is the running per-prompt reward mean, $\sigma_{\mathrm{global}}$ is the reward standard deviation computed over the gathered rollout batch, and $\epsilon_{\mathrm{std}}=10^{-4}$ is a numerical stabilizer. We then clip $\widehat A_i=\operatorname{clip}(A_i,-5,5)$. The policy update recomputes the log probability of each stochastic ASFMC transition and applies the clipped GRPO objective.

\paragraph{Stochasticization during training.}
MeanFlow uses the local-anchor Gaussian transition in Equation~\eqref{eq:local_anchor_update}, with reverse-time coordinates $s=1-t$. Its average-velocity prediction $\bar u^\theta$ gives the deterministic transport $\psi^\theta_{s_i\to s_{i+1}}(x_{s_i})=x_{s_i}-(s_i-s_{i+1})\bar u^\theta_{s_i\to s_{i+1}}(x_{s_i})$. The instantaneous velocity is obtained by evaluating the same network with zero time interval, $u_\theta(x,s)=\bar u^\theta_{s\to s}(x)$. We evaluate this velocity at $x=\psi^\theta_{s_i\to s_{i+1}}(x_{s_i})$ and $s=s_{i+1}$, stop gradients through the velocity prediction, and apply the drift correction and Gaussian noise. We use $\lambda=7$ and $h_i=0.01(s_i-s_{i+1})$.

We take five uniformly spaced base nodes from $1$ to $0.05$, append $0$, and apply the timestep transformation used in T2I-Distill~\citep{pu2025few} to obtain the rollout times $s_i$, from which we compute $h_i$. The terminal base-grid node is $0.05$, and the injected noise standard deviation is $\lambda\sqrt{2h_i s_{i+1}/(1-s_{i+1})}$. The final endpoint segment is deterministic.

For sCM, $f_t^\theta(x_t)$ denotes its learned endpoint predictor, which takes the noisy latent $x_t$ and time $t$ as input and directly predicts the clean-image latent at the fixed data endpoint $t=1$. We use this prediction as the endpoint anchor, $X_1=f_t^\theta(x_t)=\psi^\theta_{t\to1}(x_t)$, in the policy of \autoref{sec:endpoint_anchor}. The resulting transition is $\widetilde X_r=a_rf_t^\theta(x_t)+b_r\xi$, with $\xi\sim\N(0,I)$, and requires one network call per stochastic step. For both backbones, only stochastic transitions contribute to the GRPO likelihood-ratio objective.

\paragraph{Reward mixing.}
Mixed-reward training samples OCR and PickScore prompts with equal probability. OCR prompts receive both rewards, whereas PickScore prompts receive only PickScore. We normalize each component separately within each reward-availability subset using pooled within-group normalization. For subset $a$, let $I_a$ contain its rollout samples, let $\bar R_{q,a}$ be the subset mean for component $q$, and let $\bar R_{q,g(i)}$ be its mean in prompt group $g(i)$. We compute
\begin{equation}
 s_{q,a}^2=\frac{1}{|I_a|}\sum_{i\in I_a}(R_{q,i}-\bar R_{q,g(i)})^2,
 \qquad \widetilde R_{q,i}=\frac{R_{q,i}-\bar R_{q,a}}{s_{q,a}+10^{-6}}.
\end{equation}
A component with scale at most $10^{-6}$ contributes zero. Available weights are renormalized to sum to one: $R_{\mathrm{mix}}=0.7\widetilde R_{\mathrm{OCR}}+0.3\widetilde R_{\mathrm{Pick}}$ applies to OCR prompts, and $R_{\mathrm{mix}}=\widetilde R_{\mathrm{Pick}}$ to PickScore prompts. We then apply the standard prompt-relative advantage normalization. The OCR/PickScore weights remain fixed at 0.7/0.3. For GenEval, we follow the standard evaluation procedure of \citet{ghosh2023geneval} and use the resulting compositional score as the reward.

\paragraph{Training steps.}
Table~\ref{tab:updated_runs} lists the training steps and hardware for the text-to-image experiments. Results use EMA adapters and one training seed per configuration.

\begin{table}[H]
\centering\small
\caption{Training steps for the evaluated text-to-image EMA adapters. GPU counts refer to training; offline evaluations use one H100.}
\label{tab:updated_runs}
\setlength{\tabcolsep}{3pt}
\begin{tabular}{llrr}
\toprule
Backbone & Reward & Checkpoint step & H100s \\
\midrule
MeanFlow & OCR & 800 & 8 \\
sCM & OCR & 700 & 8 \\
MeanFlow & PickScore & 1,500 & 8 \\
sCM & PickScore & 1,500 & 8 \\
MeanFlow & GenEval & 1,000 & 8 \\
sCM & GenEval & 1,000 & 8 \\
MeanFlow & OCR 0.7 + PickScore 0.3 & 1,000 & 8 \\
sCM & OCR 0.7 + PickScore 0.3 & 1,000 & 8 \\
\bottomrule
\end{tabular}
\end{table}

\paragraph{Evaluation setting.}
The sweeps are evaluated on one H100, at $512\times512$ resolution, guidance 3.5, and bf16 precision. Base and LoRA use matched prompts, sampling schedules, and random-number initialization within each family. MeanFlow disables post-flow resampling noise and uses per-sample seeds $42+i$. sCM OCR, PickScore, and DrawBench use the native stochastic sampler with global CUDA RNG seeded at 42 and $\sigma_{\mathrm{data}}=0.5$. sCM GenEval uses per-sample seeds. The tables use $K+1$ configured sampler steps. MeanFlow reports $K\in\{3,4,7,15\}$ and sCM reports $K\in\{2,3,4,7,15\}$.

Following the OCR scoring implementation of Flow-GRPO~\citep{liu2026flow}, we evaluate 1,018 test prompts using PaddleOCR~\citep{du2020ppocr}. Its score is one minus normalized, clipped Levenshtein distance to the first quoted string, after lowercasing and removing spaces; a substring match receives full credit. PickScore Task uses 2,048 separate preference-test prompts and raw PickScore logits~\citep{kirstain2023pickapic}. Following the standard GenEval evaluation~\citep{ghosh2023geneval}, we use 2,212 images (553 prompts, four samples each), the Mask2Former/CLIP scorer, and the equal-weight mean of six category accuracies.

We use the 200 prompts from DrawBench~\citep{saharia2022imagen}, generating five images per prompt (1,000 images). We evaluate PickScore~\citep{kirstain2023pickapic}, the LAION aesthetic predictor~\citep{schuhmann2022aesthetic}, DeQA~\citep{you2025deqa}, ImageReward~\citep{xu2023imagereward}, and UnifiedReward~\citep{wang2025unifiedreward} on these images. Both backbones use raw 1--5 DeQA and UniReward scores. DB Sum is the unweighted sum of the five DrawBench metrics.

\FloatBarrier
\subsection{Anchor Choice for MeanFlow Text-to-Image Post-Training}\label{app:meanflow_anchor_choice}
To assess anchor choice in our text-to-image setting, we compare deterministic sampling, our local-anchor policy, and the endpoint-anchor policy using the same pretrained FLUX.1-lite MeanFlow checkpoint before RL post-training. All three use five sampling steps ($K=4$), matched prompts and seeds, and the evaluation settings in Appendix~\ref{appendix:t2i_details}.

\begin{table}[H]
\centering\small
\caption{Five-step sampling with the pretrained FLUX.1-lite MeanFlow checkpoint in our text-to-image setting, before RL post-training. Metric definitions follow Appendix~\ref{appendix:t2i_details}; DB Sum includes all five DrawBench metrics, including those not individually displayed.}
\label{tab:meanflow_anchor_choice}
\setlength{\tabcolsep}{4pt}
\begin{tabular}{lcc@{\hspace{0.7em}}ccc}
\toprule
\multirow{2}{*}{\textbf{Sampler}} & \multicolumn{2}{c}{\textbf{Task Metrics}} & \multicolumn{3}{c}{\textbf{DrawBench Metrics}} \\
\cmidrule(lr){2-3}\cmidrule(lr){4-6}
& \textbf{OCR}$\uparrow$ & \textbf{PickScore}$\uparrow$ & \textbf{DeQA}$\uparrow$ & \textbf{ImgRwd}$\uparrow$ & \textbf{DB Sum}$\uparrow$ \\
\midrule
Deterministic (standard evaluation) & 0.349 & 20.98 & 4.17 & 0.39 & 34.45 \\
Local anchor (Ours) & 0.444 & 20.72 & 3.03 & 0.38 & 32.68 \\
Endpoint anchor (full noise) & 0.208 & 17.51 & 2.68 & $-1.80$ & 25.70 \\
\bottomrule
\end{tabular}
\end{table}

Under this text-to-image setting, Table~\ref{tab:meanflow_anchor_choice} shows lower scores for endpoint-anchor sampling than for deterministic sampling and our local-anchor policy on every reported metric. These results motivate our choice of the local anchor for T2I post-training of this FLUX.1-lite MeanFlow checkpoint.

\FloatBarrier
\subsection{Baseline Setting and Challenges Posed by Flow Maps}
\label{app:baseline_objectives}
Table~\ref{tab:method_comparison} and Figure~\ref{fig:training_curves} compare SFT and DiffusionNFT, representing two widely used classes of post-training methods, with Flow-Map GRPO under matched training and evaluation settings. The central challenge for applying these post-training methods to flow maps arises from the prediction quantities learned by flow maps: a long-range transport or endpoint map is generally neither an instantaneous velocity nor a posterior-mean denoiser. Native flow-map supervision also couples these predictions to an underlying instantaneous field through transport or consistency identities. Consequently, reward-selected endpoints can conflict with a frozen teacher's dynamics, while forward-noising regression can target a different quantity from the learned flow map. These are objective mismatches rooted in flow-map structure; matching optimization settings alone does not resolve them. This highlights the value of Flow-Map GRPO, which addresses the challenges posed by flow maps in RL post-training.

\paragraph{Matched comparison settings.}
We control the shared experimental settings within each backbone and task: the pretrained checkpoint, training and test prompt splits, reward function, LoRA-only adaptation, $512\times512$ resolution, guidance scale 3.5, text length 256, five-step collection and evaluation, and EMA evaluation. The recipes use AdamW with learning rate $3\times10^{-4}$, 24 generated images per prompt, and gradient accumulation of 24. Table~\ref{tab:method_comparison} matches the optimizer-step endpoint across methods within each column: 700 for sCM OCR, 800 for MeanFlow OCR, and 1,500 for both PickScore comparisons. The comparison uses the same number of optimizer steps. Evaluation settings are specified in Appendix~\ref{appendix:t2i_details}.

\paragraph{Online SFT objective.}
For each condition $c$, let $\{(x_1^{(i)},R_i)\}_{i=1}^{G}$ denote the collected clean latents (at the main text's data endpoint $t=1$) and their rewards. Let $F_\theta^{(i)}$ and $F_{\mathrm{base}}^{(i)}$ denote the current and frozen-base backbone predictions at the same noisy input, conditioning, and sampled times, with these arguments suppressed. Define the per-sample reference penalty as $\ell_{\mathrm{ref},i}=d^{-1}\|F_\theta^{(i)}-\sg[F_{\mathrm{base}}^{(i)}]\|^2$, where $d$ is the number of prediction elements. Following Flow-GRPO~\citep{liu2026flow}, online SFT keeps all samples tied for the highest reward within each condition and minimizes
\begin{equation}
 m_i=\mathbf{1}\!\left\{R_i=\max_{1\leq k\leq G}R_k\right\},\qquad
 \mathcal L_{\mathrm{SFT}}=\mathbb E\!\left[\frac{1}{G}\sum_{i=1}^{G}
 \left(m_i\ell_{\mathrm{native},i}+\alpha_{\mathrm{SFT}}\ell_{\mathrm{ref},i}\right)\right].
\label{eq:online_sft_objective}
\end{equation}
Here $\ell_{\mathrm{native},i}$ is the backbone's native MeanFlow or sCM loss, detailed below, and $\alpha_{\mathrm{SFT}}=10$. The expectation includes conditions, collected samples, and native-loss time and noise sampling. The selection mask weights only the native loss; the reference penalty applies to every collected sample. The collection model is periodically refreshed from the learner, while the base reference and velocity teacher remain frozen.

\paragraph{DiffusionNFT objective.}
Following DiffusionNFT~\citep{zheng2026diffusionnft}, the reward determines a weight between positive and negative regression. Using the same reward notation as the main text, we set $\mu_c=G^{-1}\sum_{k=1}^{G}R_k$, compute $\sigma_{\mathrm{global}}$ over the collected batch, and form $\widehat A_i=\operatorname{clip}((R_i-\mu_c)/(\sigma_{\mathrm{global}}+\epsilon_{\mathrm{std}}),-A_{\max},A_{\max})$. The NFT weight is $w_i=\frac12+\widehat A_i/(2A_{\max})$, with $A_{\max}=5$ and $\epsilon_{\mathrm{std}}=10^{-4}$. Let $F_{\theta_{\mathrm{old}}}^{(i)}$ be the detached prediction of the collection model. The positive and negative predictions are
\begin{equation}
 F_i^{\pm}=\sg[F_{\theta_{\mathrm{old}}}^{(i)}]\pm\beta_{\mathrm{NFT}}\bigl(F_\theta^{(i)}-\sg[F_{\theta_{\mathrm{old}}}^{(i)}]\bigr),\qquad
 \widehat x_1^{(i,\pm)}=\mathcal C_i(F_i^{\pm}),\qquad \beta_{\mathrm{NFT}}=0.1.
\label{eq:nft_positive_negative}
\end{equation}
Here $F_\theta^{(i)}$ denotes the raw backbone prediction, not the endpoint map $f_t^\theta$ in the main text. The reconstruction $\mathcal C_i$ makes this distinction explicit. For MeanFlow, use reverse-time coordinates $s=1-t$ and a sampled target time $s'\leq s$. The noisy input is $x_s^{(i)}=(1-s)x_1^{(i)}+s\epsilon_i$, with $\epsilon_i\sim\mathcal N(0,I)$, and $F_\theta^{(i)}=\bar u_{s\to s'}^\theta(x_s^{(i)})$ is the reverse-time average velocity defined in Appendix~\ref{appendix:t2i_details}. NFT reconstructs $\mathcal C_i(F)=x_s^{(i)}-sF$. For sCM, let $\omega$ denote its native angular noise coordinate and set $x_\omega^{(i)}=\cos(\omega)x_1^{(i)}+\sigma_{\mathrm{data}}\sin(\omega)\epsilon_i$. Its reconstruction is $\mathcal C_i(F)=\cos(\omega)x_\omega^{(i)}-\sigma_{\mathrm{data}}\sin(\omega)F$, where $\sigma_{\mathrm{data}}=0.5$. The sCM network receives $x_\omega^{(i)}/\sigma_{\mathrm{data}}$ and $\omega$; $\mathcal C_i(F_\theta^{(i)})$ is its clean-endpoint prediction. Conditioning is suppressed in both expressions.

For each reconstruction, adaptive normalization uses the detached mean absolute error:
\begin{equation}
\begin{aligned}
 D(\widehat x,x)&=\frac{d^{-1}\|\widehat x-x\|^2}
 {\max\!\left(\sg[d^{-1}\|\widehat x-x\|_1],10^{-5}\right)},\\
 \mathcal L_{\mathrm{NFT}}&=\mathbb E\!\left[\frac{1}{G}\sum_{i=1}^{G}
 \left(w_iD(\widehat x_1^{(i,+)},x_1^{(i)})+(1-w_i)D(\widehat x_1^{(i,-)},x_1^{(i)})
 +\alpha_{\mathrm{NFT}}\ell_{\mathrm{ref},i}\right)\right].
\end{aligned}
\label{eq:nft_objective}
\end{equation}
where $\alpha_{\mathrm{NFT}}=0.01$. Collected targets, reward weights, old-model predictions, and normalization factors are detached during each update. After each collection-and-update round, the old model is updated as $\theta_{\mathrm{old}}\leftarrow\eta_{\mathrm{old}}\theta_{\mathrm{old}}+(1-\eta_{\mathrm{old}})\theta$, on the LoRA parameters, with the scheduled EMA decay $\eta_{\mathrm{old}}$; the base reference remains frozen.

\paragraph{Native SFT retains an unadapted supervision field.}
Online SFT collects generated images, selects high-reward samples, and applies the backbone's native training loss together with a frozen-base reference penalty. This loss supervises more than the selected clean endpoint: it combines a velocity or tangent target with a JVP-based self-consistency condition. In our MeanFlow implementation, the instantaneous target comes from a frozen velocity teacher, while the JVP supplies the finite-interval correction. The sCM implementation likewise uses an independent frozen velocity teacher to define the tangent direction in its native consistency loss. Both use the pretrained rescaled FLUX velocity model; the frozen few-step base additionally provides reference regularization.

For the forward-time convention $t=0$ for noise and $t=1$ for data, the MeanFlow target combines the frozen instantaneous field $u_t^\phi$ with derivatives of the average-velocity prediction:
\begin{equation}
\begin{aligned}
 T_\theta(x,t,r)&=u_t^\phi(x)+(r-t)
 \left[\partial_t u_{t\to r}^\theta(x)+D_xu_{t\to r}^\theta(x)u_t^\phi(x)\right],\\
 \mathcal L_{\mathrm{MF}}(\theta)&=\E\|u_{t\to r}^\theta(X_t)-\sg[T_\theta(X_t,t,r)]\|^2.
\end{aligned}
\label{eq:meanflow_sft_target}
\end{equation}
Here $D_x$ is the spatial Jacobian and $\sg$ denotes stop-gradient. The target follows from the flow-map transport identity $\partial_t\psi_{t\to r}+D_x\psi_{t\to r}u_t=0$. For the sCM endpoint predictor $f_t^\theta$, the same identity yields the consistency residual $\partial_tf_t^\theta+D_xf_t^\theta u_t^\phi$, with $f_1(x)=x$~\citep{lu2025simplifying}. Its squared residual, with the native weighting and stabilization, couples endpoint learning to the teacher's tangent direction. These relations identify the dependence on the unadapted instantaneous field in both native losses.

Reward selection changes the endpoint distribution but does not update this teacher. To see the resulting mismatch, let $p_{\mathrm{sel}}(x_1\mid c)$ denote the reward-selected distribution and use the linear path $x_t=t x_1+(1-t)\epsilon$, with $\epsilon\sim\mathcal N(0,I)$. The instantaneous field associated with these endpoints is
\begin{equation}
 u_{\mathrm{sel}}(x_t,t,c)
 =\mathbb{E}_{x_1\sim p_{\mathrm{sel}}}
 [x_1-\epsilon\mid x_t,t,c].
\end{equation}
The frozen teacher instead approximates a field associated with its original pretraining distribution. In general, reward selection changes the conditional expectation above, so the two fields need not agree. Increasing the frequency of high-reward endpoints therefore does not supply a correspondingly adapted instantaneous target. The JVP/consistency term continues to couple the student to the frozen teacher's dynamics, and the base-reference penalty further discourages departure from the pretrained predictor. Consequently, the unchanged native objective is not, in general, a consistent distillation objective for the reward-selected distribution: the endpoint samples and supervisory dynamics refer to different distributions. Reward filtering changes the endpoint distribution without updating the teacher's instantaneous field, leaving the native flow-map loss with inconsistent supervision.

\paragraph{Flow-map outputs differ from forward-noising targets.}
The distributional interpretation of DiffusionNFT relies on instantaneous-velocity identities (Theorems~3.1--3.2 and Appendix~A of \citealp{zheng2026diffusionnft}). Let $p_{\mathrm{old}}(x_1\mid c)$ be the collected endpoint distribution and $w(x_1,c)\in[0,1]$ its reward-derived weight. The positive distribution is $p_+(x_1\mid c)=w(x_1,c)p_{\mathrm{old}}(x_1\mid c)/Z_c$, where $Z_c=\mathbb E_{p_{\mathrm{old}}}[w(X_1,c)]$; the negative distribution is defined analogously using $1-w$. Their proof converts a mixture identity for the corresponding posteriors into an identity for their instantaneous velocity fields, using the affine relation between instantaneous velocity and posterior mean. This establishes NFT's reinforcement-guidance direction; an appropriate guidance strength recovers the positive velocity field and hence $p_+$. Thus, applying this argument requires the predictor to represent the instantaneous velocity, or an equivalent denoising quantity satisfying the same identity. Substituting a long-range average velocity or transported endpoint does not meet that requirement and therefore does not inherit the claimed connection between the regression objective and the reward-preferred distribution.

DiffusionNFT~\citep{zheng2026diffusionnft} builds on forward-noising regression, whose prediction target is determined by the conditional distribution at a single noise level. To make the distinction explicit, consider the underlying squared-error regression with $x_s=(1-s)x_1+s\epsilon$ in reverse-time coordinates $s=1-t$. Its instantaneous-velocity target is $\epsilon-x_1$, so the population minimizer is $u_s^*(x_s)=\mathbb E[\epsilon-x_1\mid x_s,c]$. Consequently, $x_s-su_s^*(x_s)=\mathbb E[x_1\mid x_s,c]$: the corresponding reconstruction is a posterior-mean denoiser. These identities describe the prediction quantities underlying the regression, rather than asserting an unchanged minimizer for the reward-weighted NFT objective.

MeanFlow instead learns the interval-average velocity $\bar u_{s\to s'}^\theta(x_s)$ defining the long-range transport $\psi_{s\to s'}^\theta(x_s)=x_s-(s-s')\bar u_{s\to s'}^\theta(x_s)$. Only the equal-time limit $s'\to s$ recovers the instantaneous velocity; a finite-interval average generally differs from $u_s(x_s)$. In the NFT reconstruction in Equation~\eqref{eq:nft_positive_negative}, using this average as $x_s-s\bar u_{s\to s'}^\theta(x_s)$ therefore subjects a long-range prediction to an instantaneous-prediction objective. The distinction is not merely the coefficient $s$ versus $s-s'$: even when $s'=0$, the transported endpoint $\psi_{s\to0}(x_s)$ need not equal $\mathbb E[x_1\mid x_s,c]$. Transport follows the marginal dynamics, whereas the conditional mean averages the clean samples compatible with a noisy observation. Replacing the coefficient alone cannot reconcile these quantities; the flow-map transport identity introduces the derivative correction shown in Equation~\eqref{eq:meanflow_sft_target}.

The same distinction applies to sCM. Its endpoint prediction $f_t^\theta(x_t)$ approximates the endpoint of a trajectory, while forward-noising regression targets a conditional denoising statistic. For an accurate flow map, $f_t(X_t)$ has the data marginal when $X_t\sim p_t$; a posterior mean generally has reduced variability because it averages over uncertain clean endpoints. For example, under the law of total covariance, $\operatorname{Cov}(X_1)=\operatorname{Cov}(\mathbb E[X_1\mid X_t])+\mathbb E[\operatorname{Cov}(X_1\mid X_t)]$. With nonzero posterior uncertainty and finite second moments, the posterior mean therefore cannot have the same covariance as the data. Thus, a flow-map endpoint and a posterior-mean denoiser are generally different functions, even when both are written as clean-sample predictions. NFT's positive/negative regression and adaptive normalization in Equation~\eqref{eq:nft_objective} do not establish the missing identity between them. This mismatch arises from the quantities represented by flow maps and requires a compatible objective, rather than a change to shared optimization settings.

\paragraph{Relation to RWR and DPO.}
Reward-weighted regression (RWR)~\citep{lee2023aligning} and Diffusion-DPO~\citep{wallace2023diffusion} also use local denoising losses, respectively as a weighted fitting objective and a diffusion-likelihood surrogate. Average-velocity and consistency-map outputs do not generally satisfy the conditional-prediction identities underlying these losses. This observation concerns their standard diffusion implementations, rather than RWR or DPO as general optimization principles. We select DiffusionNFT as a representative of this family based on its reported performance in diffusion post-training~\citep{zheng2026diffusionnft}.

\paragraph{Direct optimization of the flow-map policy.}
Flow-Map GRPO resolves both dependencies by defining a stochastic policy around the pretrained flow map and optimizing its tractable transition likelihood ratios using rollout rewards. The update acts on the long-range generator used by the policy, without requiring a separately reward-adapted instantaneous teacher or identifying the flow-map output with a local denoising statistic. This is the relevant distinction behind the comparison in Table~\ref{tab:method_comparison}.

\FloatBarrier
\subsection{OCR and PickScore Training Curves}\label{appendix:ocr_training_curves}
Figure~\ref{fig:training_curves} compares SFT, NFT, and Flow-Map GRPO using full-test rewards evaluated every 20 training steps with EMA weights and five sampler steps ($K=4$). OCR runs are compared through step 700 for sCM and 800 for MeanFlow; PickScore runs extend to step 1,500. The curves show training-time evaluations without additional smoothing. Dataset sizes and evaluation conventions are given in Appendix~\ref{appendix:t2i_details}.

\begin{figure}[H]
\centering
\includegraphics[width=\linewidth]{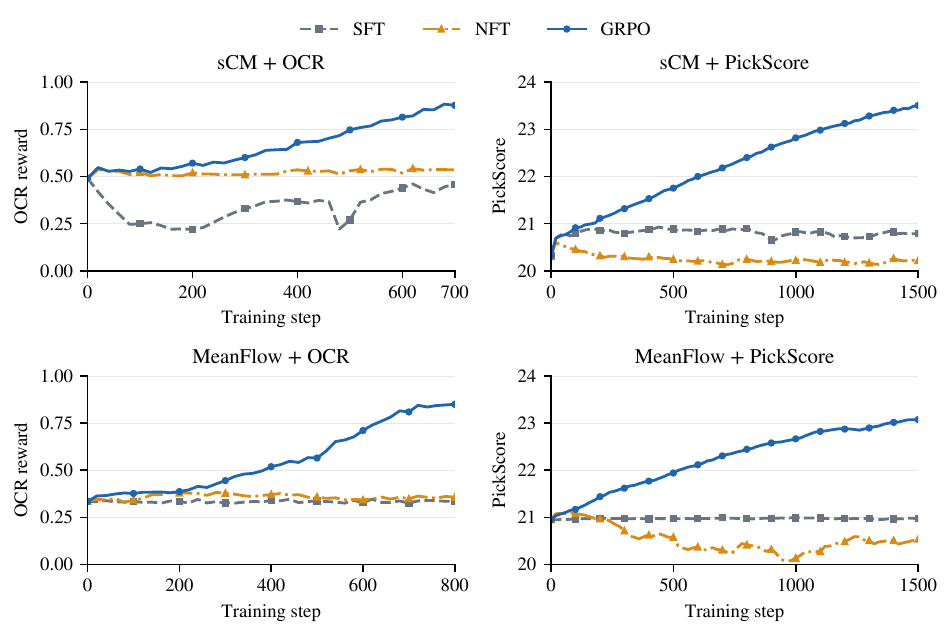}
\caption{Training-time EMA full-evaluation rewards for sCM (top) and MeanFlow (bottom), optimizing OCR (left) or PickScore (right). OCR is shown through steps 700 and 800, respectively; both PickScore axes extend to 1500. GRPO curves share the corresponding SFT initialization reference.}
\label{fig:training_curves}
\label{fig:meanflow_ocr_training}
\end{figure}

\section{Complete Text-to-Image Results}\label{app:full_results}
We evaluate the post-trained models across multiple numbers of sampling steps to assess transfer to inference settings that differ from the five-step training sampler ($K=4$). The same EMA adapters are evaluated without further training and compared with Base at each inference step count. Task Metrics measure RL improvements on the target tasks using their respective test sets. DrawBench Metrics are auxiliary general-quality metrics used to monitor degradation outside the optimized reward. UniRwd and DeQA use raw 1--5 scores for both backbones. DB Sum denotes the unweighted sum of the five displayed quality metrics. All numbers are single-run point estimates.

\subsection{OCR-Only Evaluations}\label{app:ocr_sweeps}
\begin{table}[H]
\centering\small
\caption{Evaluation of post-trained MeanFlow at different numbers of sampling steps (OCR-only reward). Task Metrics measure RL improvements on the target tasks; DrawBench Metrics are auxiliary general-quality metrics used to monitor degradation outside the optimized reward. Evaluation settings are given in Appendix~\ref{appendix:t2i_details}.}
\label{tab:ocr_meanflow_uniform}
\setlength{\tabcolsep}{1.4pt}
\renewcommand{\arraystretch}{1.08}
\resizebox{\linewidth}{!}{
\begin{tabular}{rlccc@{\hspace{0.5em}}|@{\hspace{0.5em}}cccccc}
\toprule
\multirow{2}{*}{$K$} & \multirow{2}{*}{\textbf{Model}} & \multicolumn{3}{c}{\textbf{Task Metrics}} & \multicolumn{6}{c}{\textbf{DrawBench Metrics}} \\
\cmidrule(lr){3-5}\cmidrule(lr){6-11}
 & & \textbf{GenEval}$\uparrow$ & \shortstack{\textbf{OCR}\\\textbf{Acc.}$\uparrow$} & \shortstack{\textbf{PickScore}\\\textbf{Task}$\uparrow$} & \shortstack{\textbf{DB-}\\\textbf{PickScore}$\uparrow$} & \textbf{Aesthetic}$\uparrow$ & \textbf{DeQA}$\uparrow$ & \textbf{ImgRwd}$\uparrow$ & \textbf{UniRwd}$\uparrow$ & \textbf{DB Sum}$\uparrow$ \\
\midrule
3 & MeanFlow Base & -- & 0.2876 & -- & 21.3577 & \textbf{5.4391} & \textbf{4.0001} & 0.2157 & 2.5495 & 33.5621 \\
3 & MeanFlow + Flow-Map GRPO LoRA & -- & \textbf{0.7730} & -- & \textbf{21.5779} & 5.3550 & 3.8437 & \textbf{0.3076} & \textbf{2.5688} & \textbf{33.6530} \\
\midrule
4 & MeanFlow Base & -- & 0.3491 & -- & 21.6298 & \textbf{5.5368} & \textbf{4.1712} & 0.3918 & 2.7187 & 34.4483 \\
4 & MeanFlow + Flow-Map GRPO LoRA & -- & \textbf{0.8504} & -- & \textbf{21.9869} & 5.5111 & 4.1396 & \textbf{0.5386} & \textbf{2.8049} & \textbf{34.9812} \\
\midrule
7 & MeanFlow Base & -- & 0.4517 & -- & 22.0761 & 5.6090 & 4.2500 & 0.5744 & 2.9216 & 35.4311 \\
7 & MeanFlow + Flow-Map GRPO LoRA & -- & \textbf{0.9256} & -- & \textbf{22.4061} & \textbf{5.6454} & \textbf{4.3415} & \textbf{0.7855} & \textbf{3.0262} & \textbf{36.2047} \\
\midrule
15 & MeanFlow Base & -- & 0.5507 & -- & 22.2302 & 5.6489 & 4.3217 & 0.6870 & 3.0208 & 35.9086 \\
15 & MeanFlow + Flow-Map GRPO LoRA & -- & \textbf{0.9385} & -- & \textbf{22.5182} & \textbf{5.6824} & \textbf{4.4106} & \textbf{0.8755} & \textbf{3.0968} & \textbf{36.5835} \\
\bottomrule
\end{tabular}
}
\end{table}

\begin{table}[H]
\centering\small
\caption{Evaluation of post-trained sCM at different numbers of sampling steps (OCR-only reward). Task Metrics measure RL improvements on the target tasks; DrawBench Metrics are auxiliary general-quality metrics used to monitor degradation outside the optimized reward. Evaluation settings are given in Appendix~\ref{appendix:t2i_details}.}
\label{tab:ocr_scm_uniform}
\setlength{\tabcolsep}{1.4pt}
\renewcommand{\arraystretch}{1.08}
\resizebox{\linewidth}{!}{
\begin{tabular}{rlccc@{\hspace{0.5em}}|@{\hspace{0.5em}}cccccc}
\toprule
\multirow{2}{*}{$K$} & \multirow{2}{*}{\textbf{Model}} & \multicolumn{3}{c}{\textbf{Task Metrics}} & \multicolumn{6}{c}{\textbf{DrawBench Metrics}} \\
\cmidrule(lr){3-5}\cmidrule(lr){6-11}
 & & \textbf{GenEval}$\uparrow$ & \shortstack{\textbf{OCR}\\\textbf{Acc.}$\uparrow$} & \shortstack{\textbf{PickScore}\\\textbf{Task}$\uparrow$} & \shortstack{\textbf{DB-}\\\textbf{PickScore}$\uparrow$} & \textbf{Aesthetic}$\uparrow$ & \textbf{DeQA}$\uparrow$ & \textbf{ImgRwd}$\uparrow$ & \textbf{UniRwd}$\uparrow$ & \textbf{DB Sum}$\uparrow$ \\
\midrule
2 & sCM Base & -- & 0.3394 & -- & \textbf{21.4149} & \textbf{5.3850} & \textbf{3.4327} & \textbf{0.5623} & \textbf{2.6225} & \textbf{33.4174} \\
2 & sCM + Flow-Map GRPO LoRA & -- & \textbf{0.7608} & -- & 21.1626 & 5.1030 & 3.1430 & 0.1033 & 2.4512 & 31.9631 \\
\midrule
3 & sCM Base & -- & 0.4290 & -- & 21.4524 & \textbf{5.4448} & \textbf{3.6554} & \textbf{0.5742} & \textbf{2.6727} & \textbf{33.7994} \\
3 & sCM + Flow-Map GRPO LoRA & -- & \textbf{0.8592} & -- & \textbf{21.4766} & 5.3037 & 3.5807 & 0.2819 & 2.6603 & 33.3033 \\
\midrule
4 & sCM Base & -- & 0.4903 & -- & 21.4993 & \textbf{5.4747} & 3.7469 & \textbf{0.5750} & 2.6725 & \textbf{33.9684} \\
4 & sCM + Flow-Map GRPO LoRA & -- & \textbf{0.8863} & -- & \textbf{21.5647} & 5.4049 & \textbf{3.7800} & 0.3475 & \textbf{2.7137} & 33.8107 \\
\midrule
7 & sCM Base & -- & 0.6120 & -- & 21.5128 & 5.5084 & 3.8796 & \textbf{0.6037} & 2.7196 & 34.2241 \\
7 & sCM + Flow-Map GRPO LoRA & -- & \textbf{0.8825} & -- & \textbf{21.6504} & \textbf{5.5197} & \textbf{4.0625} & 0.4857 & \textbf{2.8560} & \textbf{34.5743} \\
\midrule
15 & sCM Base & -- & 0.6583 & -- & 21.3638 & 5.5358 & 3.8862 & \textbf{0.5312} & 2.6591 & 33.9761 \\
15 & sCM + Flow-Map GRPO LoRA & -- & \textbf{0.8988} & -- & \textbf{21.5153} & \textbf{5.5614} & \textbf{4.1601} & 0.4618 & \textbf{2.7936} & \textbf{34.4922} \\
\bottomrule
\end{tabular}
}
\end{table}

\subsection{PickScore-Only Evaluations}\label{app:pickscore_sweeps}
These two tables report the MeanFlow and sCM PickScore-only adapters with EMA weights at training step 1,500. Base rows use the evaluation setting in Appendix~\ref{appendix:t2i_details}. Both backbones use raw 1--5 UniRwd scores, and DB Sum denotes the unweighted sum of the five DrawBench metrics. Each value $K$ corresponds to $K+1$ sampling steps.

\begin{table}[H]
\centering
\small
\setlength{\tabcolsep}{3.5pt}
\renewcommand{\arraystretch}{1.08}
\caption{Evaluation of post-trained MeanFlow at different numbers of sampling steps (PickScore-only reward). Task Metrics measure RL improvements on the target tasks; DrawBench Metrics are auxiliary general-quality metrics used to monitor degradation outside the optimized reward. Evaluation settings are given in Appendix~\ref{appendix:t2i_details}.}
\label{tab:pickscore_meanflow_uniform}
\resizebox{\linewidth}{!}{
\begin{tabular}{c l c c c @{\hspace{0.8em}}|@{\hspace{0.8em}} c c c c c c}
\toprule
\multirow{2}{*}{\textbf{K}}
& \multirow{2}{*}{\textbf{Model}}
& \multicolumn{3}{c}{\textbf{Task Metrics}}
& \multicolumn{6}{c}{\textbf{DrawBench Metrics}} \\
\cmidrule(lr){3-5}
\cmidrule(lr){6-11}
&
& \textbf{GenEval}$\uparrow$
& \textbf{OCR Acc.}$\uparrow$
& \textbf{PickScore Task}$\uparrow$
& \textbf{DB-PickScore}$\uparrow$
& \textbf{Aesthetic}$\uparrow$
& \textbf{DeQA}$\uparrow$
& \textbf{ImgRwd}$\uparrow$
& \textbf{UniRwd}$\uparrow$
& \textbf{DB Sum}$\uparrow$ \\
\midrule
\multirow{2}{*}{3} & MeanFlow Base & -- & -- & 20.6926 & 21.3577 & 5.4391 & \textbf{4.0001} & 0.2157 & 2.5495 & 33.5621 \\
 & MeanFlow + Flow-Map GRPO LoRA & -- & -- & \textbf{22.7035} & \textbf{22.8560} & \textbf{5.9958} & 3.7091 & \textbf{1.0043} & \textbf{3.0077} & \textbf{36.5730} \\
\midrule
\multirow{2}{*}{4} & MeanFlow Base & -- & -- & 20.9758 & 21.6298 & 5.5368 & \textbf{4.1712} & 0.3918 & 2.7187 & 34.4483 \\
 & MeanFlow + Flow-Map GRPO LoRA & -- & -- & \textbf{23.0772} & \textbf{23.3711} & \textbf{6.1510} & 4.0829 & \textbf{1.1402} & \textbf{3.2005} & \textbf{37.9456} \\
\midrule
\multirow{2}{*}{7} & MeanFlow Base & -- & -- & 21.3595 & 22.0761 & 5.6090 & 4.2500 & 0.5744 & 2.9216 & 35.4311 \\
 & MeanFlow + Flow-Map GRPO LoRA & -- & -- & \textbf{23.2963} & \textbf{23.6881} & \textbf{6.2571} & \textbf{4.3427} & \textbf{1.1950} & \textbf{3.3168} & \textbf{38.7997} \\
\midrule
\multirow{2}{*}{15} & MeanFlow Base & -- & -- & 21.4382 & 22.2302 & 5.6489 & 4.3217 & 0.6870 & 3.0208 & 35.9086 \\
 & MeanFlow + Flow-Map GRPO LoRA & -- & -- & \textbf{23.3229} & \textbf{23.7649} & \textbf{6.2959} & \textbf{4.4283} & \textbf{1.2246} & \textbf{3.3445} & \textbf{39.0582} \\
\bottomrule
\end{tabular}
}
\end{table}

\begin{table}[H]
\centering
\small
\setlength{\tabcolsep}{3.5pt}
\renewcommand{\arraystretch}{1.08}
\caption{Evaluation of post-trained sCM at different numbers of sampling steps (PickScore-only reward). Task Metrics measure RL improvements on the target tasks; DrawBench Metrics are auxiliary general-quality metrics used to monitor degradation outside the optimized reward. Evaluation settings are given in Appendix~\ref{appendix:t2i_details}.}
\label{tab:pickscore_scm_uniform}
\resizebox{\linewidth}{!}{
\begin{tabular}{c l c c c @{\hspace{0.8em}}|@{\hspace{0.8em}} c c c c c c}
\toprule
\multirow{2}{*}{\textbf{K}}
& \multirow{2}{*}{\textbf{Model}}
& \multicolumn{3}{c}{\textbf{Task Metrics}}
& \multicolumn{6}{c}{\textbf{DrawBench Metrics}} \\
\cmidrule(lr){3-5}
\cmidrule(lr){6-11}
&
& \textbf{GenEval}$\uparrow$
& \textbf{OCR Acc.}$\uparrow$
& \textbf{PickScore Task}$\uparrow$
& \textbf{DB-PickScore}$\uparrow$
& \textbf{Aesthetic}$\uparrow$
& \textbf{DeQA}$\uparrow$
& \textbf{ImgRwd}$\uparrow$
& \textbf{UniRwd}$\uparrow$
& \textbf{DB Sum}$\uparrow$ \\
\midrule
\multirow{2}{*}{2} & sCM Base & -- & -- & 20.3023 & 21.4149 & 5.3850 & 3.4327 & 0.5623 & 2.6225 & 33.4174 \\
 & sCM + Flow-Map GRPO LoRA & -- & -- & \textbf{23.2390} & \textbf{23.5695} & \textbf{6.0445} & \textbf{3.7773} & \textbf{1.1965} & \textbf{3.2104} & \textbf{37.7981} \\
\midrule
\multirow{2}{*}{3} & sCM Base & -- & -- & 20.2942 & 21.4524 & 5.4448 & 3.6554 & 0.5742 & 2.6727 & 33.7994 \\
 & sCM + Flow-Map GRPO LoRA & -- & -- & \textbf{23.4376} & \textbf{23.7651} & \textbf{6.0433} & \textbf{4.0933} & \textbf{1.2456} & \textbf{3.3186} & \textbf{38.4659} \\
\midrule
\multirow{2}{*}{4} & sCM Base & -- & -- & 20.3517 & 21.4993 & 5.4747 & 3.7469 & 0.5750 & 2.6725 & 33.9684 \\
 & sCM + Flow-Map GRPO LoRA & -- & -- & \textbf{23.4820} & \textbf{23.8084} & \textbf{6.0360} & \textbf{4.2343} & \textbf{1.2600} & \textbf{3.3361} & \textbf{38.6748} \\
\midrule
\multirow{2}{*}{7} & sCM Base & -- & -- & 20.2940 & 21.5128 & 5.5084 & 3.8796 & 0.6037 & 2.7196 & 34.2241 \\
 & sCM + Flow-Map GRPO LoRA & -- & -- & \textbf{23.4160} & \textbf{23.7285} & \textbf{6.0048} & \textbf{4.3955} & \textbf{1.2236} & \textbf{3.3273} & \textbf{38.6797} \\
\midrule
\multirow{2}{*}{15} & sCM Base & -- & -- & 20.0556 & 21.3638 & 5.5358 & 3.8862 & 0.5312 & 2.6591 & 33.9761 \\
 & sCM + Flow-Map GRPO LoRA & -- & -- & \textbf{23.0871} & \textbf{23.4470} & \textbf{5.9367} & \textbf{4.4677} & \textbf{1.1595} & \textbf{3.2983} & \textbf{38.3090} \\
\bottomrule
\end{tabular}
}
\end{table}

\subsection{GenEval Evaluations}\label{app:geneval_sweeps}
\begin{table}[H]
\centering\small
\caption{Evaluation of post-trained MeanFlow at different numbers of sampling steps (GenEval reward). Task Metrics measure RL improvements on the target tasks; DrawBench Metrics are auxiliary general-quality metrics used to monitor degradation outside the optimized reward. Evaluation settings are given in Appendix~\ref{appendix:t2i_details}.}
\label{tab:geneval_meanflow_uniform}
\setlength{\tabcolsep}{1.4pt}
\renewcommand{\arraystretch}{1.08}
\resizebox{\linewidth}{!}{
\begin{tabular}{rlccc@{\hspace{0.5em}}|@{\hspace{0.5em}}cccccc}
\toprule
\multirow{2}{*}{$K$} & \multirow{2}{*}{\textbf{Model}} & \multicolumn{3}{c}{\textbf{Task Metrics}} & \multicolumn{6}{c}{\textbf{DrawBench Metrics}} \\
\cmidrule(lr){3-5}\cmidrule(lr){6-11}
 & & \textbf{GenEval}$\uparrow$ & \shortstack{\textbf{OCR}\\\textbf{Acc.}$\uparrow$} & \shortstack{\textbf{PickScore}\\\textbf{Task}$\uparrow$} & \shortstack{\textbf{DB-}\\\textbf{PickScore}$\uparrow$} & \textbf{Aesthetic}$\uparrow$ & \textbf{DeQA}$\uparrow$ & \textbf{ImgRwd}$\uparrow$ & \textbf{UniRwd}$\uparrow$ & \textbf{DB Sum}$\uparrow$ \\
\midrule
3 & MeanFlow Base & 0.4684 & -- & -- & 21.3577 & \textbf{5.4391} & \textbf{4.0001} & 0.2157 & 2.5495 & \textbf{33.5621} \\
3 & MeanFlow + Flow-Map GRPO LoRA & \textbf{0.7938} & -- & -- & \textbf{21.4279} & 5.2559 & 2.8478 & \textbf{0.3979} & \textbf{2.6137} & 32.5431 \\
\midrule
4 & MeanFlow Base & 0.5041 & -- & -- & 21.6298 & \textbf{5.5368} & \textbf{4.1712} & 0.3918 & 2.7187 & \textbf{34.4483} \\
4 & MeanFlow + Flow-Map GRPO LoRA & \textbf{0.8454} & -- & -- & \textbf{21.8184} & 5.3905 & 3.2086 & \textbf{0.6426} & \textbf{2.8719} & 33.9319 \\
\midrule
7 & MeanFlow Base & 0.5358 & -- & -- & 22.0761 & \textbf{5.6090} & \textbf{4.2500} & 0.5744 & 2.9216 & \textbf{35.4311} \\
7 & MeanFlow + Flow-Map GRPO LoRA & \textbf{0.8683} & -- & -- & \textbf{22.2332} & 5.5216 & 3.6592 & \textbf{0.8449} & \textbf{3.0929} & 35.3517 \\
\midrule
15 & MeanFlow Base & 0.5518 & -- & -- & 22.2302 & \textbf{5.6489} & \textbf{4.3217} & 0.6870 & 3.0208 & \textbf{35.9086} \\
15 & MeanFlow + Flow-Map GRPO LoRA & \textbf{0.8789} & -- & -- & \textbf{22.3262} & 5.5534 & 3.9166 & \textbf{0.8904} & \textbf{3.1842} & 35.8708 \\
\bottomrule
\end{tabular}
}
\end{table}

\begin{table}[H]
\centering\small
\caption{Evaluation of post-trained sCM at different numbers of sampling steps (GenEval reward). Task Metrics measure RL improvements on the target tasks; DrawBench Metrics are auxiliary general-quality metrics used to monitor degradation outside the optimized reward. Evaluation settings are given in Appendix~\ref{appendix:t2i_details}.}
\label{tab:geneval_scm_uniform}
\setlength{\tabcolsep}{1.4pt}
\renewcommand{\arraystretch}{1.08}
\resizebox{\linewidth}{!}{
\begin{tabular}{rlccc@{\hspace{0.5em}}|@{\hspace{0.5em}}cccccc}
\toprule
\multirow{2}{*}{$K$} & \multirow{2}{*}{\textbf{Model}} & \multicolumn{3}{c}{\textbf{Task Metrics}} & \multicolumn{6}{c}{\textbf{DrawBench Metrics}} \\
\cmidrule(lr){3-5}\cmidrule(lr){6-11}
 & & \textbf{GenEval}$\uparrow$ & \shortstack{\textbf{OCR}\\\textbf{Acc.}$\uparrow$} & \shortstack{\textbf{PickScore}\\\textbf{Task}$\uparrow$} & \shortstack{\textbf{DB-}\\\textbf{PickScore}$\uparrow$} & \textbf{Aesthetic}$\uparrow$ & \textbf{DeQA}$\uparrow$ & \textbf{ImgRwd}$\uparrow$ & \textbf{UniRwd}$\uparrow$ & \textbf{DB Sum}$\uparrow$ \\
\midrule
2 & sCM Base & 0.4797 & -- & -- & 21.4149 & \textbf{5.3850} & \textbf{3.4327} & 0.5623 & 2.6225 & 33.4174 \\
2 & sCM + Flow-Map GRPO LoRA & \textbf{0.7721} & -- & -- & \textbf{21.9073} & 5.0662 & 3.2860 & \textbf{0.6124} & \textbf{2.9040} & \textbf{33.7758} \\
\midrule
3 & sCM Base & 0.4560 & -- & -- & 21.4524 & \textbf{5.4448} & 3.6554 & 0.5742 & 2.6727 & 33.7994 \\
3 & sCM + Flow-Map GRPO LoRA & \textbf{0.8005} & -- & -- & \textbf{22.1208} & 5.1971 & \textbf{3.6765} & \textbf{0.6813} & \textbf{3.0374} & \textbf{34.7131} \\
\midrule
4 & sCM Base & 0.4513 & -- & -- & 21.4993 & \textbf{5.4747} & 3.7469 & 0.5750 & 2.6725 & 33.9684 \\
4 & sCM + Flow-Map GRPO LoRA & \textbf{0.8049} & -- & -- & \textbf{22.1046} & 5.2401 & \textbf{3.8938} & \textbf{0.6514} & \textbf{3.0835} & \textbf{34.9734} \\
\midrule
7 & sCM Base & 0.4299 & -- & -- & 21.5128 & \textbf{5.5084} & 3.8796 & 0.6037 & 2.7196 & 34.2241 \\
7 & sCM + Flow-Map GRPO LoRA & \textbf{0.7968} & -- & -- & \textbf{22.1326} & 5.3483 & \textbf{4.1705} & \textbf{0.6881} & \textbf{3.1347} & \textbf{35.4741} \\
\midrule
15 & sCM Base & 0.3983 & -- & -- & 21.3638 & \textbf{5.5358} & 3.8862 & 0.5312 & 2.6591 & 33.9761 \\
15 & sCM + Flow-Map GRPO LoRA & \textbf{0.7585} & -- & -- & \textbf{21.8789} & 5.3724 & \textbf{4.2218} & \textbf{0.5578} & \textbf{3.0663} & \textbf{35.0972} \\
\bottomrule
\end{tabular}
}
\end{table}

\subsection{Joint OCR--PickScore Evaluations}\label{app:mixed_sweeps}
\begin{table}[H]
\centering\small
\caption{Evaluation of post-trained MeanFlow at different numbers of sampling steps (joint OCR--PickScore rewards). Task Metrics measure RL improvements on the target tasks; DrawBench Metrics are auxiliary general-quality metrics used to monitor degradation outside the optimized reward. Evaluation settings are given in Appendix~\ref{appendix:t2i_details}.}
\label{tab:mixed_meanflow}
\setlength{\tabcolsep}{1.4pt}
\renewcommand{\arraystretch}{1.08}
\resizebox{\linewidth}{!}{
\begin{tabular}{rlccc@{\hspace{0.5em}}|@{\hspace{0.5em}}cccccc}
\toprule
\multirow{2}{*}{$K$} & \multirow{2}{*}{\textbf{Model}} & \multicolumn{3}{c}{\textbf{Task Metrics}} & \multicolumn{6}{c}{\textbf{DrawBench Metrics}} \\
\cmidrule(lr){3-5}\cmidrule(lr){6-11}
 & & \textbf{GenEval}$\uparrow$ & \shortstack{\textbf{OCR}\\\textbf{Acc.}$\uparrow$} & \shortstack{\textbf{PickScore}\\\textbf{Task}$\uparrow$} & \shortstack{\textbf{DB-}\\\textbf{PickScore}$\uparrow$} & \textbf{Aesthetic}$\uparrow$ & \textbf{DeQA}$\uparrow$ & \textbf{ImgRwd}$\uparrow$ & \textbf{UniRwd}$\uparrow$ & \textbf{DB Sum}$\uparrow$ \\
\midrule
3 & MeanFlow Base & -- & 0.2876 & 20.6926 & 21.3577 & 5.4391 & \textbf{4.0001} & 0.2157 & 2.5495 & 33.5621 \\
3 & MeanFlow + Flow-Map GRPO LoRA & -- & \textbf{0.6275} & \textbf{22.1122} & \textbf{22.5958} & \textbf{5.8244} & 3.8551 & \textbf{0.8696} & \textbf{3.0020} & \textbf{36.1468} \\
\midrule
4 & MeanFlow Base & -- & 0.3491 & 20.9758 & 21.6298 & 5.5368 & 4.1712 & 0.3918 & 2.7187 & 34.4483 \\
4 & MeanFlow + Flow-Map GRPO LoRA & -- & \textbf{0.7719} & \textbf{22.4964} & \textbf{23.0506} & \textbf{5.9299} & \textbf{4.1892} & \textbf{1.0439} & \textbf{3.1876} & \textbf{37.4012} \\
\midrule
7 & MeanFlow Base & -- & 0.4517 & 21.3595 & 22.0761 & 5.6090 & 4.2500 & 0.5744 & 2.9216 & 35.4311 \\
7 & MeanFlow + Flow-Map GRPO LoRA & -- & \textbf{0.8701} & \textbf{22.7455} & \textbf{23.4005} & \textbf{6.0176} & \textbf{4.3873} & \textbf{1.1342} & \textbf{3.3354} & \textbf{38.2750} \\
\midrule
15 & MeanFlow Base & -- & 0.5507 & 21.4382 & 22.2302 & 5.6489 & 4.3217 & 0.6870 & 3.0208 & 35.9086 \\
15 & MeanFlow + Flow-Map GRPO LoRA & -- & \textbf{0.8845} & \textbf{22.7747} & \textbf{23.4243} & \textbf{6.0410} & \textbf{4.4403} & \textbf{1.1352} & \textbf{3.3692} & \textbf{38.4101} \\
\bottomrule
\end{tabular}
}
\end{table}

\begin{table}[H]
\centering\small
\caption{Evaluation of post-trained sCM at different numbers of sampling steps (joint OCR--PickScore rewards). Task Metrics measure RL improvements on the target tasks; DrawBench Metrics are auxiliary general-quality metrics used to monitor degradation outside the optimized reward. Evaluation settings are given in Appendix~\ref{appendix:t2i_details}.}
\label{tab:mixed_scm}
\setlength{\tabcolsep}{1.4pt}
\renewcommand{\arraystretch}{1.08}
\resizebox{\linewidth}{!}{
\begin{tabular}{rlccc@{\hspace{0.5em}}|@{\hspace{0.5em}}cccccc}
\toprule
\multirow{2}{*}{$K$} & \multirow{2}{*}{\textbf{Model}} & \multicolumn{3}{c}{\textbf{Task Metrics}} & \multicolumn{6}{c}{\textbf{DrawBench Metrics}} \\
\cmidrule(lr){3-5}\cmidrule(lr){6-11}
 & & \textbf{GenEval}$\uparrow$ & \shortstack{\textbf{OCR}\\\textbf{Acc.}$\uparrow$} & \shortstack{\textbf{PickScore}\\\textbf{Task}$\uparrow$} & \shortstack{\textbf{DB-}\\\textbf{PickScore}$\uparrow$} & \textbf{Aesthetic}$\uparrow$ & \textbf{DeQA}$\uparrow$ & \textbf{ImgRwd}$\uparrow$ & \textbf{UniRwd}$\uparrow$ & \textbf{DB Sum}$\uparrow$ \\
\midrule
2 & sCM Base & -- & 0.3394 & 20.3023 & 21.4149 & 5.3850 & 3.4327 & 0.5623 & 2.6225 & 33.4174 \\
2 & sCM + Flow-Map GRPO LoRA & -- & \textbf{0.7744} & \textbf{22.3077} & \textbf{22.8201} & \textbf{5.8086} & \textbf{3.7308} & \textbf{0.9311} & \textbf{2.9627} & \textbf{36.2533} \\
\midrule
3 & sCM Base & -- & 0.4290 & 20.2942 & 21.4524 & 5.4448 & 3.6554 & 0.5742 & 2.6727 & 33.7994 \\
3 & sCM + Flow-Map GRPO LoRA & -- & \textbf{0.8380} & \textbf{22.4789} & \textbf{23.0224} & \textbf{5.9183} & \textbf{4.1679} & \textbf{0.9719} & \textbf{3.0648} & \textbf{37.1452} \\
\midrule
4 & sCM Base & -- & 0.4903 & 20.3517 & 21.4993 & 5.4747 & 3.7469 & 0.5750 & 2.6725 & 33.9684 \\
4 & sCM + Flow-Map GRPO LoRA & -- & \textbf{0.8568} & \textbf{22.5153} & \textbf{23.0403} & \textbf{5.9494} & \textbf{4.3311} & \textbf{0.9741} & \textbf{3.1004} & \textbf{37.3952} \\
\midrule
7 & sCM Base & -- & 0.6120 & 20.2940 & 21.5128 & 5.5084 & 3.8796 & 0.6037 & 2.7196 & 34.2241 \\
7 & sCM + Flow-Map GRPO LoRA & -- & \textbf{0.8592} & \textbf{22.3719} & \textbf{22.9681} & \textbf{5.9591} & \textbf{4.4511} & \textbf{0.9381} & \textbf{3.0669} & \textbf{37.3833} \\
\midrule
15 & sCM Base & -- & 0.6583 & 20.0556 & 21.3638 & 5.5358 & 3.8862 & 0.5312 & 2.6591 & 33.9761 \\
15 & sCM + Flow-Map GRPO LoRA & -- & \textbf{0.8713} & \textbf{21.9501} & \textbf{22.6570} & \textbf{5.9168} & \textbf{4.4899} & \textbf{0.8496} & \textbf{2.9970} & \textbf{36.9102} \\
\bottomrule
\end{tabular}
}
\end{table}

\section{Additional Qualitative Results}\label{sec:additional}
These examples use the adapters and numbers of inference steps in the quantitative tables. Each prompt and seed is matched across Base (top) and Ours (bottom) and across the step sweep. Columns show $K$, corresponding to $K+1$ sampler steps.

\subsection{MeanFlow: OCR}\label{app:qual_meanflow_ocr}
\begin{figure}[H]
\centering
{\footnotesize\raggedright\textbf{Prompt 1:} A premium editorial photograph of a frosted peach soda can splashing into turquoise water, sliced peaches suspended nearby. The can label clearly reads ``PEACH CLUB''. Crisp droplets, direct summer sunlight, clean commercial composition.\par}
\smallskip
\includegraphics[width=0.68\linewidth]{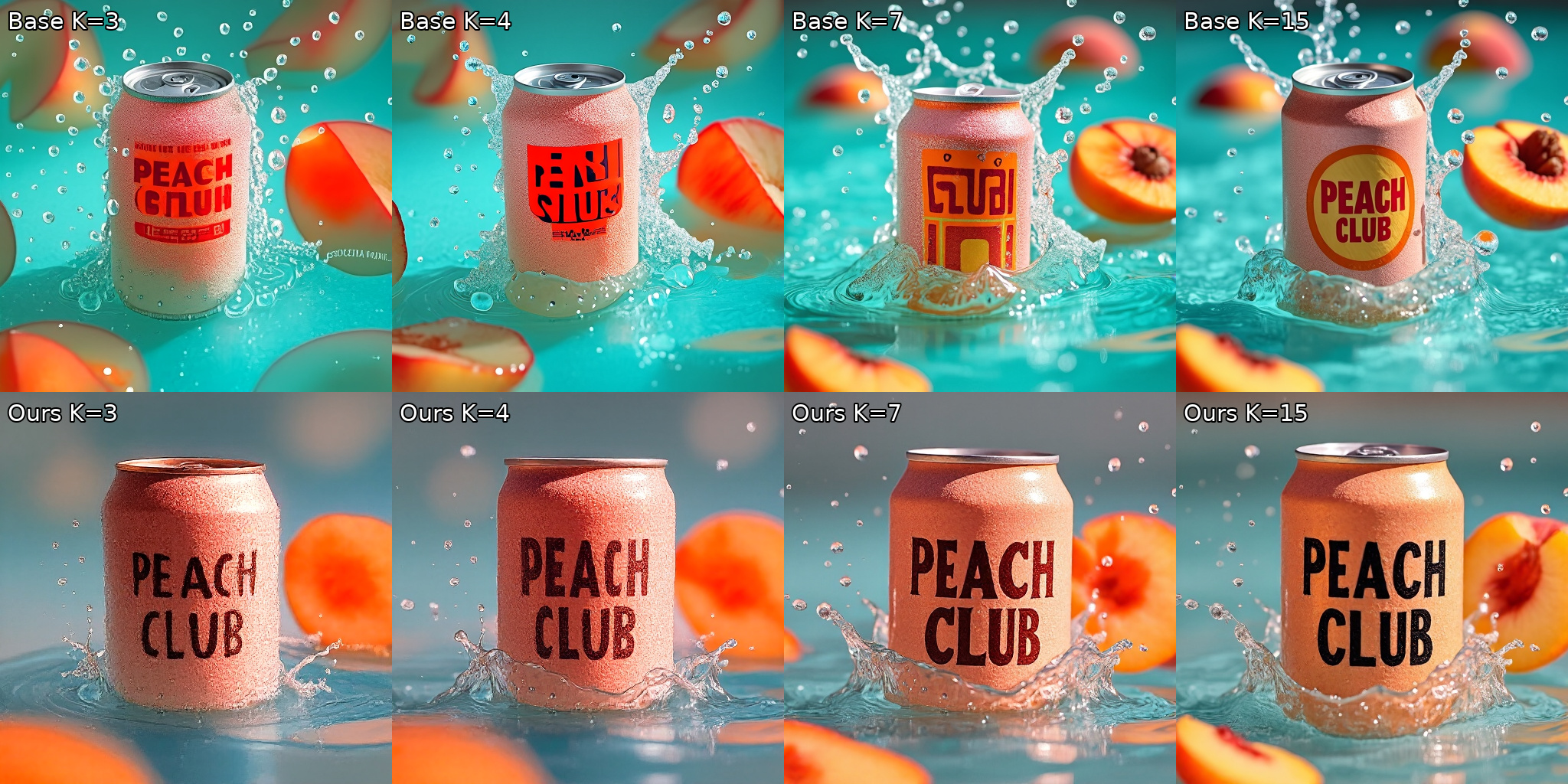}
\par\vspace{0.45em}
{\footnotesize\raggedright\textbf{Prompt 2:} A sophisticated illustrated book cover with golden koi fish swimming through dark indigo waves and floating white petals. Large elegant ivory title ``THE SILENT SEA''. Intricate woodblock illustration, refined spacing.\par}
\smallskip
\includegraphics[width=0.68\linewidth]{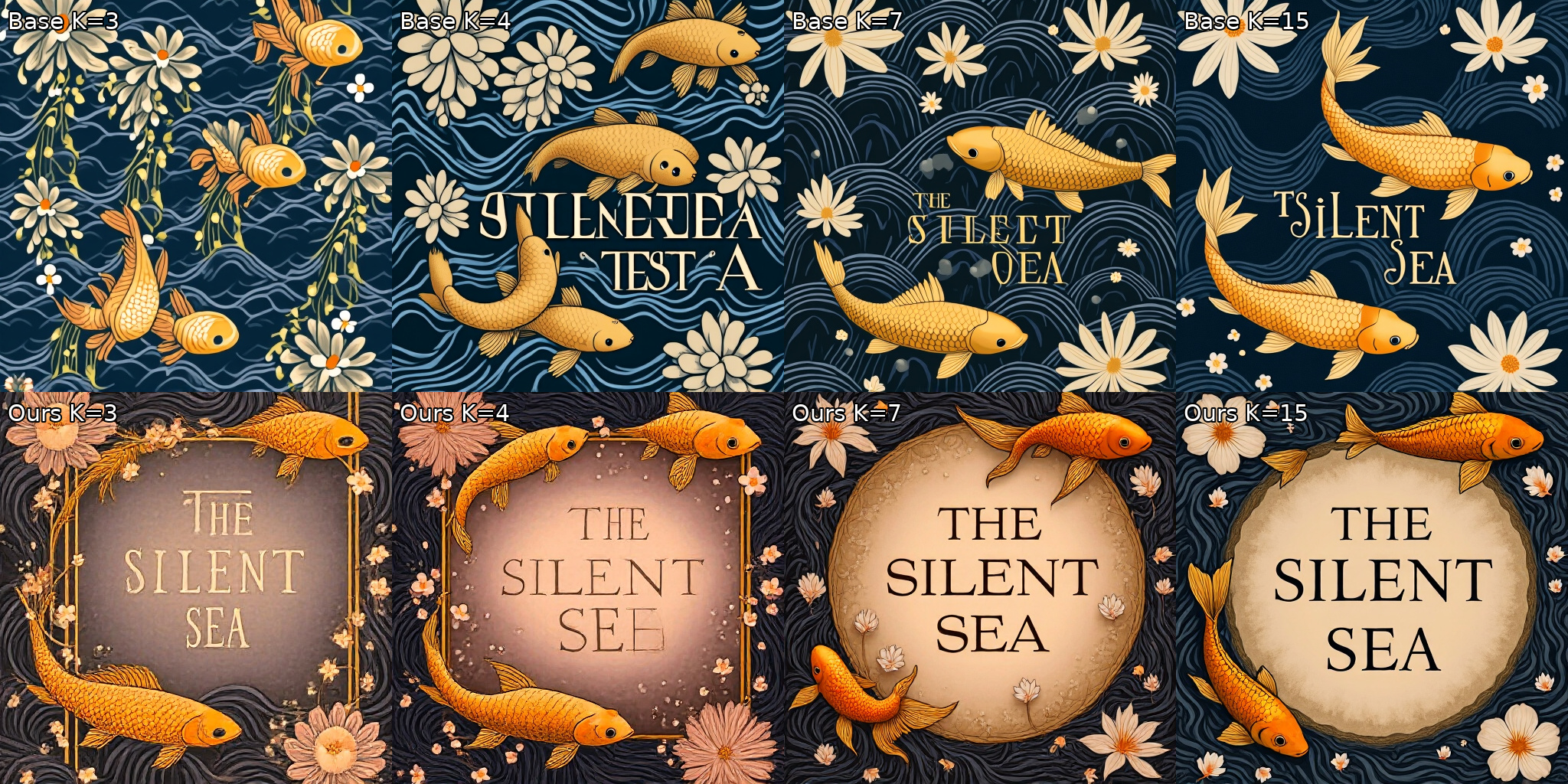}
\par\vspace{0.45em}
{\footnotesize\raggedright\textbf{Prompt 3:} A sophisticated luxury product photograph of a transparent amber glass bottle with label ``GOLDEN HOUR'', placed on a moss-covered rock beside a mountain stream. Sunlight through the glass, emerald water, delicate vapor.\par}
\smallskip
\includegraphics[width=0.68\linewidth]{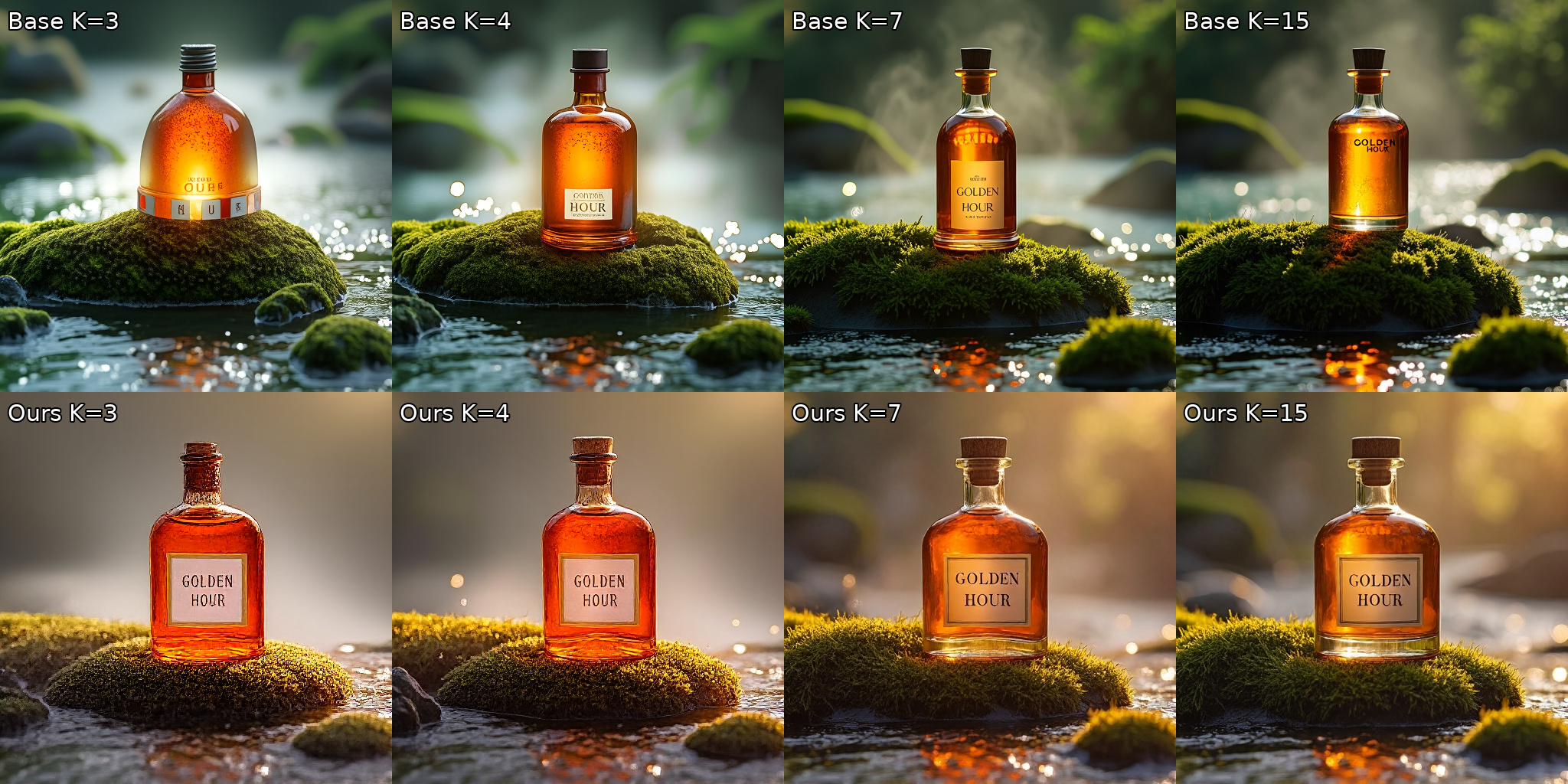}
\caption{MeanFlow before and after OCR post-training. Three prompts, each shown as a paired Base/Ours sweep at different numbers of sampling steps.}
\label{fig:qual_meanflow_ocr}
\end{figure}
\newpage
\subsection{sCM: OCR}\label{app:qual_scm_ocr}
\begin{figure}[H]
\centering
{\footnotesize\raggedright\textbf{Prompt 1:} A cinematic photograph of a cherry-red racing car beside a desert gas station at blue hour. A tall illuminated sign clearly reads ``LAST STOP''. Pink clouds, cyan shadows, chrome highlights, expansive composition.\par}
\smallskip
\includegraphics[width=0.97\linewidth]{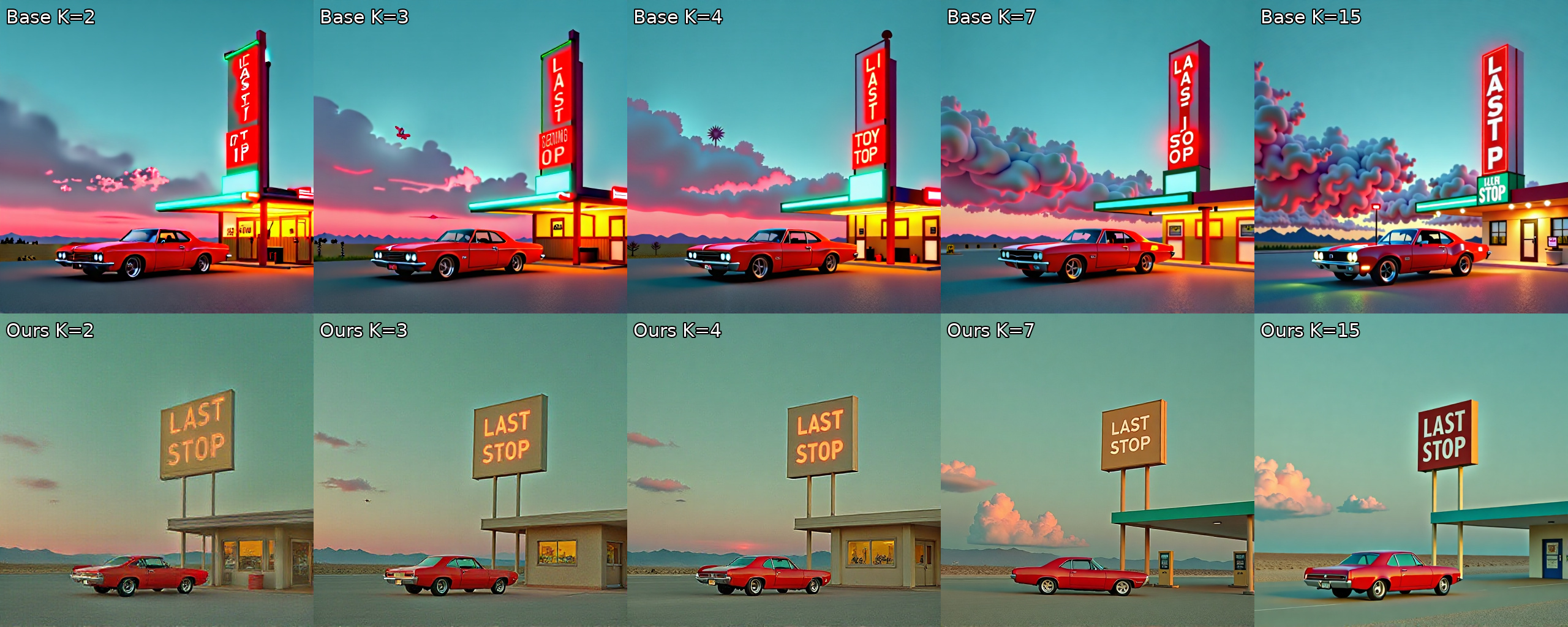}
\par\vspace{0.45em}
{\footnotesize\raggedright\textbf{Prompt 2:} A beautiful interior photograph of a coral-pink greenhouse filled with white orchids and ferns. A prominent cream wall sign reads ``GROW WITH LOVE''. Warm sunlight, patterned shadows, polished green tile floor.\par}
\smallskip
\includegraphics[width=0.97\linewidth]{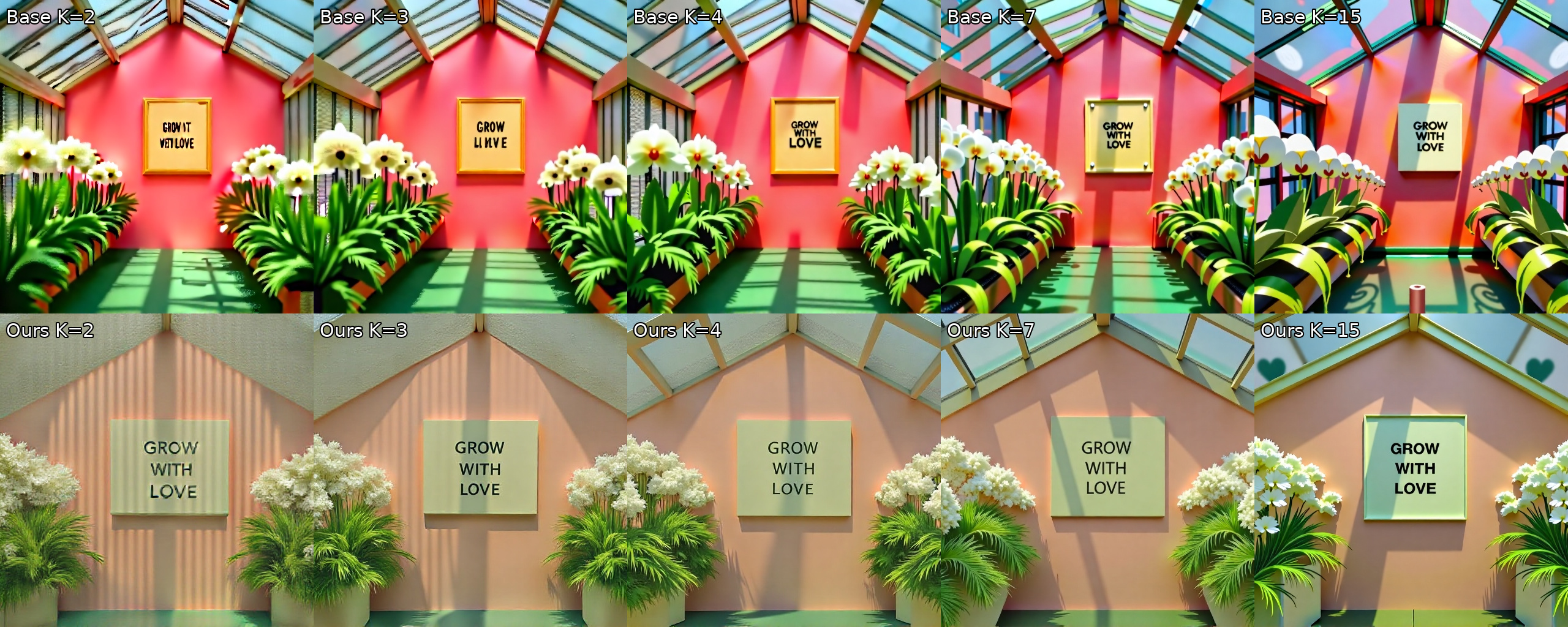}
\par\vspace{0.45em}
{\footnotesize\raggedright\textbf{Prompt 3:} A close-up commercial photograph of a jade-green matcha bottle standing among bright green leaves on pale pink stone. Large white label reads ``MATCHA MORNING''. Dewdrops, soft luminous backlighting, clean sophisticated composition.\par}
\smallskip
\includegraphics[width=0.97\linewidth]{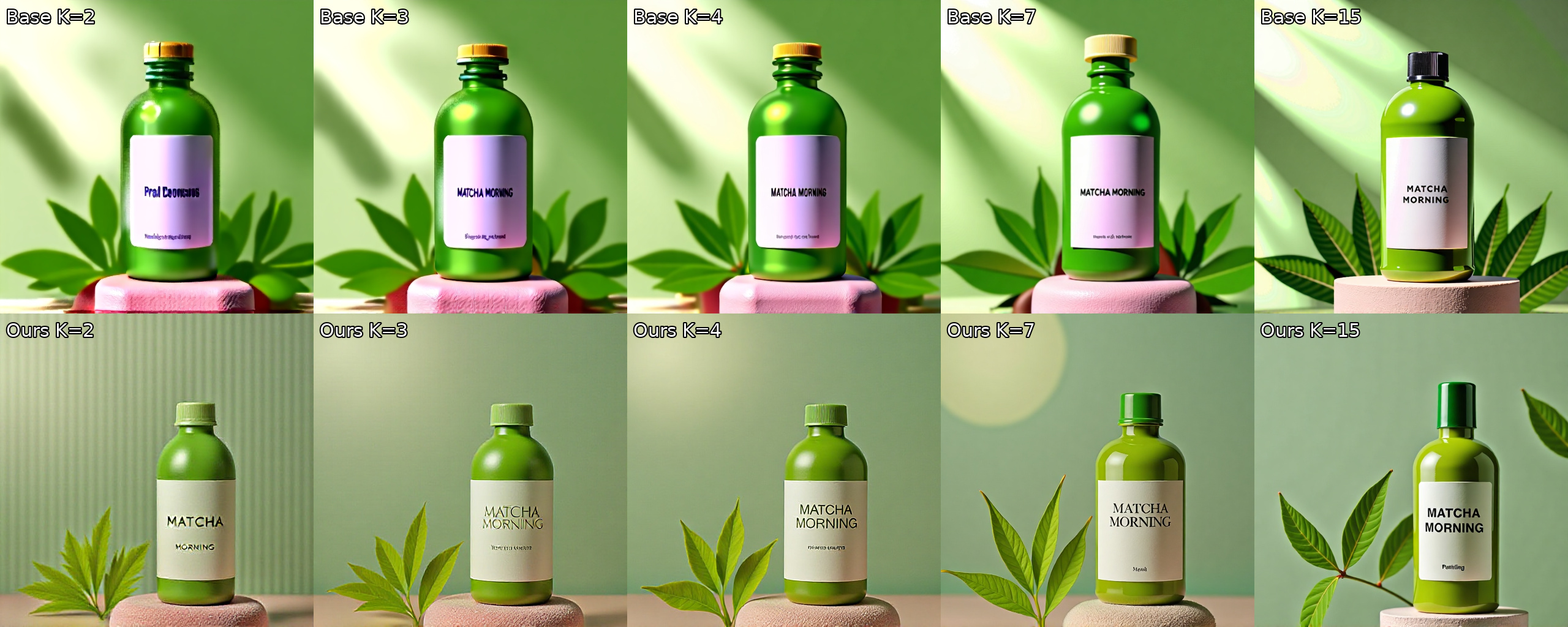}
\caption{sCM before and after OCR post-training. Three prompts, each shown as a paired Base/Ours sweep at different numbers of sampling steps.}
\label{fig:qual_scm_ocr}
\end{figure}
\newpage
\subsection{MeanFlow: PickScore}\label{app:qual_meanflow_pickscore}
\begin{figure}[H]
\centering
{\footnotesize\raggedright\textbf{Prompt 1:} A cinematic macro photograph of a tiny red frog sitting inside a translucent white orchid, rain droplets, luminous tropical foliage, very shallow depth of field.\par}
\smallskip
\includegraphics[width=0.78\linewidth]{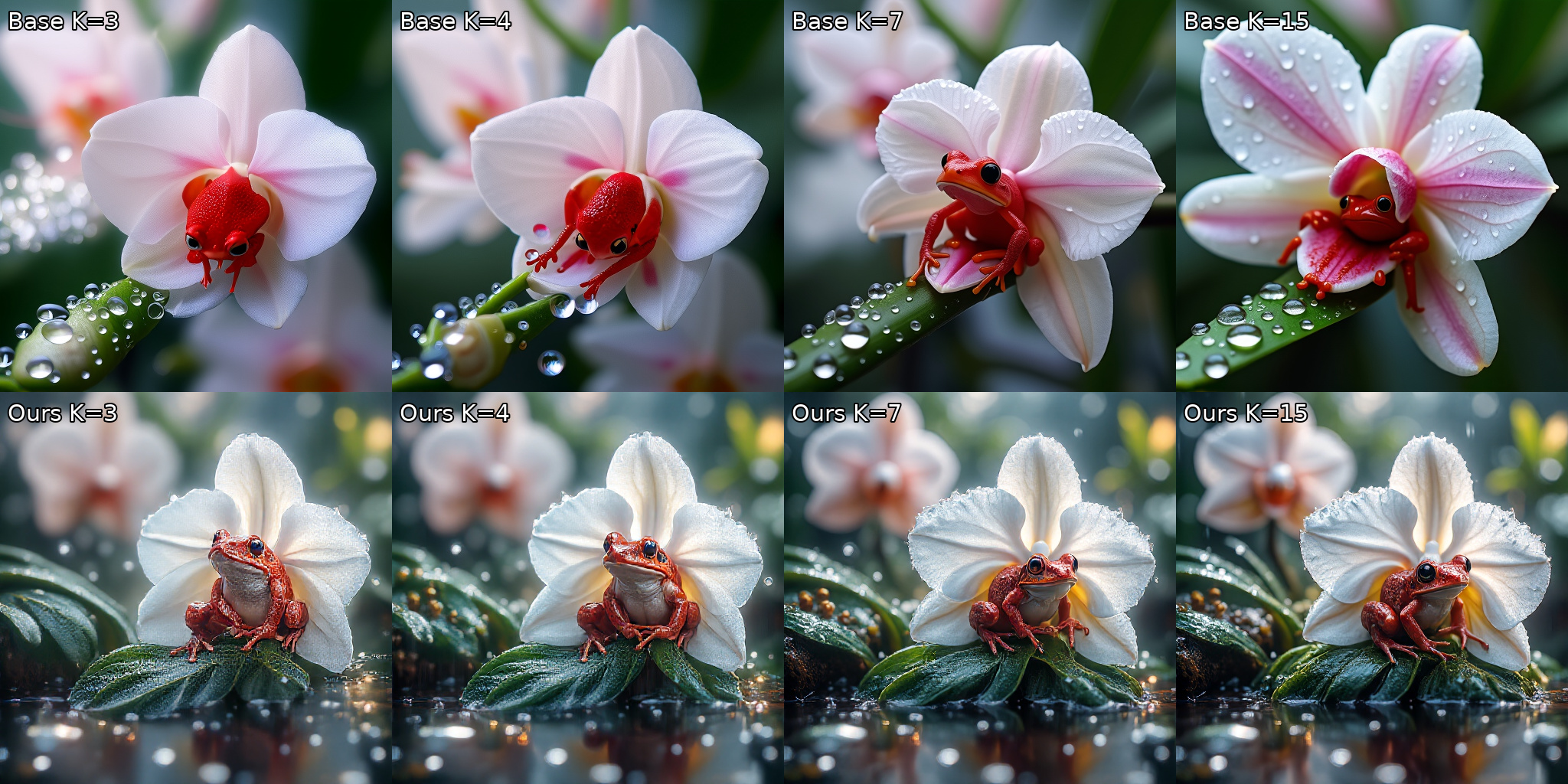}
\par\vspace{0.45em}
{\footnotesize\raggedright\textbf{Prompt 2:} A whimsical miniature village carved inside a giant orange pumpkin, tiny glowing windows and wooden bridges, autumn leaves, richly detailed stop-motion film aesthetic.\par}
\smallskip
\includegraphics[width=0.78\linewidth]{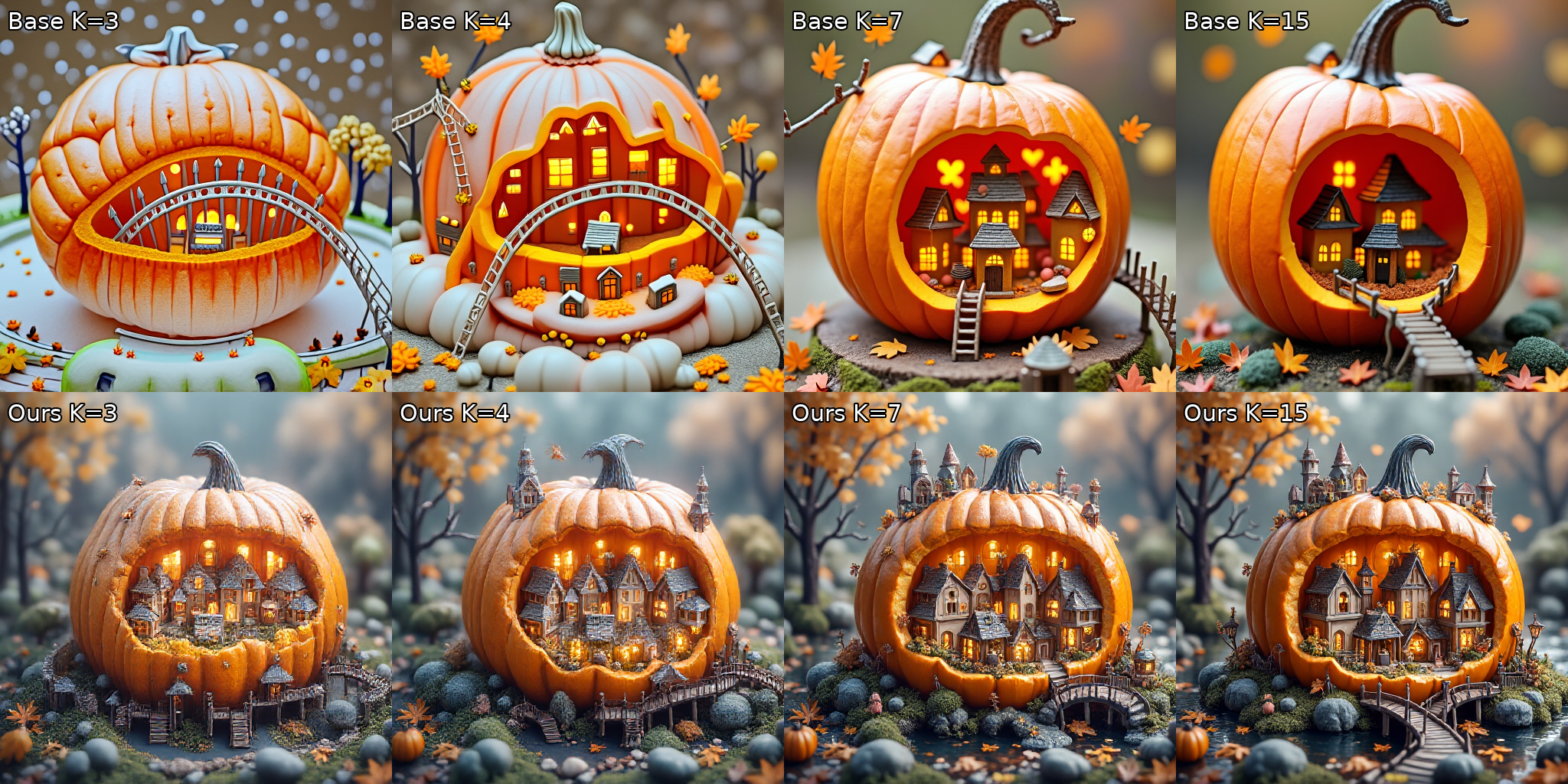}
\par\vspace{0.45em}
{\footnotesize\raggedright\textbf{Prompt 3:} A transparent crystal whale floating above a mirror-flat ocean at twilight, stars refracting through its body, a tiny boat below, breathtaking cinematic fantasy realism.\par}
\smallskip
\includegraphics[width=0.78\linewidth]{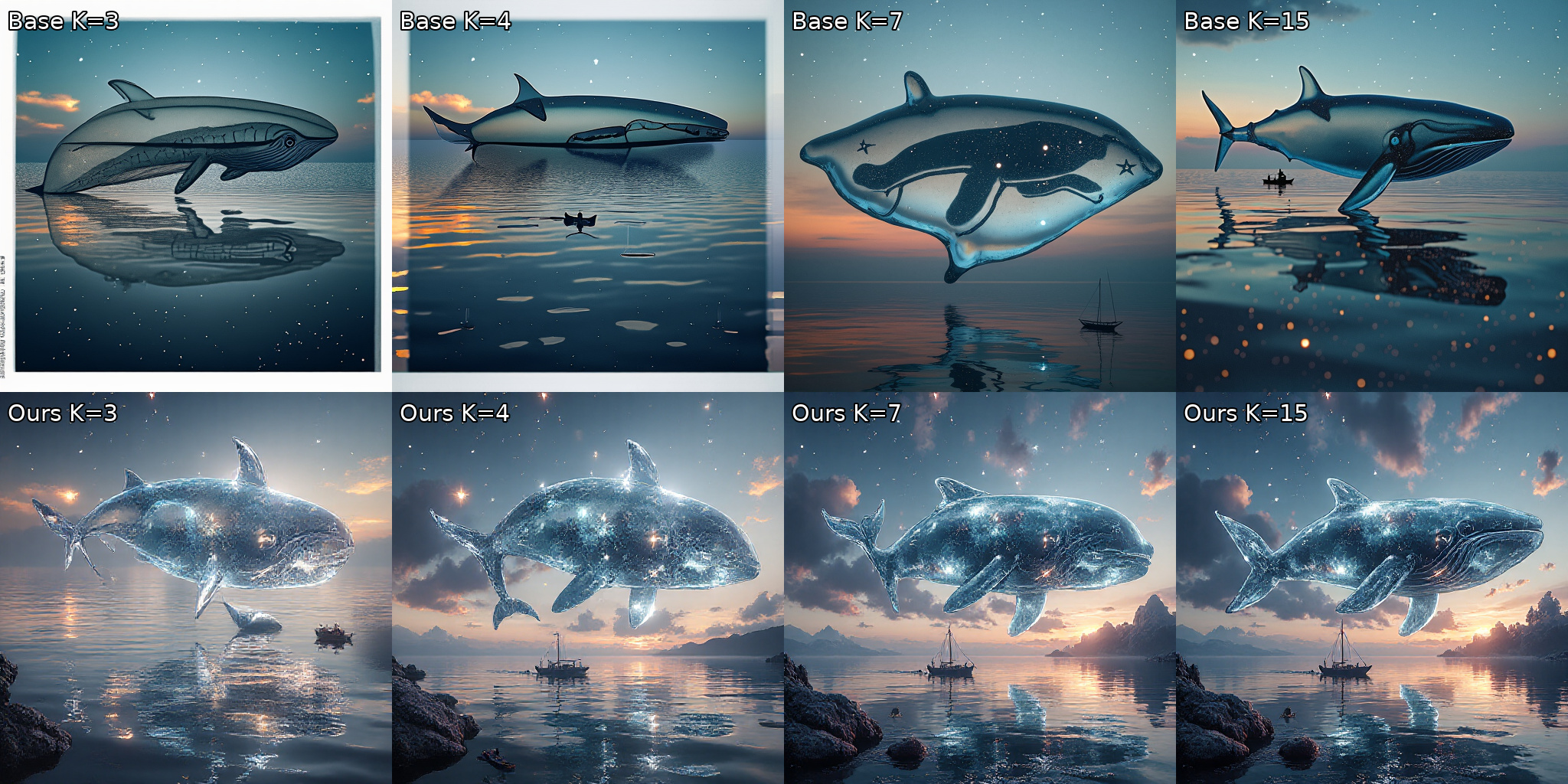}
\caption{MeanFlow before and after PickScore post-training. Three prompts, each shown as a paired Base/Ours sweep at different numbers of sampling steps.}
\label{fig:qual_meanflow_pickscore}
\end{figure}
\newpage
\subsection{sCM: PickScore}\label{app:qual_scm_pickscore}
\begin{figure}[H]
\centering
{\footnotesize\raggedright\textbf{Prompt 1:} A dramatic editorial photograph of a scarlet sailing yacht surrounded by enormous turquoise icebergs, dark storm clouds opening to a shaft of golden sunlight, fine windblown spray, sweeping aerial composition.\par}
\smallskip
\includegraphics[width=0.97\linewidth]{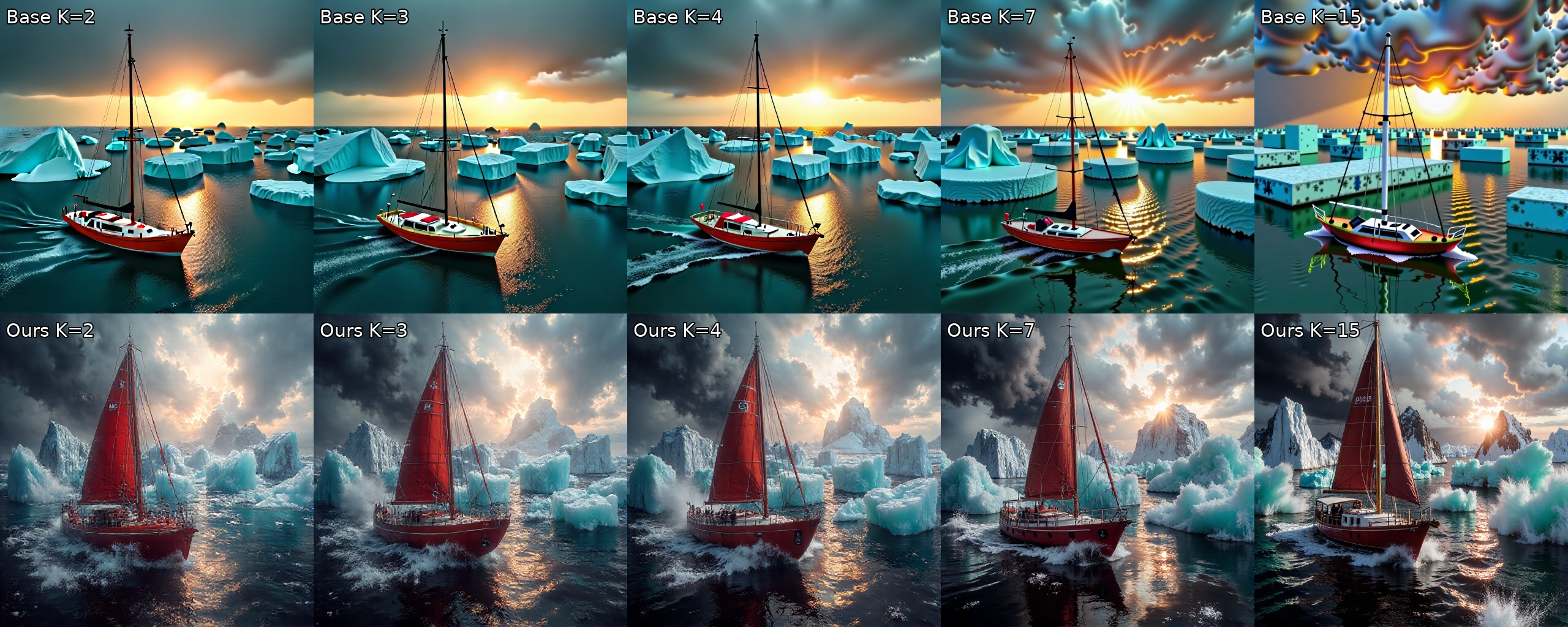}
\par\vspace{0.45em}
{\footnotesize\raggedright\textbf{Prompt 2:} A beautifully composed food photograph of a raspberry pavlova with cream and rose petals on a mint-green ceramic plate, soft side light, fine sugar dust, luxurious editorial styling.\par}
\smallskip
\includegraphics[width=0.97\linewidth]{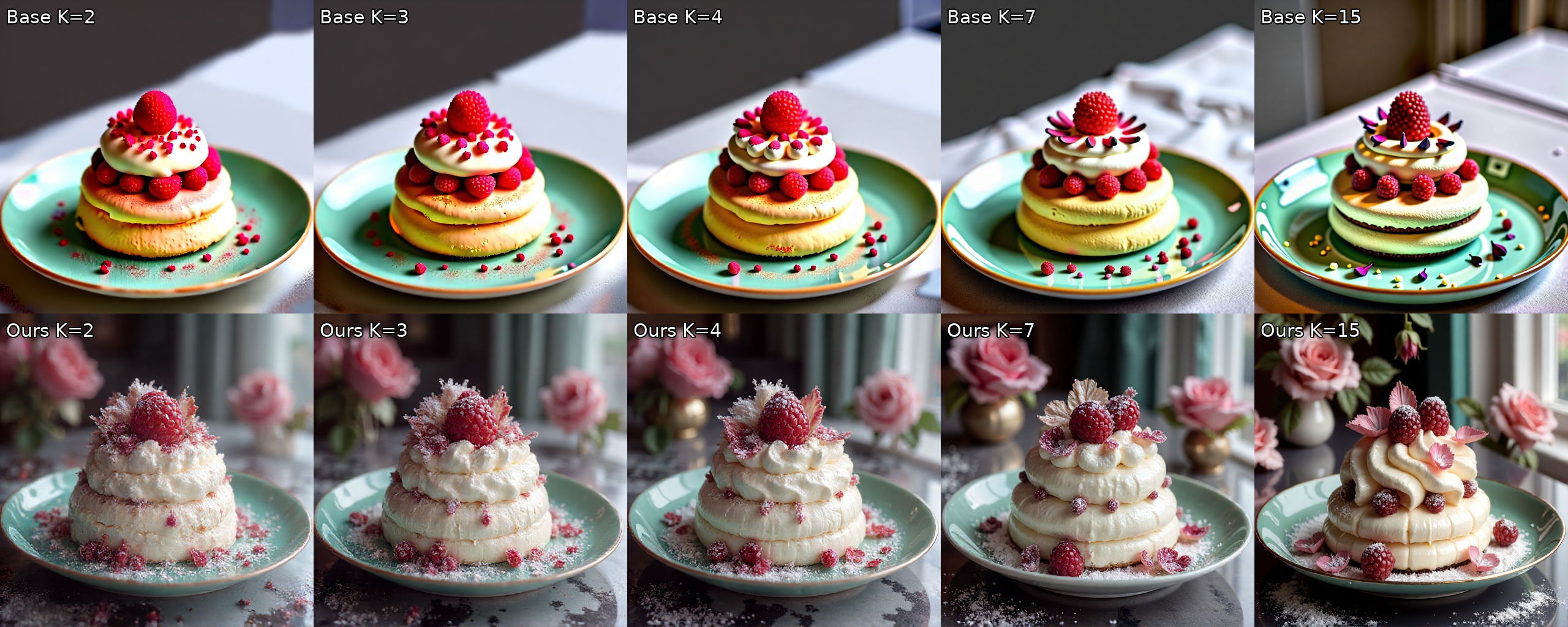}
\par\vspace{0.45em}
{\footnotesize\raggedright\textbf{Prompt 3:} A dramatic photograph of a massive curling ocean wave seen from inside the barrel, translucent aquamarine water and a golden setting sun, crisp foam and mist.\par}
\smallskip
\includegraphics[width=0.97\linewidth]{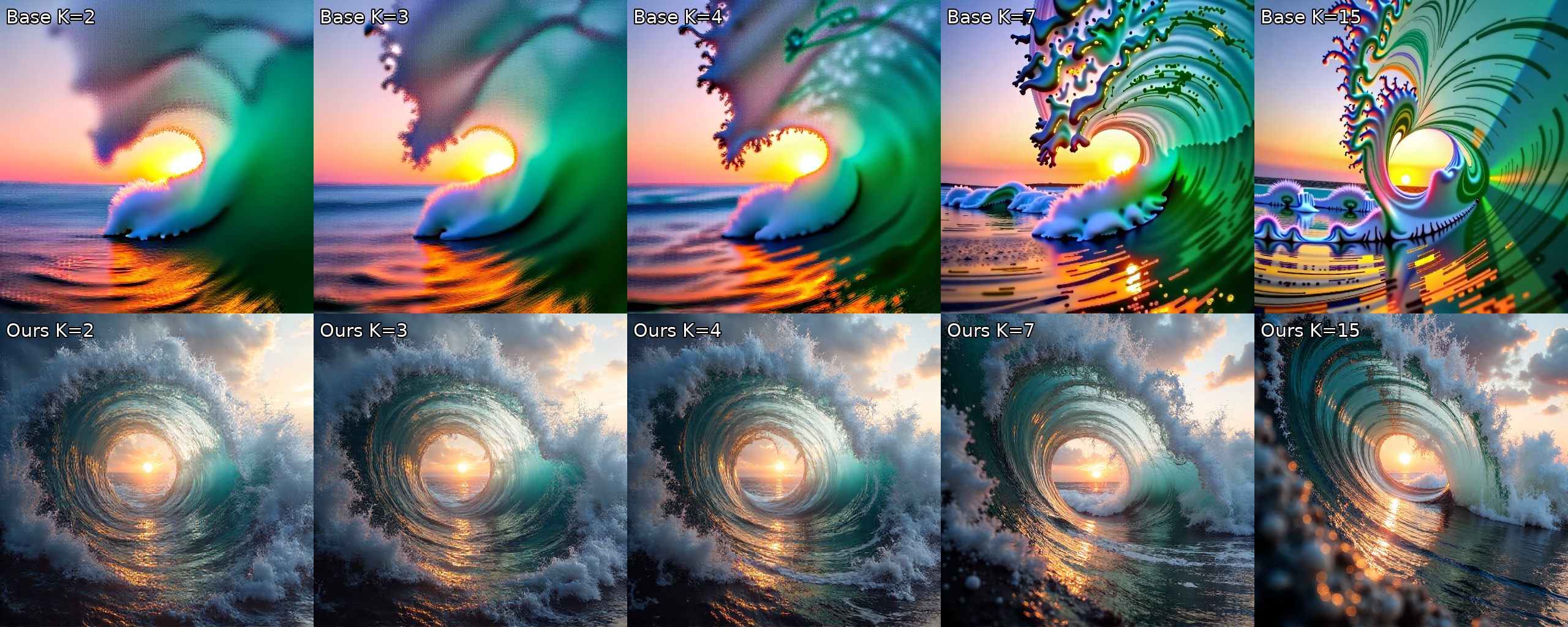}
\caption{sCM before and after PickScore post-training. Three prompts, each shown as a paired Base/Ours sweep at different numbers of sampling steps.}
\label{fig:qual_scm_pickscore}
\end{figure}
\newpage
\subsection{MeanFlow: GenEval}\label{app:qual_meanflow_geneval}
\begin{figure}[H]
\centering
\begin{minipage}[t]{0.49\linewidth}
{\scriptsize\raggedright\textbf{1.} a photo of a backpack left of a chair\par}
\smallskip
\includegraphics[width=\linewidth]{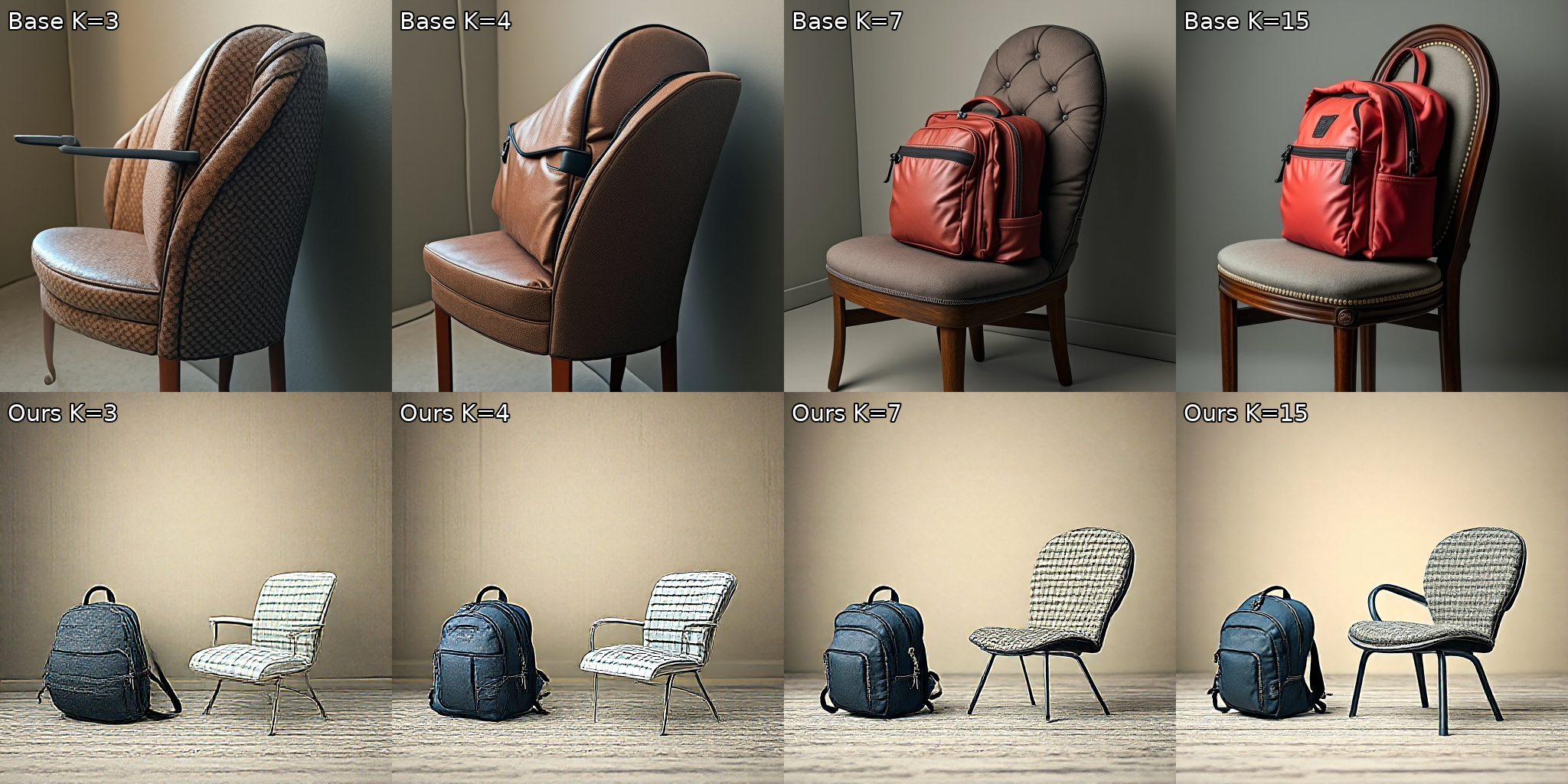}
\end{minipage}
\hfill
\begin{minipage}[t]{0.49\linewidth}
{\scriptsize\raggedright\textbf{2.} a photo of a teddy bear right of a bowl\par}
\smallskip
\includegraphics[width=\linewidth]{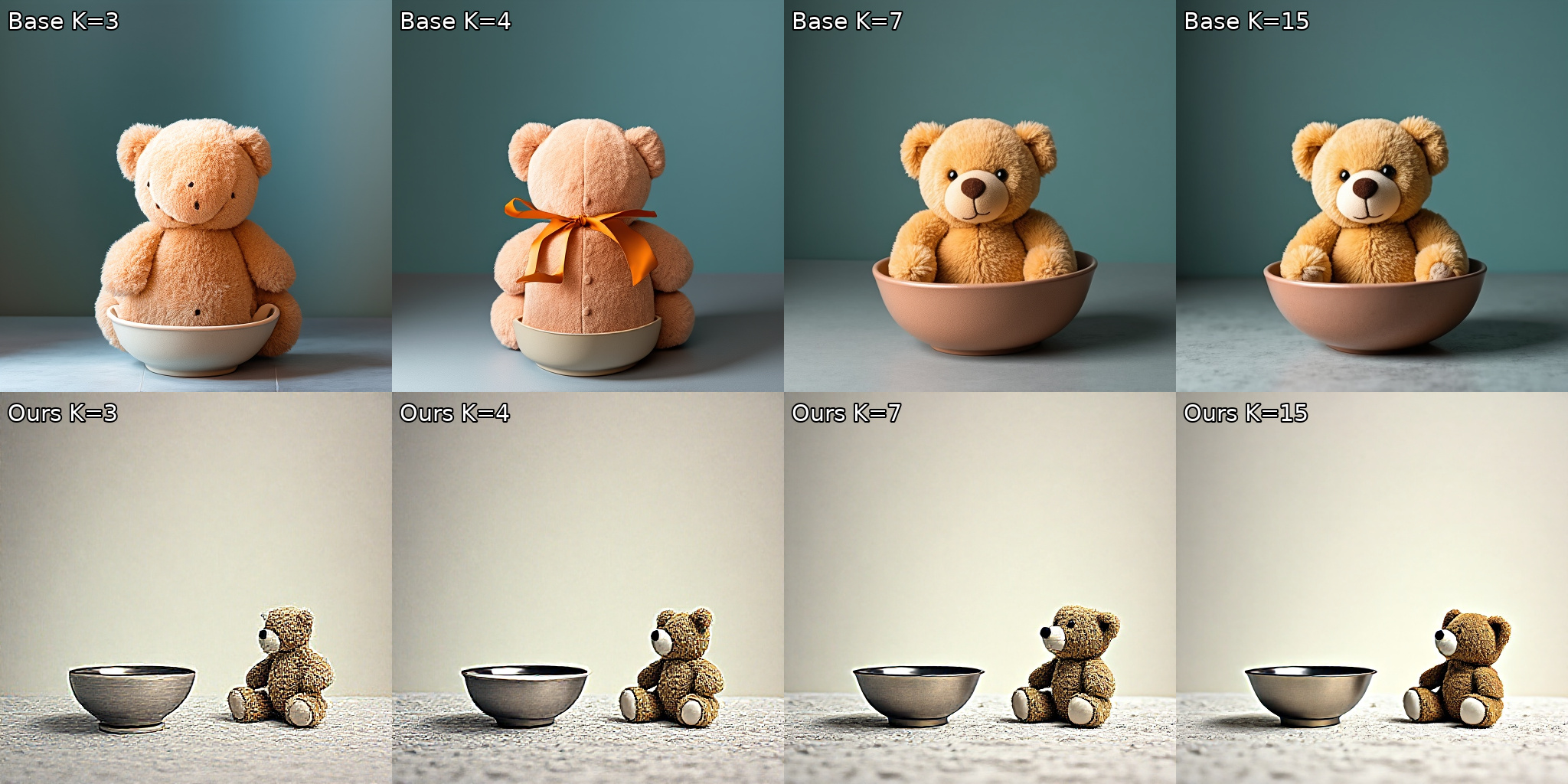}
\end{minipage}
\par\vspace{0.7em}
\begin{minipage}[t]{0.49\linewidth}
{\scriptsize\raggedright\textbf{3.} a photo of a bottle left of a clock\par}
\smallskip
\includegraphics[width=\linewidth]{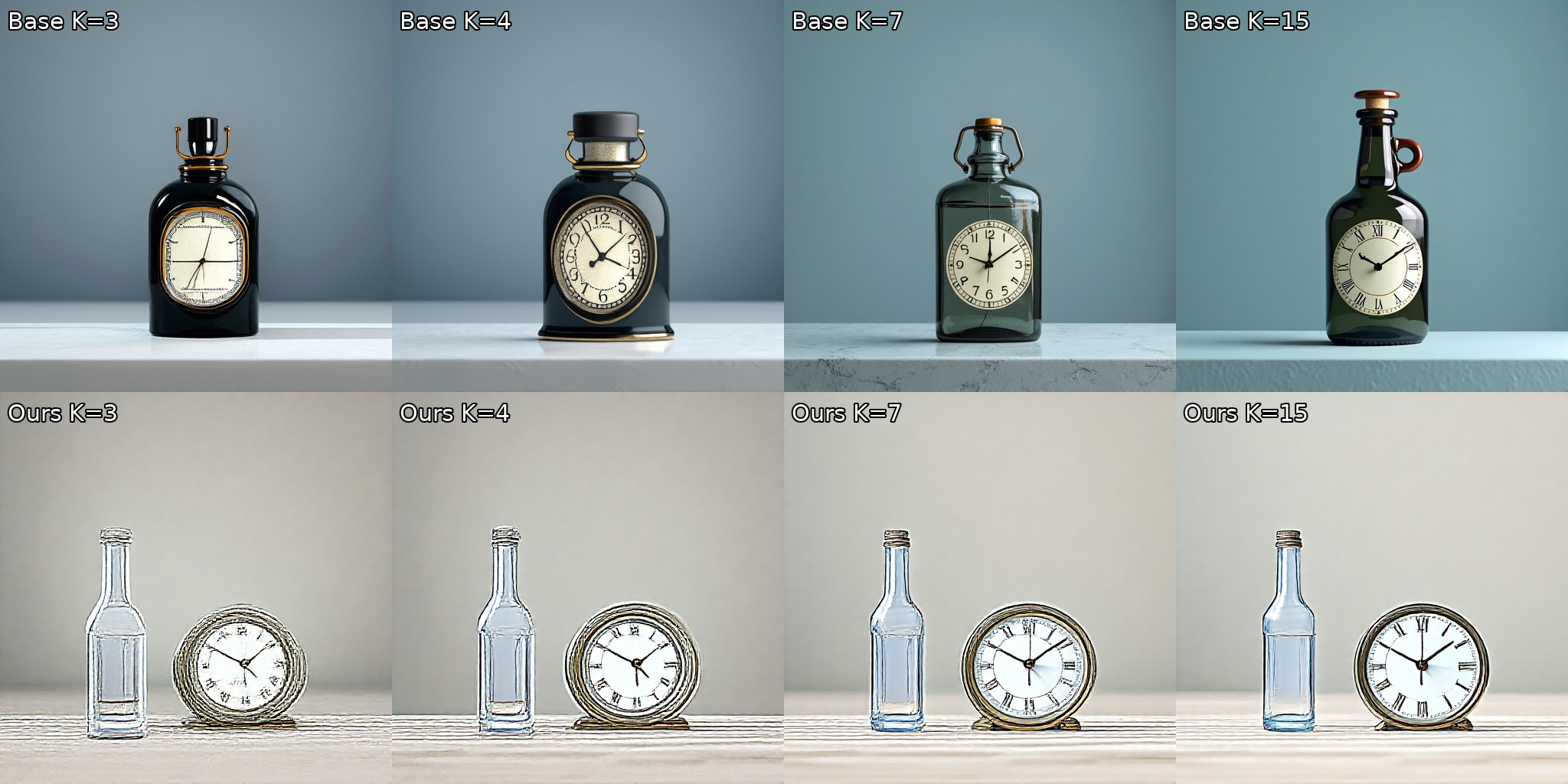}
\end{minipage}
\hfill
\begin{minipage}[t]{0.49\linewidth}
{\scriptsize\raggedright\textbf{4.} a photo of three teddy bears\par}
\smallskip
\includegraphics[width=\linewidth]{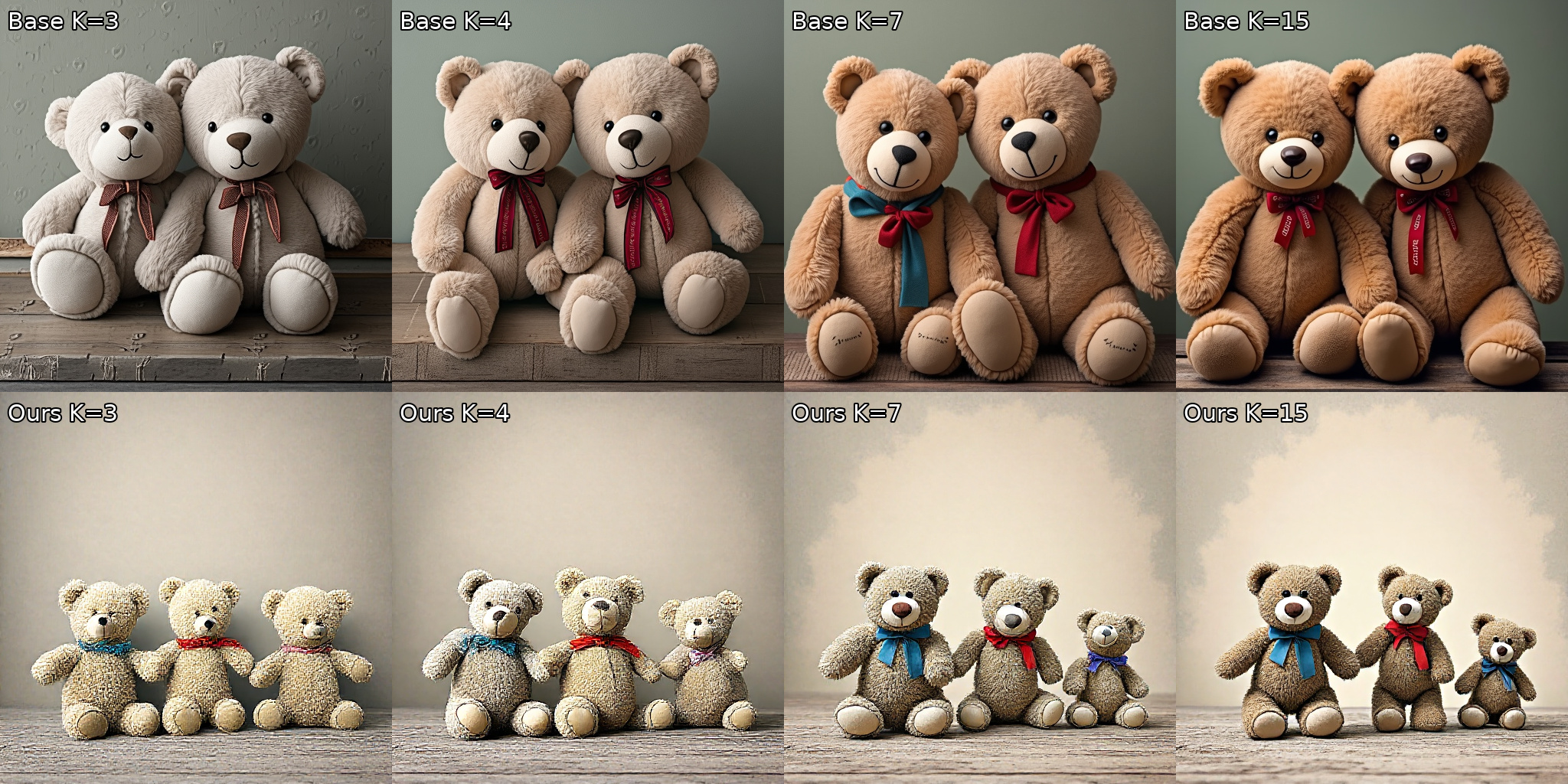}
\end{minipage}
\par\vspace{0.7em}
\begin{minipage}[t]{0.49\linewidth}
{\scriptsize\raggedright\textbf{5.} a photo of a red chair and a blue teddy bear\par}
\smallskip
\includegraphics[width=\linewidth]{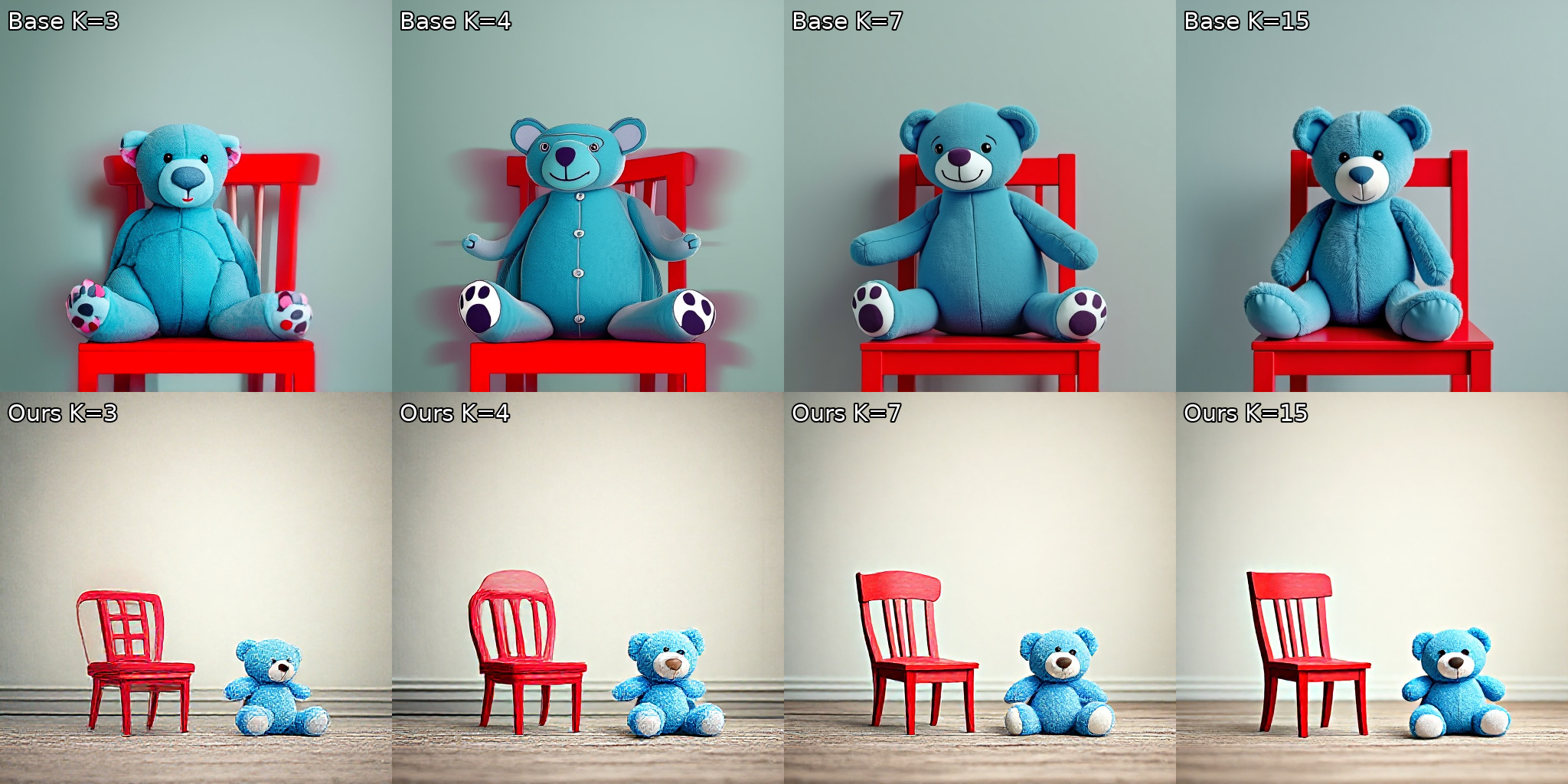}
\end{minipage}
\hfill
\begin{minipage}[t]{0.49\linewidth}
{\scriptsize\raggedright\textbf{6.} a photo of three vases\par}
\smallskip
\includegraphics[width=\linewidth]{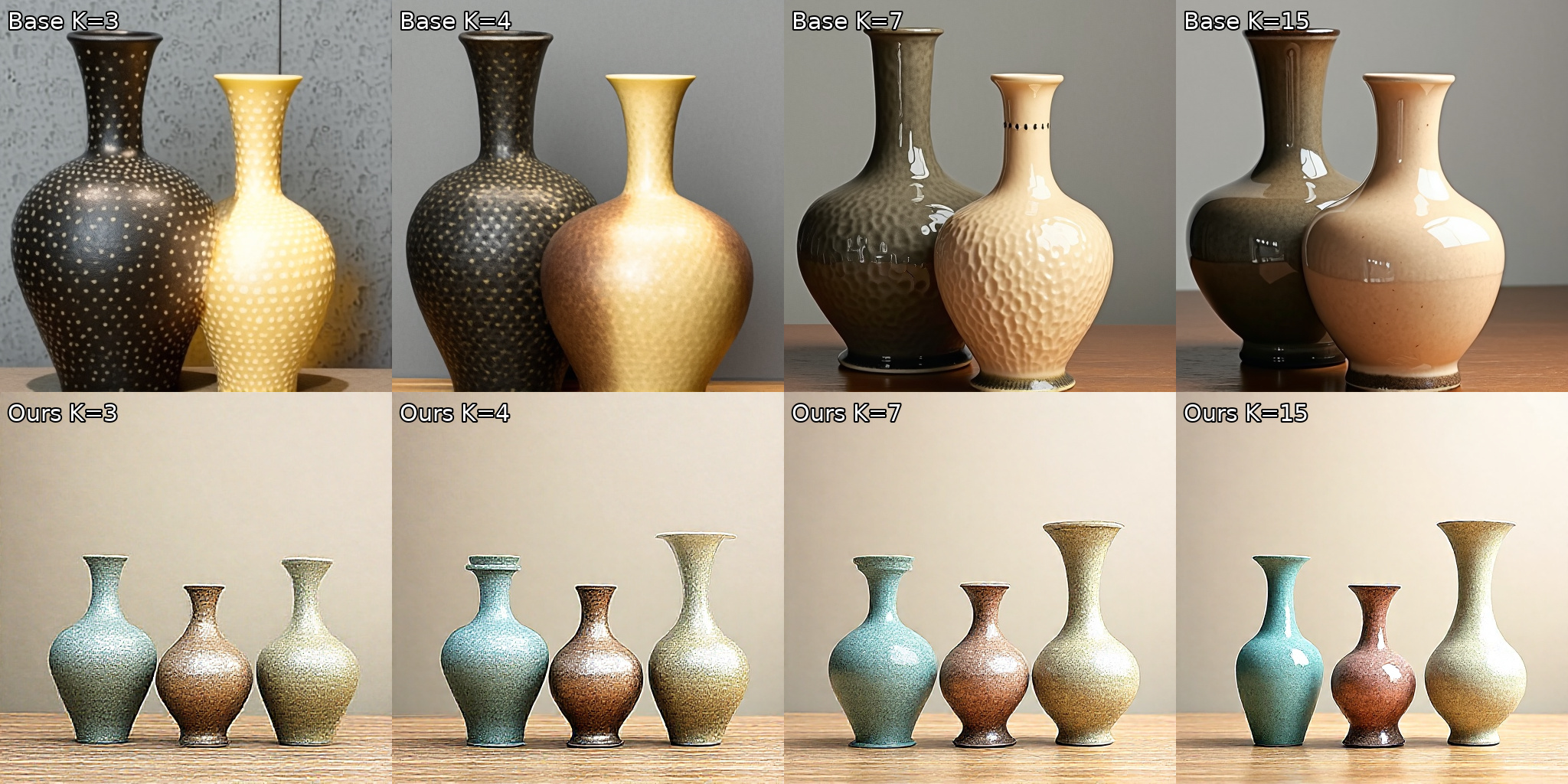}
\end{minipage}
\par\vspace{0.7em}
\begin{minipage}[t]{0.49\linewidth}
{\scriptsize\raggedright\textbf{7.} a photo of a blue bowl and a red cup\par}
\smallskip
\includegraphics[width=\linewidth]{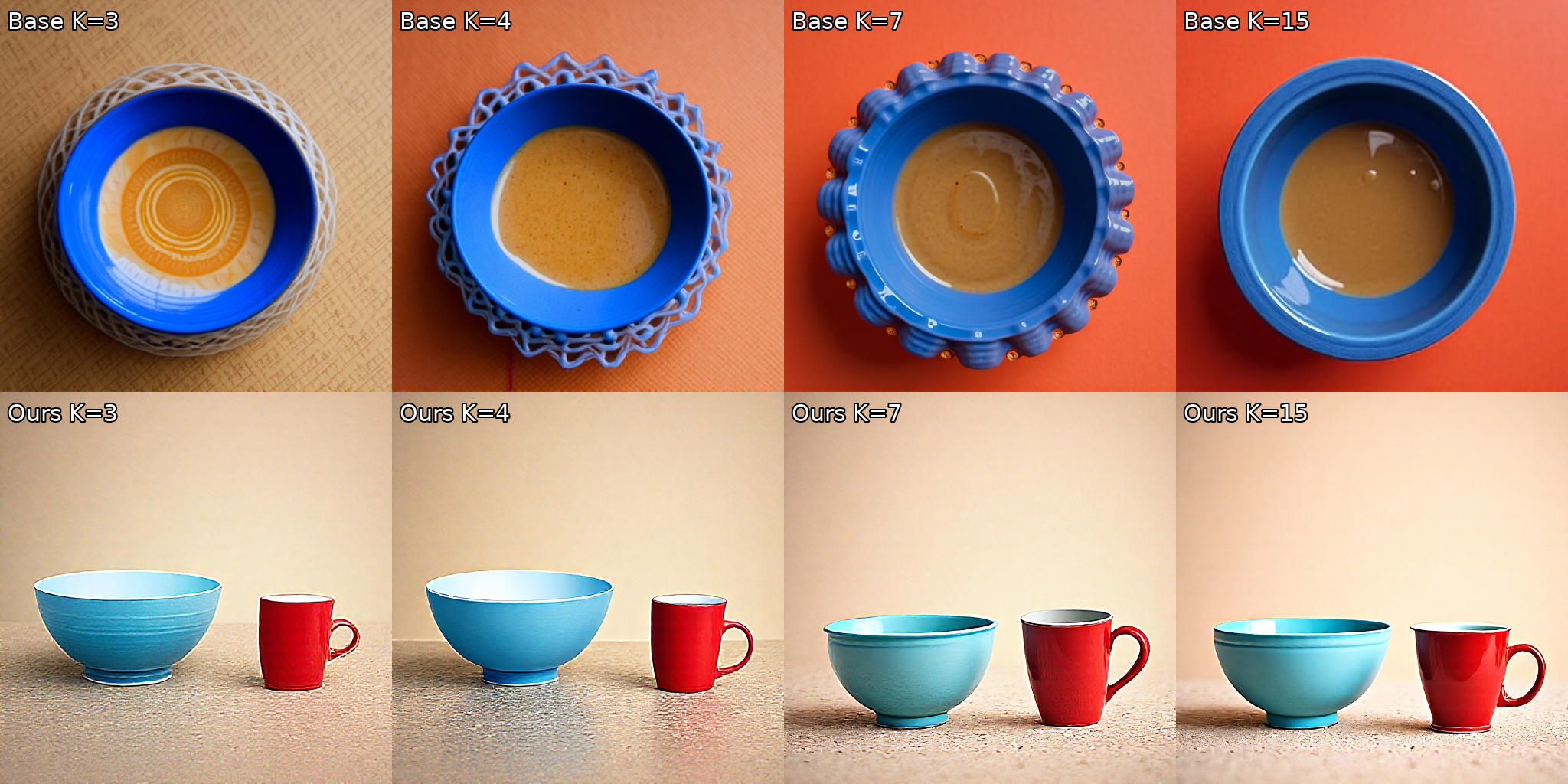}
\end{minipage}
\hfill
\begin{minipage}[t]{0.49\linewidth}
{\scriptsize\raggedright\textbf{8.} a photo of a pineapple and a camera\par}
\smallskip
\includegraphics[width=\linewidth]{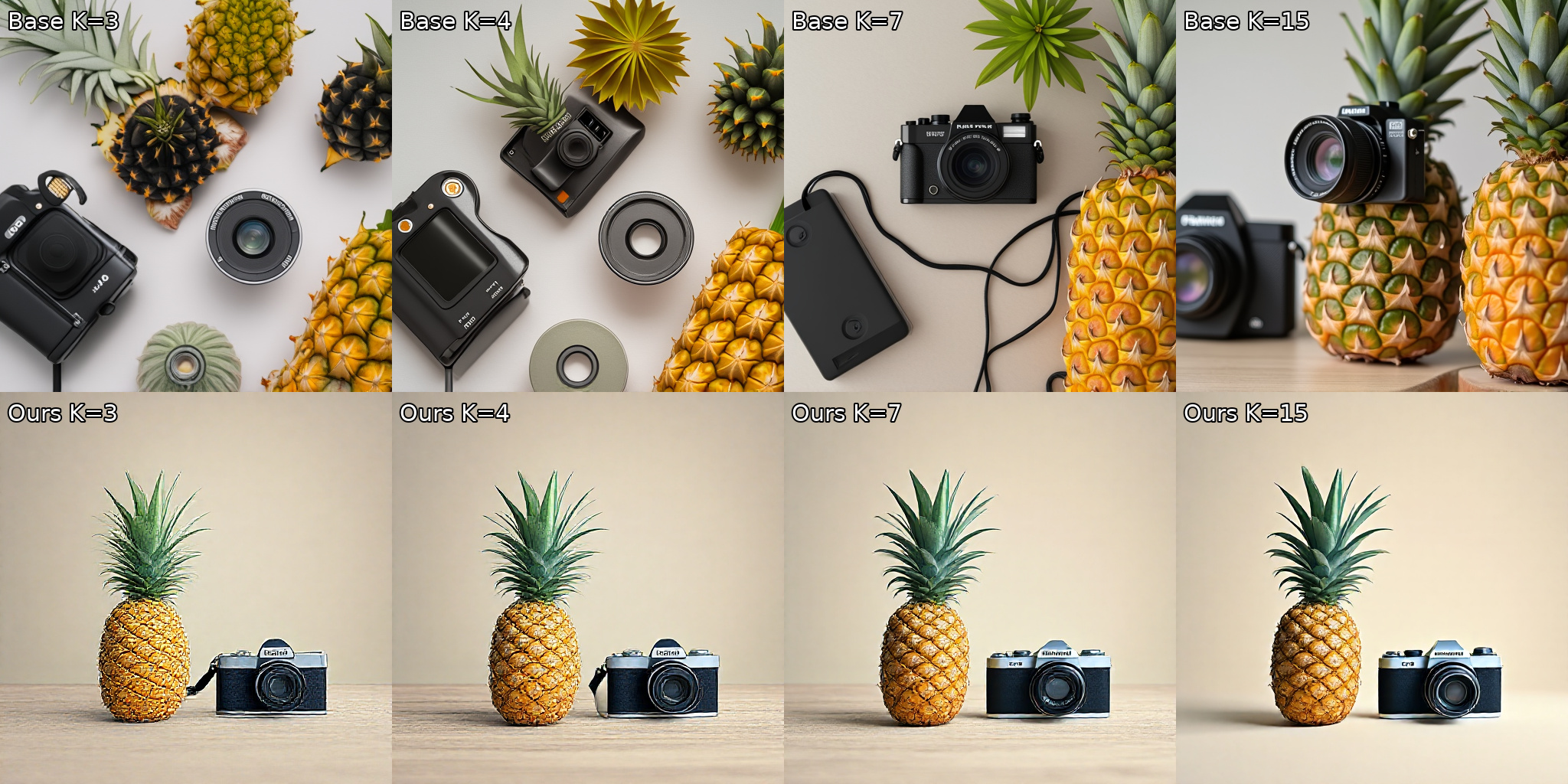}
\end{minipage}
\caption{MeanFlow before and after GenEval post-training. Eight prompt groups illustrate object presence, counting, color, and spatial relations.}
\label{fig:qual_meanflow_geneval}
\end{figure}
\newpage
\subsection{sCM: GenEval}\label{app:qual_scm_geneval}
\begin{figure}[H]
\centering
\begin{minipage}[t]{0.49\linewidth}
{\scriptsize\raggedright\textbf{1.} a photo of a cat left of a red suitcase\par}
\smallskip
\includegraphics[width=\linewidth]{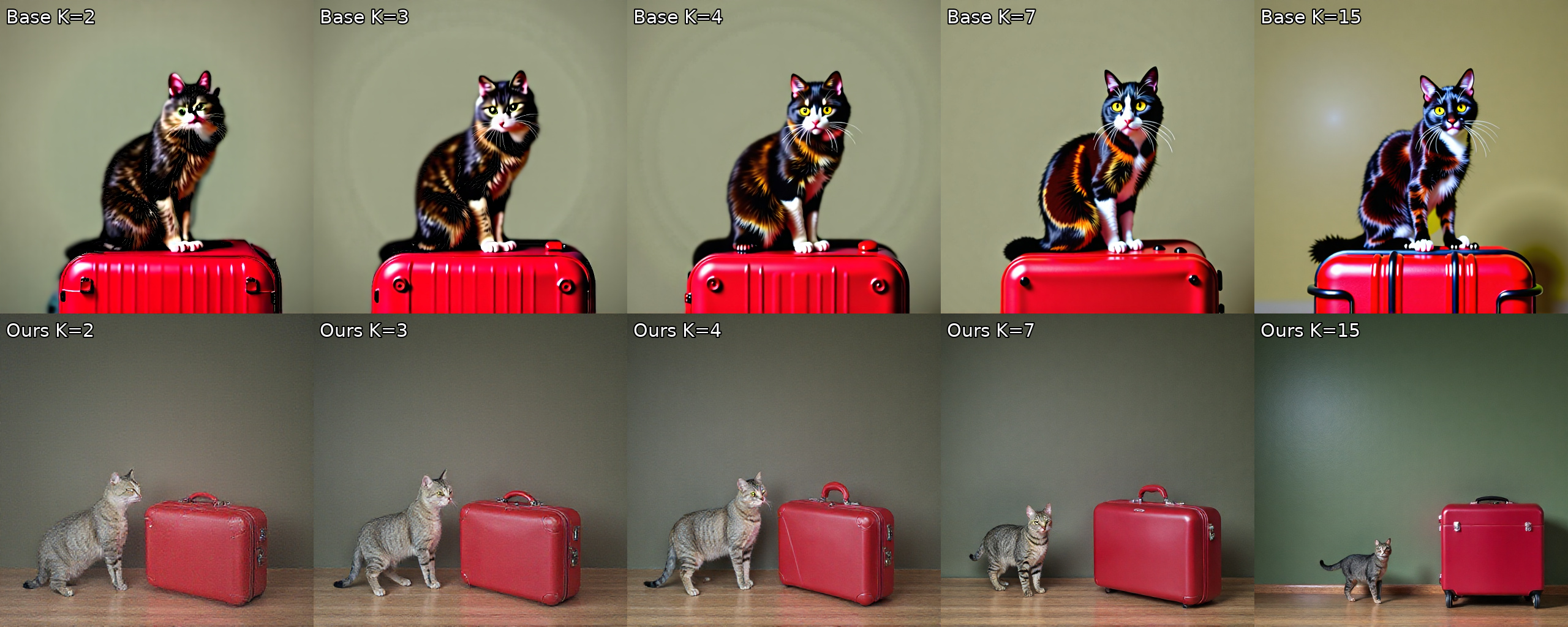}
\end{minipage}
\hfill
\begin{minipage}[t]{0.49\linewidth}
{\scriptsize\raggedright\textbf{2.} a photo of a vase left of a laptop\par}
\smallskip
\includegraphics[width=\linewidth]{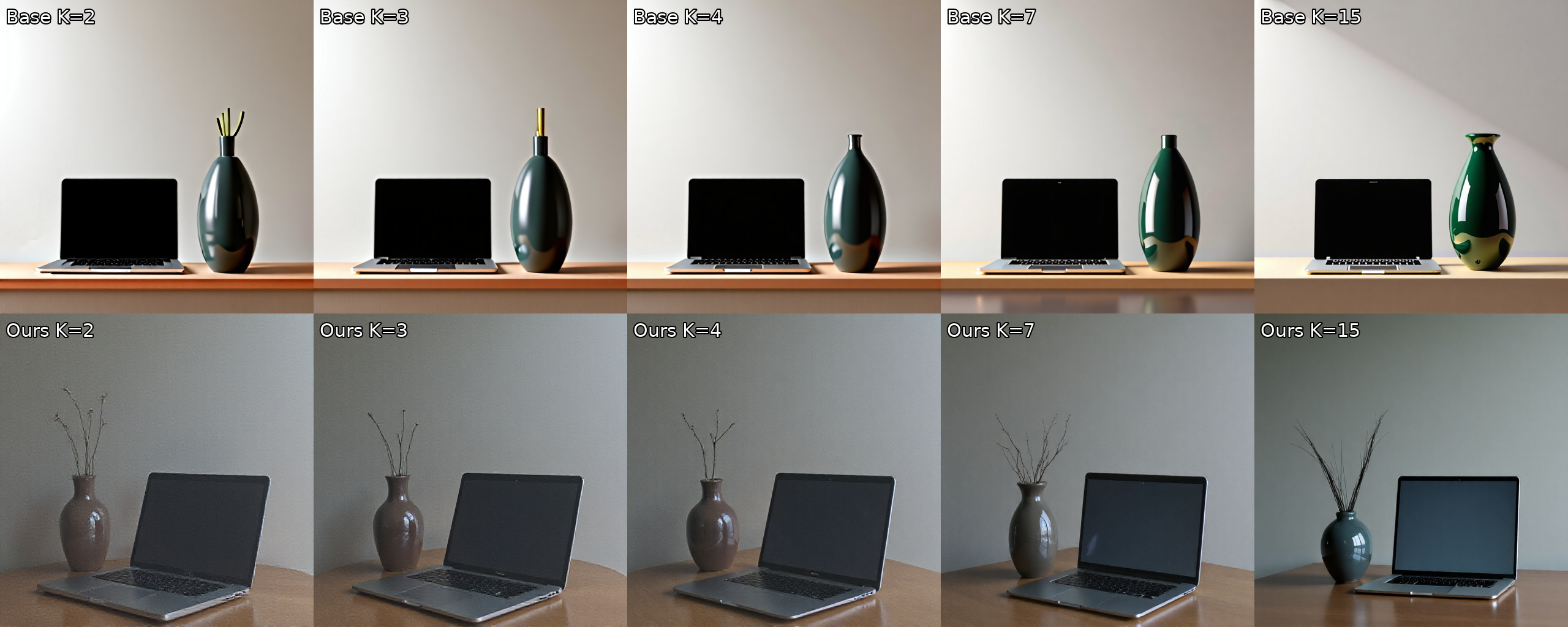}
\end{minipage}
\par\vspace{0.7em}
\begin{minipage}[t]{0.49\linewidth}
{\scriptsize\raggedright\textbf{3.} a photo of a teddy bear right of a bowl\par}
\smallskip
\includegraphics[width=\linewidth]{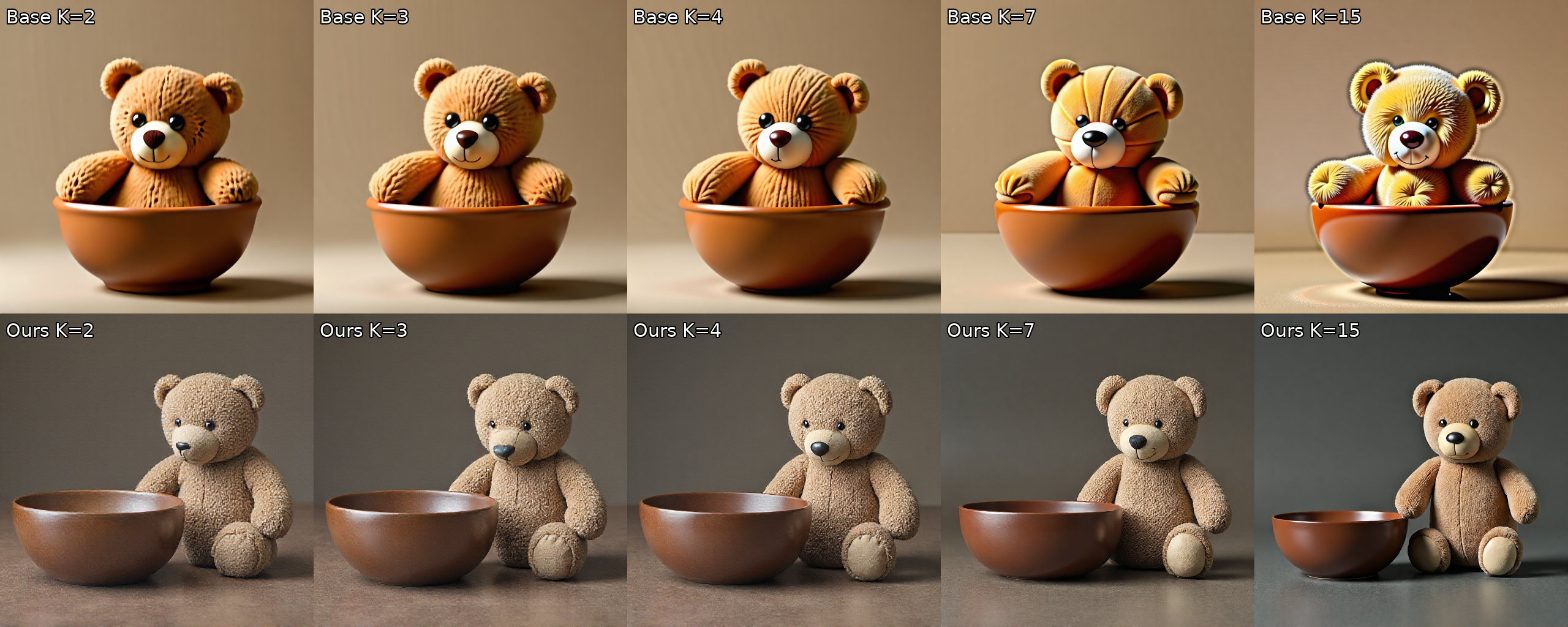}
\end{minipage}
\hfill
\begin{minipage}[t]{0.49\linewidth}
{\scriptsize\raggedright\textbf{4.} a photo of three teddy bears\par}
\smallskip
\includegraphics[width=\linewidth]{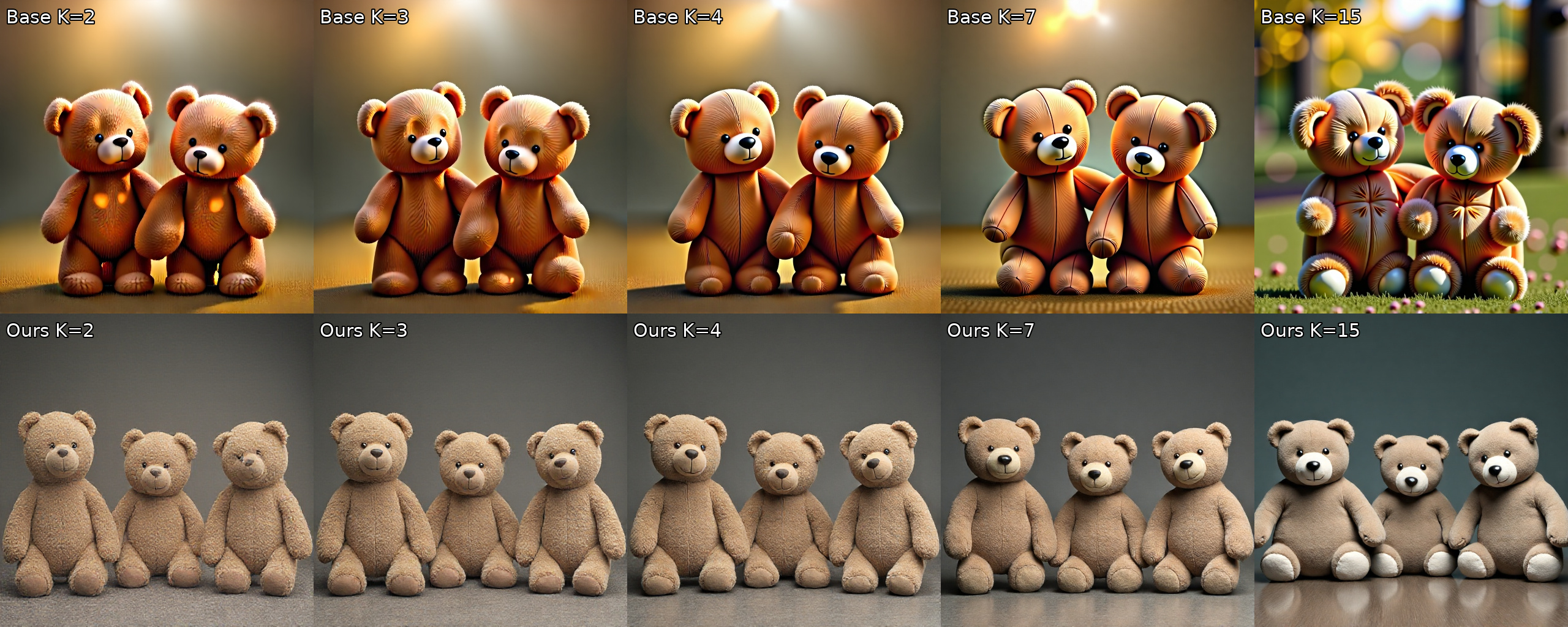}
\end{minipage}
\par\vspace{0.7em}
\begin{minipage}[t]{0.49\linewidth}
{\scriptsize\raggedright\textbf{5.} a photo of three vases\par}
\smallskip
\includegraphics[width=\linewidth]{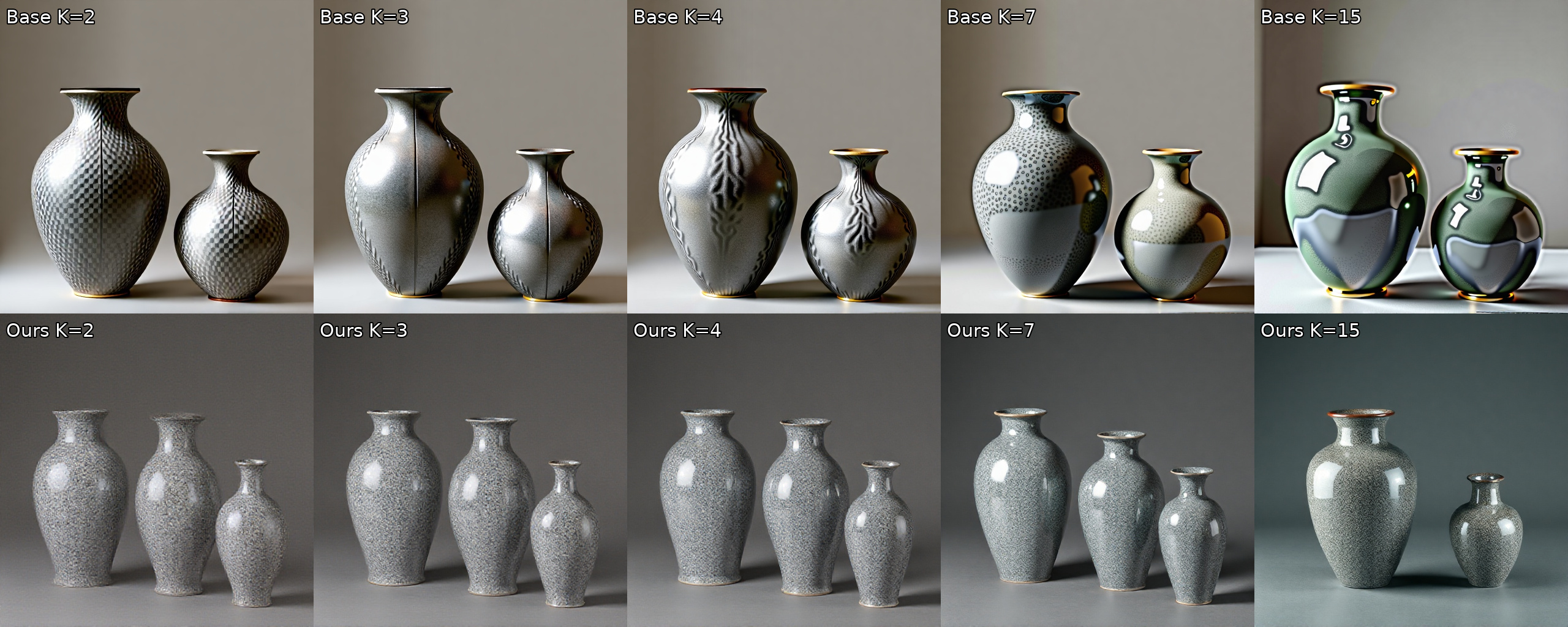}
\end{minipage}
\hfill
\begin{minipage}[t]{0.49\linewidth}
{\scriptsize\raggedright\textbf{6.} a photo of four cups\par}
\smallskip
\includegraphics[width=\linewidth]{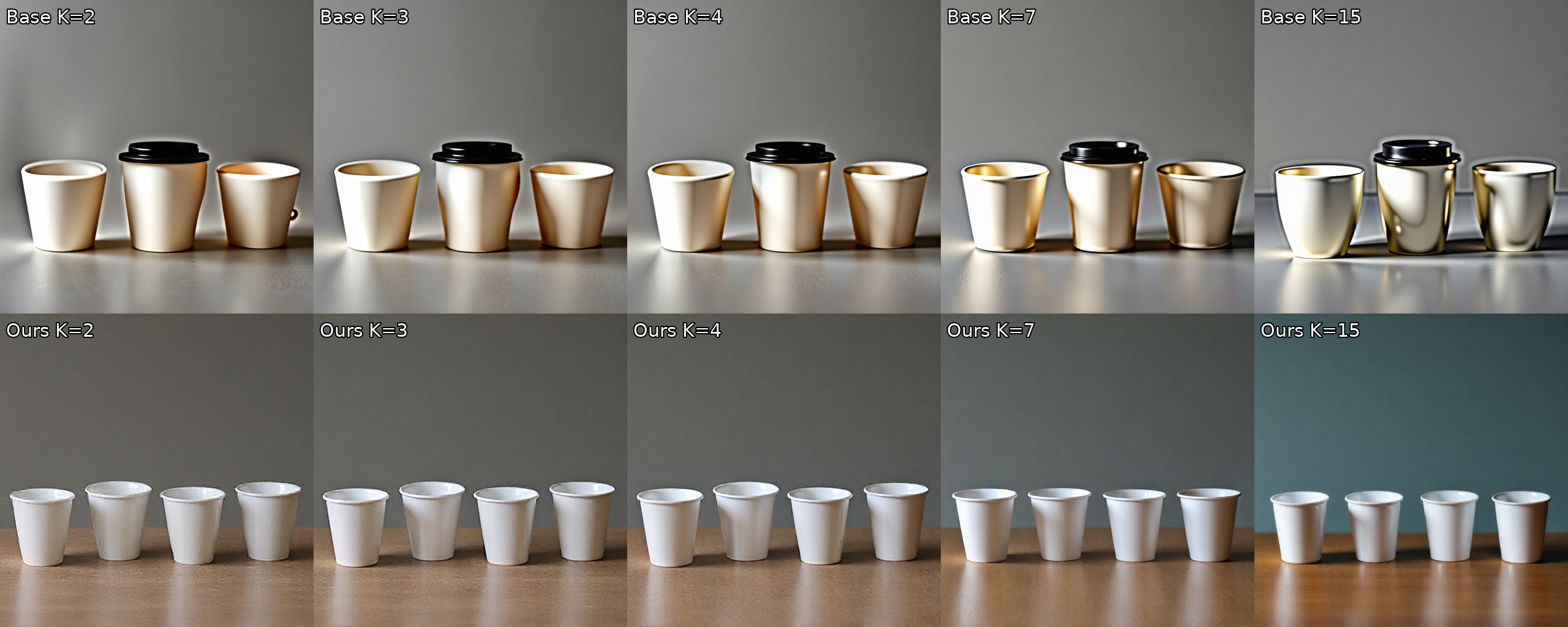}
\end{minipage}
\par\vspace{0.7em}
\begin{minipage}[t]{0.49\linewidth}
{\scriptsize\raggedright\textbf{7.} a photo of a blue bowl and a red cup\par}
\smallskip
\includegraphics[width=\linewidth]{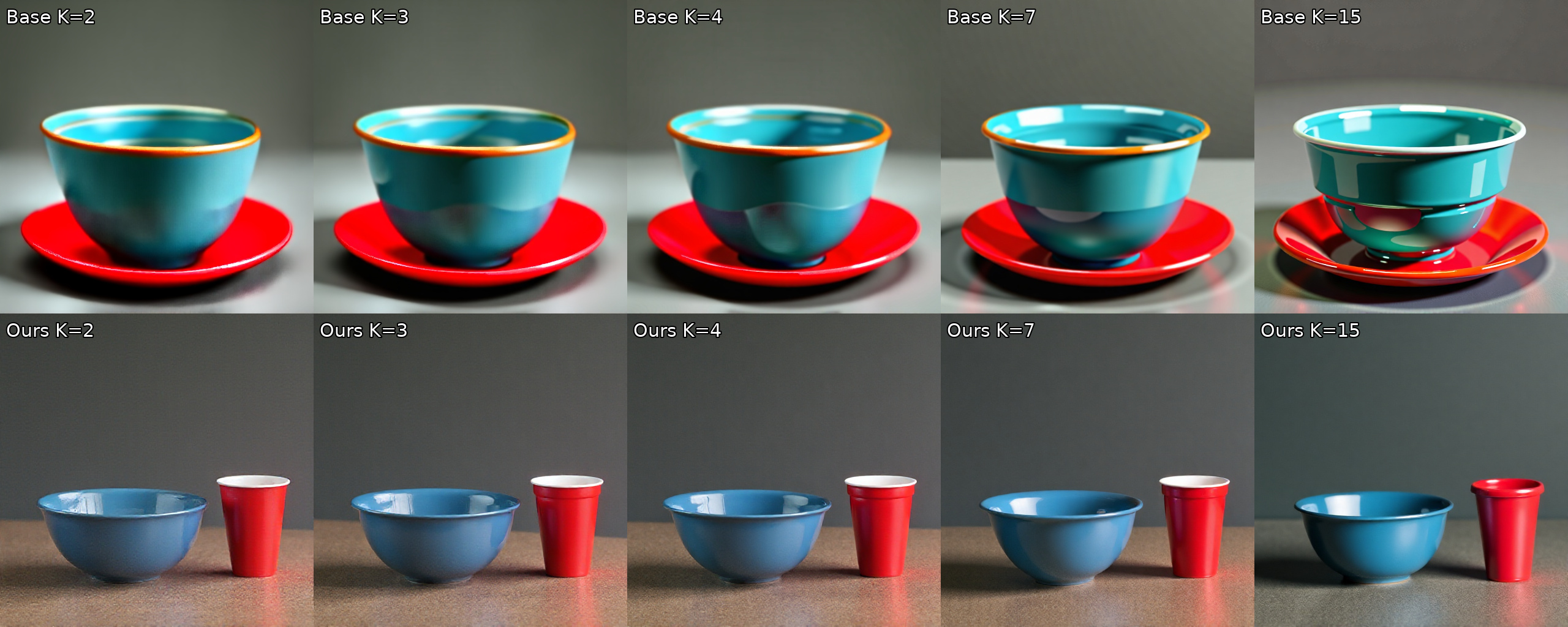}
\end{minipage}
\hfill
\begin{minipage}[t]{0.49\linewidth}
{\scriptsize\raggedright\textbf{8.} a photo of a violin and a suitcase\par}
\smallskip
\includegraphics[width=\linewidth]{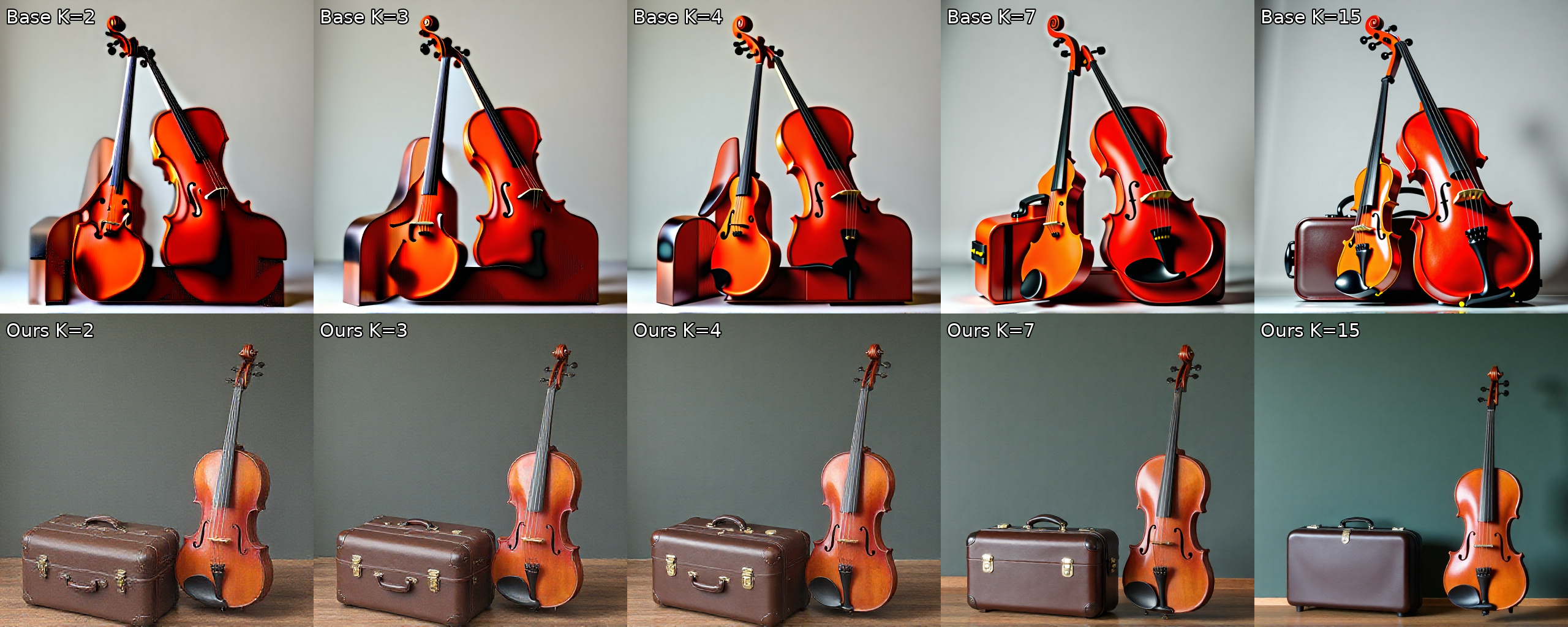}
\end{minipage}
\caption{sCM before and after GenEval post-training. Eight prompt groups illustrate object presence, counting, color, and spatial relations.}
\label{fig:qual_scm_geneval}
\end{figure}
\newpage
\subsection{MeanFlow: OCR + PickScore}\label{app:qual_meanflow_mixed}
\begin{figure}[H]
\centering
{\footnotesize\raggedright\textbf{Prompt 1:} A fantastical cinematic photograph of a floating pastry shop suspended beneath a cream balloon above lavender fields. Its large storefront sign reads ``SUGAR CLOUD''. Golden sunrise, tiny intricate pastries, beautiful pastel color palette.\par}
\smallskip
\includegraphics[width=0.78\linewidth]{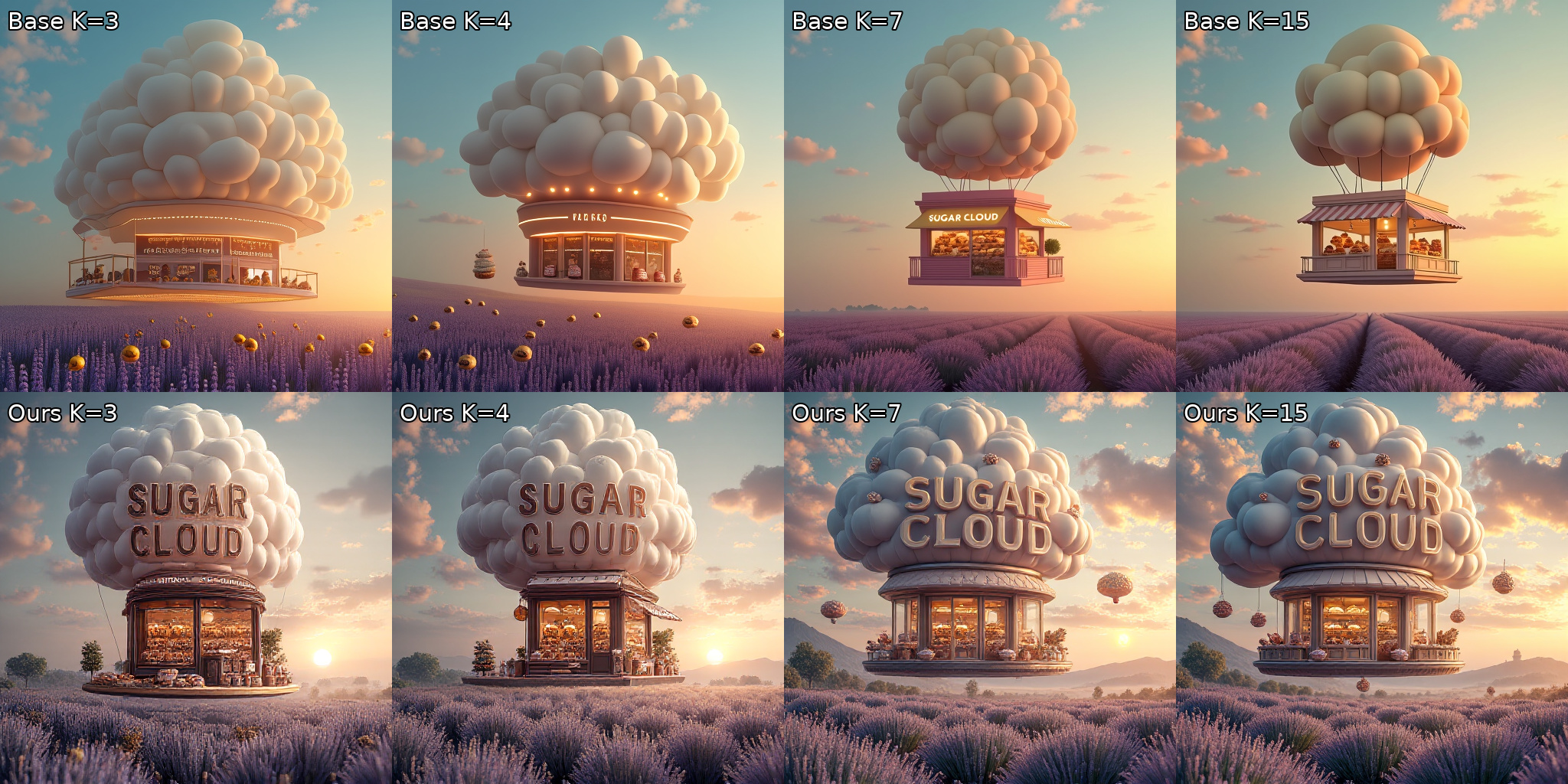}
\par\vspace{0.45em}
{\footnotesize\raggedright\textbf{Prompt 2:} A premium dark chocolate wrapper arranged with raspberries and broken chocolate on rose marble. Large elegant gold lettering reads ``DARK BLOOM''. Burgundy paper, hard editorial shadows, exquisite food photography.\par}
\smallskip
\includegraphics[width=0.78\linewidth]{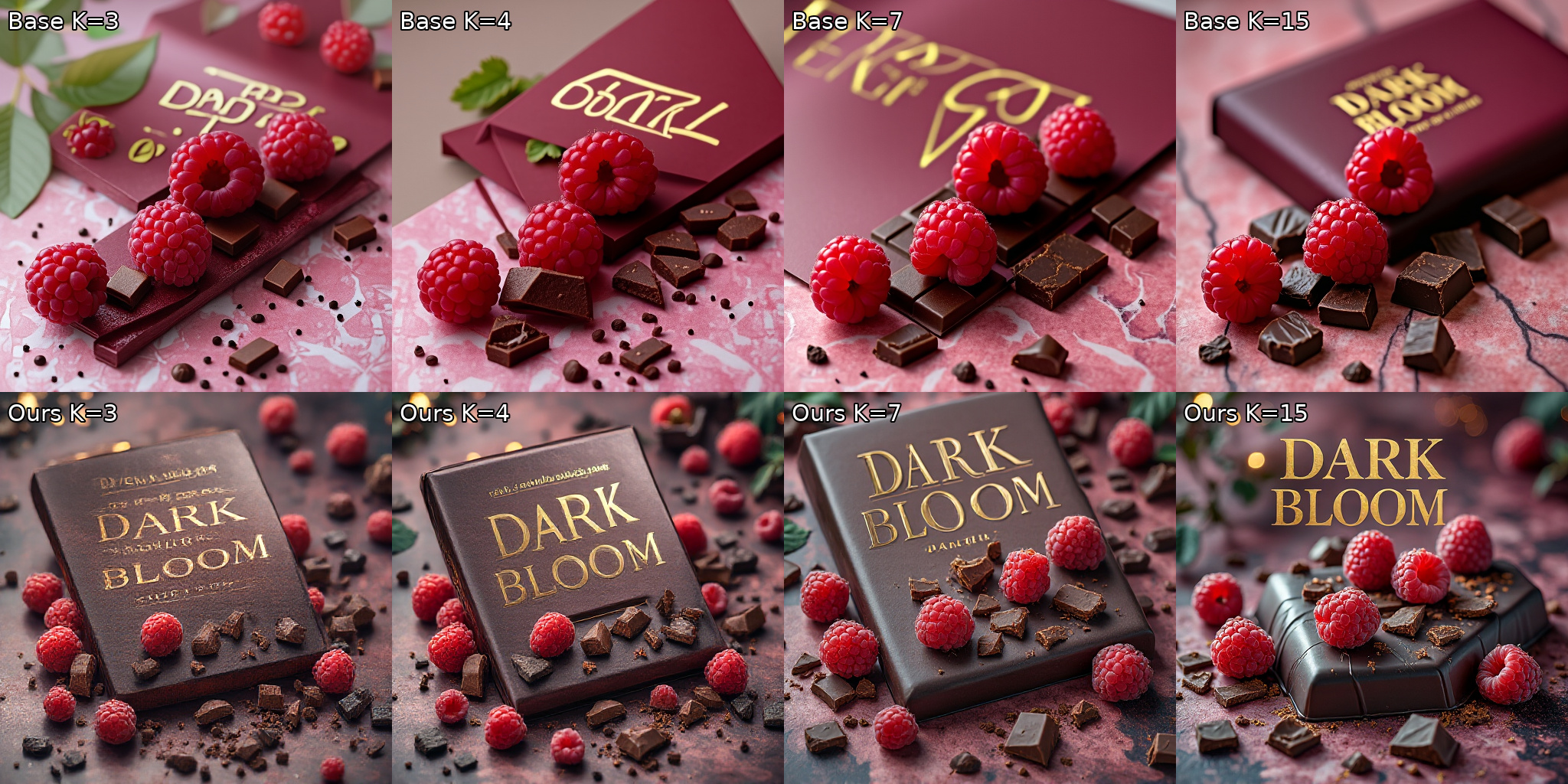}
\par\vspace{0.45em}
{\footnotesize\raggedright\textbf{Prompt 3:} A beautiful illustrated cover of a cookbook, lemons, sardines and ceramic bowls arranged around a large cream central panel. The title reads ``A TASTE OF SUMMER''. Mediterranean blue and yellow palette, textured gouache painting.\par}
\smallskip
\includegraphics[width=0.78\linewidth]{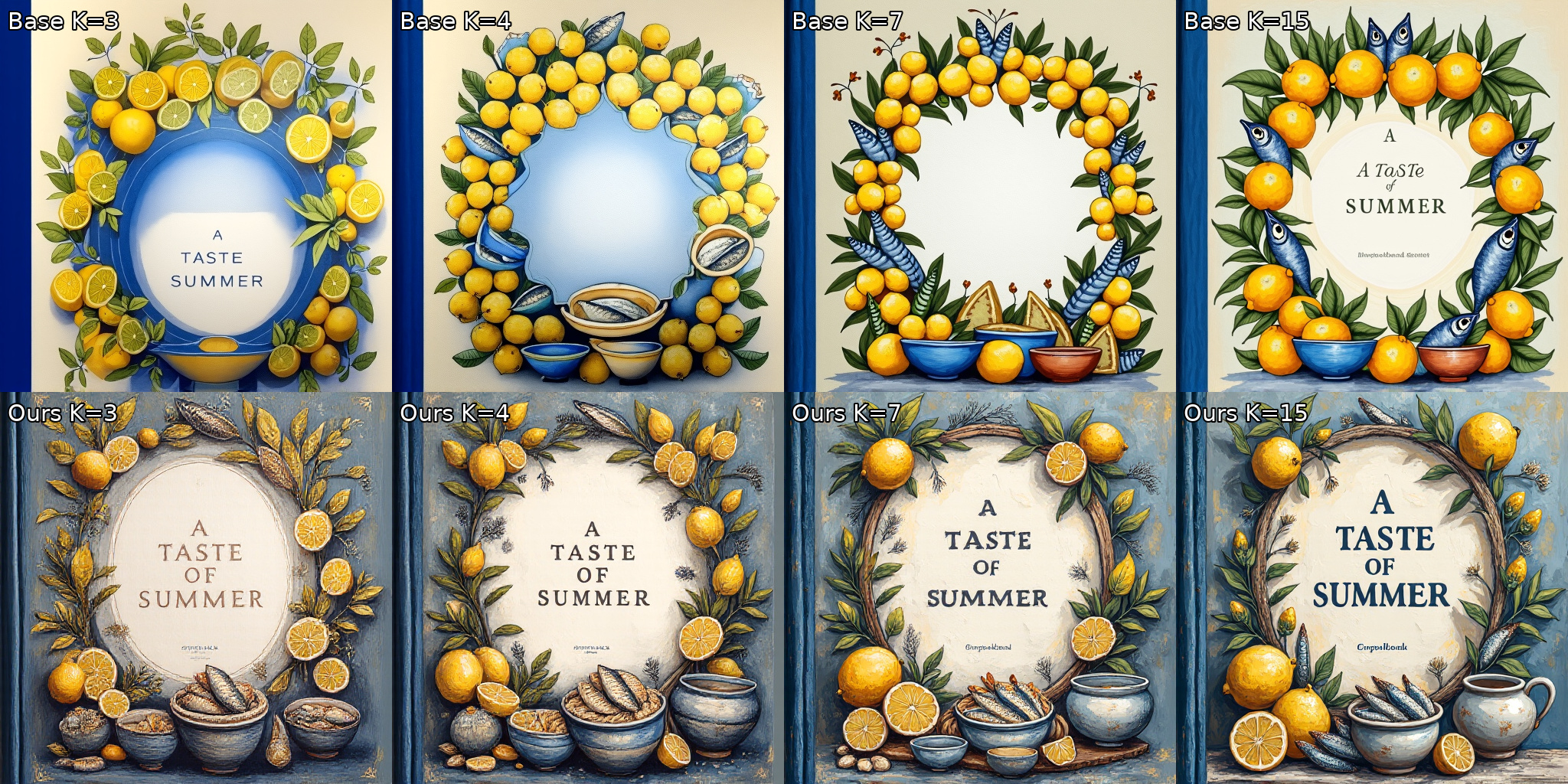}
\caption{MeanFlow before and after OCR + PickScore post-training. Three prompts, each shown as a paired Base/Ours sweep at different numbers of sampling steps.}
\label{fig:qual_meanflow_mixed}
\end{figure}
\newpage
\subsection{sCM: OCR + PickScore}\label{app:qual_scm_mixed}
\begin{figure}[H]
\centering
{\footnotesize\raggedright\textbf{Prompt 1:} A striking editorial photograph of a woman in a flowing red dress standing beside a giant white hot-air balloon on a salt flat at sunrise. The balloon clearly reads ``DREAM BIG'' in large black letters. Pink sky reflected in shallow water, strong graphic composition.\par}
\smallskip
\includegraphics[width=0.97\linewidth]{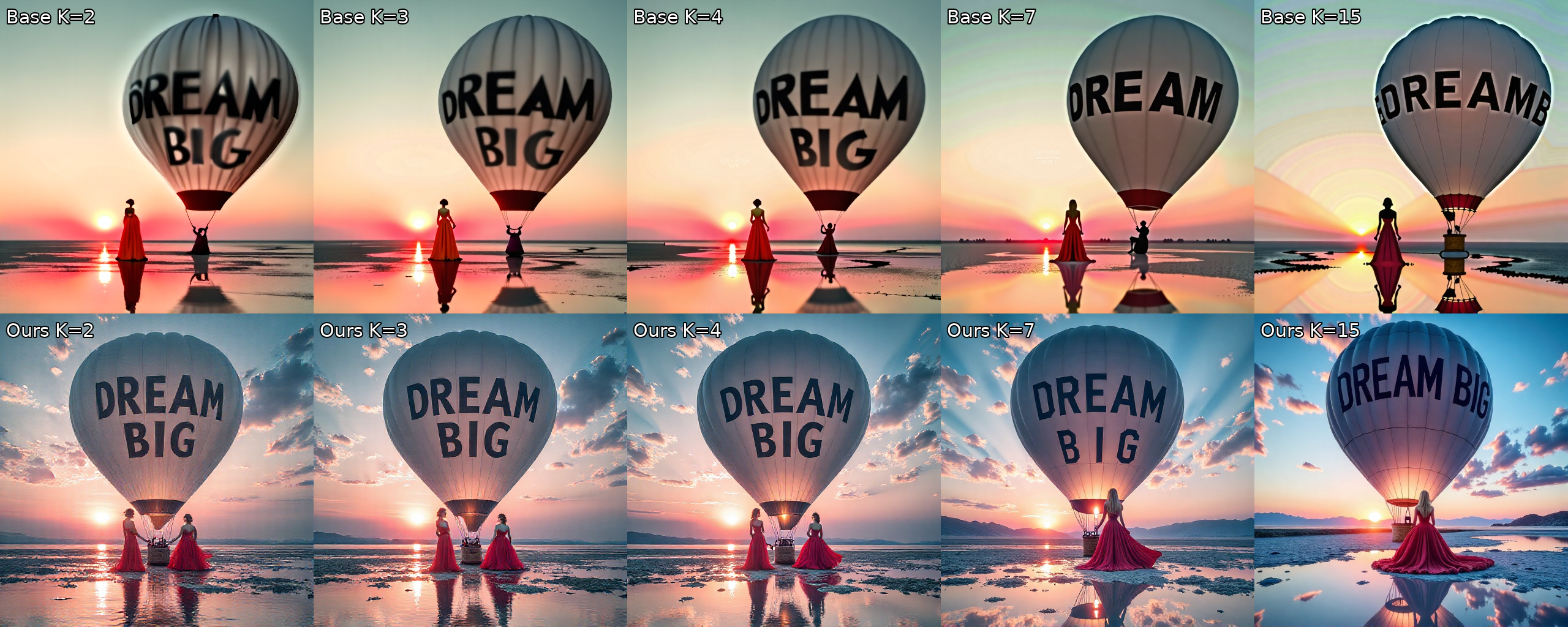}
\par\vspace{0.45em}
{\footnotesize\raggedright\textbf{Prompt 2:} A cinematic wide view of an astronaut resting on a pink lunar ridge, Earth in the dark sky. A white flag in the foreground reads ``STAY CURIOUS'' in clear dark lettering. Detailed spacesuit, dramatic rim lighting, spectacular landscape.\par}
\smallskip
\includegraphics[width=0.97\linewidth]{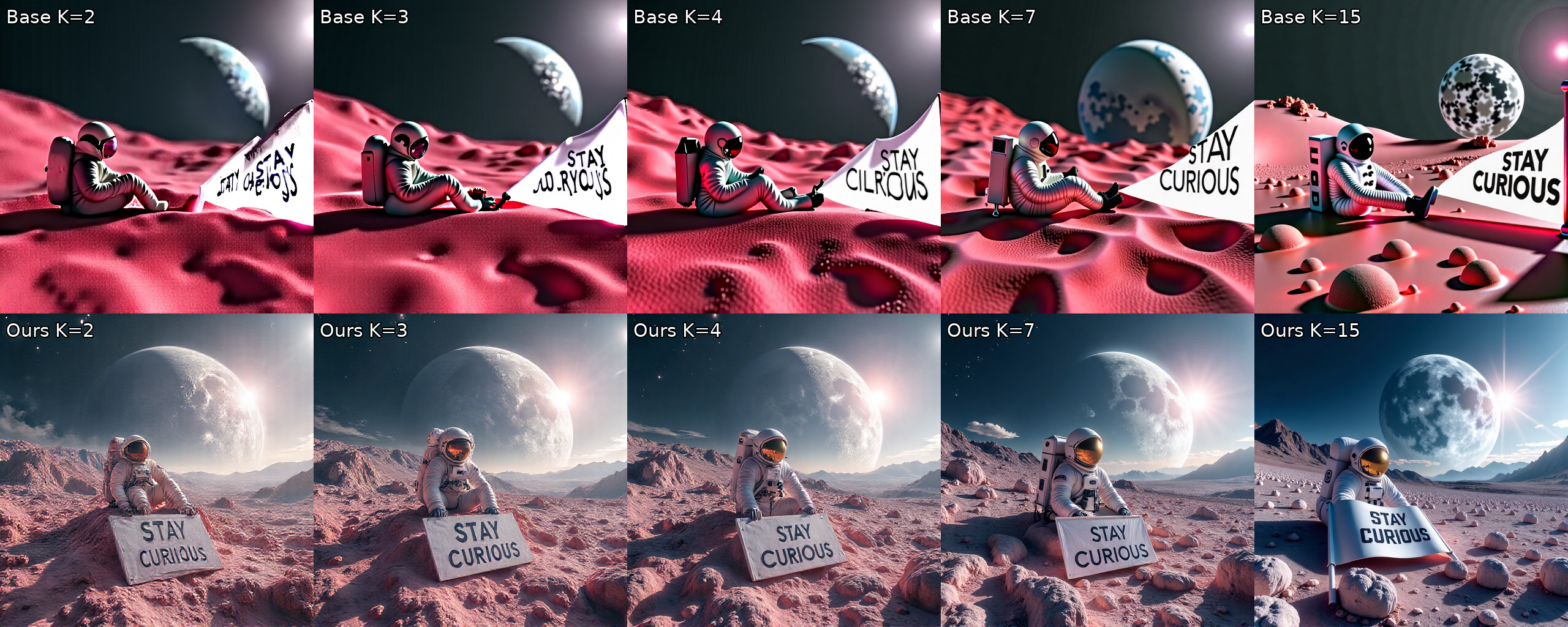}
\par\vspace{0.45em}
{\footnotesize\raggedright\textbf{Prompt 3:} A moody editorial photograph of a burgundy vinyl record sleeve beside a vintage turntable, a single red rose and warm amber light. Large cream title on the sleeve reads ``MIDNIGHT JAZZ''. Rich tactile paper and wood textures.\par}
\smallskip
\includegraphics[width=0.97\linewidth]{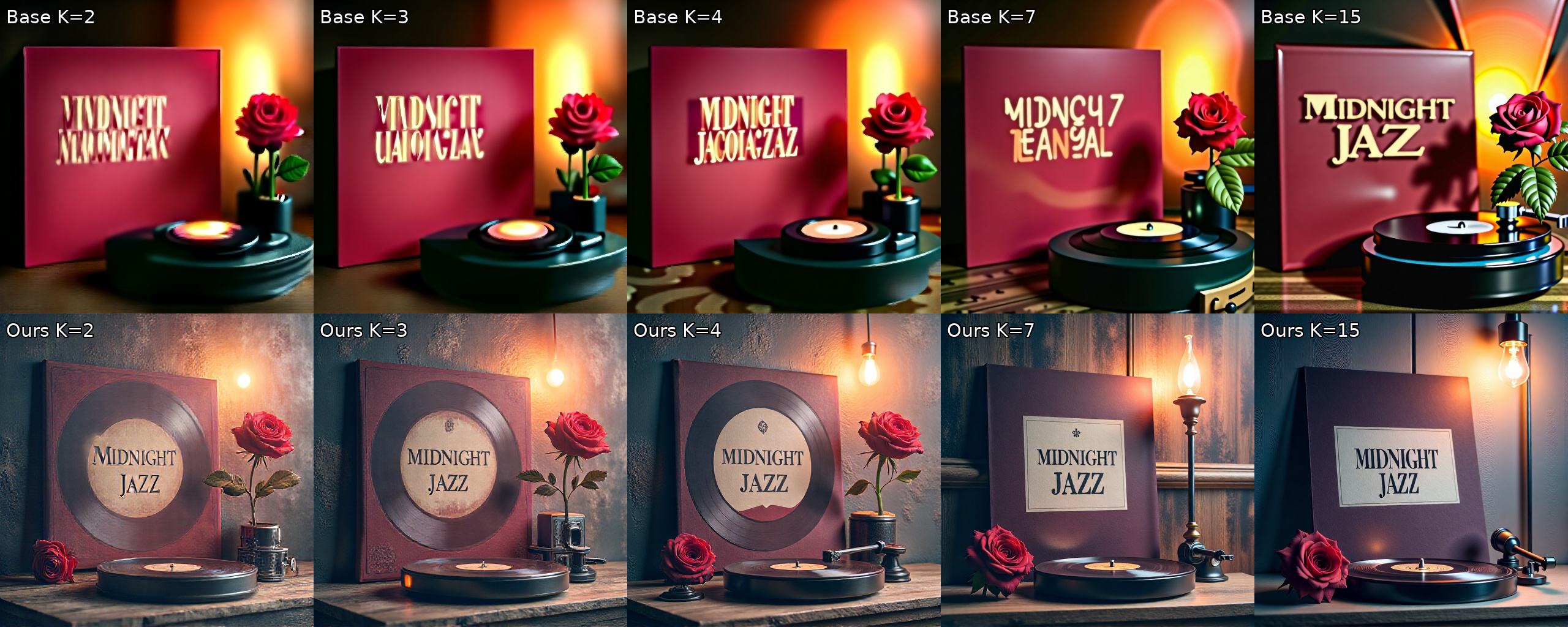}
\caption{sCM before and after OCR + PickScore post-training. Three prompts, each shown as a paired Base/Ours sweep at different numbers of sampling steps.}
\label{fig:qual_scm_mixed}
\end{figure}
\newpage
\end{document}